\documentclass[11pt]{article}
\usepackage[margin=1in]{geometry}

\usepackage[utf8]{inputenc} % allow utf-8 input
\usepackage[T1]{fontenc}    % use 8-bit T1 fonts
\usepackage{lmodern}
\usepackage{microtype}

\usepackage{amsmath,amssymb,amsfonts,amsthm}
\usepackage{thmtools}
\usepackage{stmaryrd}
\usepackage{bbold}

\usepackage{graphicx}
\usepackage{booktabs}
\usepackage{multirow}
\usepackage{wrapfig}
\usepackage{caption}
\usepackage{float}
\usepackage{enumitem}

\usepackage{algorithm}
\usepackage{algorithmicx}
\usepackage{algpseudocode}

\usepackage{xcolor}
\usepackage{nicefrac}
\usepackage{xparse}
\usepackage{url}
\usepackage{titletoc}

\usepackage{tikz}
\usetikzlibrary{arrows.meta}

\usepackage[round,authoryear]{natbib}
\usepackage[colorlinks=true,linkcolor=blue,citecolor=blue,urlcolor=blue]{hyperref}

\NewDocumentCommand{\param}{o}{%
  \mathrm{W}%
  \IfValueT{#1}{^{(#1)}}%
}

\newcommand{\E}{\mathbb{E}}

\newcommand{\un}{\mathbf{1}}

\newcommand{\loss}{\mathbf{e}}
\newcommand{\ck}{c(t^{(k)})}

\newcommand{\sgd}{\mathrm{sgd}}

\newcommand{\ola}{\overleftarrow}
\newcommand{\ora}{\overrightarrow}
\newcommand{\rmd}{\mathrm{d}}
\newcommand{\rme}{\mathrm{e}}
\newcommand{\s}{\mathrm{s}}
\newcommand{\eqsp}{\,}

\newcommand{\score}{\mathsf{s}}
\newcommand{\eqdef}{:=}

\newcommand{\x}{u}
\newcommand{\X}{\mathbf{U}}
\def \R{\mathbb{R}}

\newcommand{\net}{f}
\newcommand{\fnet}{f_{\theta}}
\newcommand{\snet}{\s_{\theta}}

\newcommand{\cscore}[1]{\score_{#1\mid 0}}

\newcommand{\Id}{\mathbf{I}}

\usepackage[textsize=tiny]{todonotes}

\newtheorem{theorem}{Theorem}[section]
\newtheorem{corollary}{Corollary}[section]
\newtheorem{lemma}{Lemma}[section]

\newtheorem{remark}{Remark}
\newtheorem{assumption}{Assumption}

\title{Non-asymptotic Convergence of Stochastic Gradient Descent in Score-based Generative Models}

\author{
Stanislas Strasman$^{1,*}$ \quad
Sobihan Surendran$^{1,2,*}$ \quad
Sylvain Le Corff$^{1}$\\[0.5em]
\small $^1$Sorbonne Université and Université Paris Cité, CNRS, LPSM, F-75005 Paris, France\\
\small $^2$LOPF, Califrais' Machine Learning Lab, Paris, France\\
\small $^*$Equal contribution
}

\date{}

\begin{document}

\maketitle

\begin{abstract}
  Score-based Generative Models (SGMs) have achieved impressive performance in data generation across a wide range of applications. While the statistical properties of their sampling procedures are increasingly well understood, the optimization dynamics underlying their training remain less explored. SGMs are typically trained by minimizing a weighted denoising score-matching objective, yet optimization guarantees with stochastic gradients remain limited. In this work, we study Stochastic Gradient Descent (SGD) for SGMs, contributing results in two complementary regimes. For general score parameterizations, we derive a non-convex analysis of SGD for the weighted denoising score-matching objective, making explicit how the resulting optimization bound depends on the loss weighting and time-sampling distribution. We then consider overparameterized two-layer ReLU networks and develop a Neural Tangent Kernel analysis tailored to diffusion training with stochastic gradients, yielding score-approximation error bounds along the SGD trajectory. Our analysis quantifies the role of the reweighting factor in these bounds, providing a theoretical characterization of weighting choices used in practice.
  %First, for general score parameterizations, we establish a non-convex convergence rate for SGD on the weighted denoising score-matching objective, with explicit dependence on the schedule-dependent weighting factors. Second, for overparameterized two-layer ReLU networks, we develop a Neural Tangent Kernel analysis tailored to diffusion training with stochastic gradients, yielding score-approximation error bounds along the SGD trajectory. Finally, our analysis quantifies the role of the reweighting factor in the score approximation error, providing theoretical guidance for weighting choices used in practice.
\end{abstract}

\section{Introduction}

Generative modeling has become a central topic in modern machine learning, driven by the remarkable progress of score-based generative models \citep{sohldickstein2015deep, song2019generative, ho2020denoising}. These models provide a flexible framework for learning complex high-dimensional distributions and generating realistic synthetic samples. Their success relies on gradually perturbing the data distribution through a forward noising process and then learning a reverse mechanism that transforms noise back into data. This paradigm has led to impressive empirical performance across a wide range of applications, including computer vision \citep{li2022srdiff,lugmayr2022repaint}, natural language processing \citep{gong2022diffuseq}, and other domains where realistic data generation is crucial. %Inspired by Hamiltonian Monte Carlo, second-order variants such as critically-damped Langevin diffusions (CLDs) \citep{dockhorn2022score} have also been proposed as extensions of conventional diffusion models.

Beyond their empirical success, diffusion-based models have been extensively studied from a theoretical perspective, particularly through stability and sampling error bounds. These typically separate initialization or mixing, discretization, and score-approximation errors \citep{chen2023sampling, chen2023improved, benton2024nearly, conforti_kl, strasman2025an, strasman2026forgetting}. Much of this literature focuses on the first two sources of error, often assuming that the approximation error is sufficiently small. However, the approximation error itself depends on both the statistical error induced by the chosen model class and the optimization error arising from the training procedure.

Several works have established generalization bounds for diffusion models
\citep{li2023generalization, chen2023score, yakovlev2025generalization, dupuisalgorithm, fu2025approximation, stephanovitch2025generalization}.
In contrast, the optimization error arising from the training of these models remains comparatively less explored. Diffusion models are typically trained using score matching \citep{Vincent}. In this direction, \citet{han2024neural} establish score-estimation
guarantees for diffusion models trained by gradient descent through a Neural Tangent Kernel (NTK) analysis \citep{jacot2018neural}. Their guarantees incorporate projection of the denoiser outputs onto a fixed Euclidean ball. More recently, \cite{zhang2025convergence} studies the optimization dynamics of gradient descent for overparameterized score matching in the simplified setting of learning a single Gaussian distribution, using a gradient Expectation-Maximization algorithm. However, both works focus on gradient descent, whereas practical diffusion models are commonly trained using SGD and its variants \citep{robbins1951stochastic, bottou2018optimization}. Moreover, \cite{han2024neural} assumes bounded data, while \cite{zhang2025convergence} assumes that the target distribution is Gaussian; both assumptions are restrictive and do not capture the complexity of real-world data distributions.

In this work, we study the optimization error of score-based diffusion models trained with SGD. We begin with a general score parameterization, where standard non-convex stochastic optimization arguments yield a baseline convergence guarantee for the score matching objective. This preliminary analysis makes explicit how the optimization bound depends on the loss weighting and the time-sampling distribution. However, it only controls the expected squared gradient norm of the objective and does not directly control the score-approximation error. Moreover, the required regularity assumptions exclude non-smooth architectures such as ReLU networks.
To address these limitations, we then consider overparameterized two-layer
ReLU networks trained with stochastic gradients. Using an NTK analysis, we derive non-asymptotic score-approximation bounds along the SGD trajectory and connect these guarantees to the generation error. We further use the resulting bounds to study the role of the reweighting factor and quantify its effect on both optimization and sampling.

More precisely, our contributions are summarized as follows.
% \begin{itemize}[wide=0.2pt]
%     \item We establish non-convex optimization guarantees for the weighted denoising score-matching objective under general score parameterizations. Our analysis shows that SGD converges at rate $\mathcal{O}(\log n/\sqrt{n})$ in expected squared gradient norm, and explicitly captures the impact of mini-batches, stochastic gradient variance, and schedule-dependent loss weights.

%     \item We provide an NTK analysis of overparameterized two-layer ReLU score networks specifically tailored to diffusion training. This leads to a non-asymptotic bound on the score matching loss along the SGD trajectory. Unlike prior works based on deterministic gradient descent, bounded data assumptions, or Gaussian target distributions, our framework directly handles stochastic optimization and allows for sub-Gaussian data distributions.
    
%     \item Using our bounds, we analyze how the reweighting factor in the denoising score-matching loss affects both optimization and sampling. This provides theoretical guidance for its choice and offers an optimization-based explanation for weighting strategies that have proved effective in state-of-the-art diffusion models.
% \end{itemize}
\begin{itemize}[wide=0.3pt]
    \item As a preliminary step, we derive a standard $\mathcal O(\log n/\sqrt n)$ non-convex SGD rate, explicitly capturing the dependence on minibatch size, step size, reweighting factor, and time-sampling distribution.
    
    \item We provide an NTK analysis of overparameterized two-layer ReLU networks,
    yielding non-asymptotic score-approximation guarantees along the SGD
    trajectory. Unlike prior analyses relying on deterministic gradient
    descent, bounded data, or Gaussian targets, our framework accommodates
    stochastic optimization and sub-Gaussian data without projecting network
    outputs onto a compact set. Under suitable stability conditions on the
    learned reverse dynamics, these guarantees translate into Wasserstein
    generation-error bounds jointly accounting for both the learning and sampling aspects of SGMs.

    \item Using our results, we analyze how the reweighting factor in the score matching loss affects both optimization and sampling. This provides theoretical guidance for its choice and offers an optimization-based explanation for weighting strategies that have proved effective in state-of-the-art diffusion models.
\end{itemize}

\section{Notation and Background}
\label{sec:background}

\subsection{Notation}

Let $\mathcal P(\R^d)$ denote the set of probability
measures on $\R^d$. We use $\pi$ for probability
measures and $p$ for their densities with respect to the Lebesgue measure when
they exist. The identity matrix of size $d$ is written $\Id_d$. For $x,y\in \mathbb{R}^d$, we denote by $\langle x,y \rangle$ the standard inner product of $\mathbb{R}^d$, by $\|\cdot\|$ the Euclidean norm for vectors and its induced operator norm for matrices. Let $\|\cdot\|_F$ be the Frobenius norm defined for $A \in \mathbb{R}^{d \times d}$ as  $\|A \|_F \eqdef \sqrt{\text{Tr} (A^\top A)}$. For random vectors $X,Y \in \mathbb{R}^d$, define $\|X\|_{L_2} \eqdef
(\mathbb{E}[ \|X\|^2])^{1/2}$. %and we write $X \perp Y$ to mean that  $X$ is independent of $Y$.
For $a, b \in \mathbb{R}$, we write $a \wedge b \eqdef \min\{a,b\}$ and $a \vee b \eqdef \max\{a,b\}$.

\subsection{Score-based Generative Models}

\paragraph{Forward noising process.} Let $\pi_{\rm data}\in \mathcal P(\R^d)$ denote the target data distribution from which we wish to generate new samples. In generative machine learning, one does not observe $\pi_{\rm data}$ directly. Instead, one  has
access to i.i.d. samples $X_1^{\rm data},\dots,X_N^{\rm data}\sim \pi_{\rm
data}$. SGMs construct a stochastic transport that
progressively maps $\pi_{\rm data}$ toward a simple reference distribution
$\pi_\infty$, and then approximately reverse this transport to generate new
samples from the data law. In particular, they rely on a forward linear noising process defined as the solution to the following SDE:
\begin{equation}
\label{eq:forward_sde}
\rmd \ora X_t
=
- \alpha \beta_t \ora X_t \,\rmd t
+
\sqrt{2 \beta_t}\,\rmd B_t, \qquad X_0 \sim \pi_{\rm data} \eqsp,
\end{equation}
where $(B_t)_{t \ge 0}$ is a standard Brownian motion in $\mathbb{R}^d$, $\beta_t : [0,T] \to \R_{\ge 0}$ is a time-dependent Lipschitz noise schedule, and $\alpha \geq 0$. This linear SDE admits an explicit time marginal representation \citep[Lemma C.1.]{strasman2026forgetting}, for each fixed $t \in [0,T]$:
\begin{equation}
\label{eq:forward_marginal}
\ora X_t
\stackrel{\mathcal{L}}{=}
m_t X_0
+
\sigma_t Z \eqsp, \quad Z \sim \mathcal{N}(0, \Id_d) \eqsp,
\quad
Z \perp X_0 \eqsp,
\end{equation}
with $m_t
\eqdef
\exp (-\alpha \int_0^t \beta_s \,\rmd s)$ and $\sigma_t^2 \eqdef 2 \int_0^t
\beta_s
\exp (
-2\alpha \int_s^t \beta_u  \rmd u
)
\rmd s$.
% \begin{align*}
% m_t
% \eqdef
% \exp\!\left(-\alpha \int_0^t \beta_s \,\rmd s\right), \qquad \qquad
% \sigma_t^2 \eqdef 2 \int_0^t
% \beta_s
% \exp \left(
% -2\alpha \int_s^t \beta_u  \rmd u
% \right)
% \rmd s \eqsp.
% \end{align*}
Equation~\eqref{eq:forward_sde} covers the standard forward diffusions used in the score-based modeling literature. When $\alpha=0$, one has $m_t\equiv 1$ and $\sigma_t^2 = 2\int_0^t \beta_s\,\rmd s$; this corresponds to the
variance-exploding (VE) regime used in score-based generative modeling \citep{song2019generative,song2021scorebased}. When $\alpha=1$, the dynamics recover the usual variance-preserving (VP) family underlying diffusion probabilistic models \citep{sohldickstein2015deep,ho2020denoising}.

\paragraph{Backward process and generative model.}

Under mild regularity assumptions \citep{anderson1982reverse, haussmann1986time}, the process defined in \eqref{eq:forward_sde} admits a time-reversed process $(\ora X_t)_{t\in[0,T]} \stackrel{\mathcal L}{=} (\ola X_{T-t})_{t\in[0,T]}$,
governed by $\ola X_0 \sim p_T$ and
% \begin{equation}
% \label{eq:backward_sde}
% \rmd \ola X_t
% =
% \Big(
% \alpha \beta_{T-t} \ola X_t
% +
% 2 \beta_{T-t} \nabla \log p_{T-t}(\ola X_t)
% \Big)\rmd t
% +
% \sqrt{2 \beta_{T-t}} \rmd  B_t \eqsp, \qquad \ola X_0 \sim p_T
% \end{equation}
% \begin{equation}
% \label{eq:backward_sde}
% \begin{aligned}
% \rmd \ola X_t
% ={}&
% \beta_{T-t}
% \Big(
% \alpha \ola X_t
% +
% 2 \nabla \log p_{T-t}(\ola X_t)
% \Big)\rmd t
% \\
% &+
% \sqrt{2\beta_{T-t}}\,\rmd B_t,
% \qquad
% \ola X_0 \sim p_T \eqsp.
% \end{aligned}
% \end{equation}
\begin{equation}
\label{eq:backward_sde}
\rmd \ola X_t
=
\bar{\beta}_t
\Big(
\alpha \ola X_t
+
2 \score_{T-t}(\ola X_t)
\Big)\rmd t
+
\sqrt{2\bar{\beta}_t}\,\rmd B_t \eqsp,
\end{equation}
where $\score_t(x) := \nabla \log p_t(x)$,
$\bar{\beta}_t := \beta_{T-t}$, and $p_t$ denotes the density of $\ora X_t$. The reverse-time dynamics in \eqref{eq:backward_sde} turn score estimation into a principled generative mechanism. Indeed, if one has access to the exact score $\nabla \log p_t$ and could sample from the terminal law $p_T$, then simulating the reverse process from $\ola X_0\sim p_T$ would recover the data distribution
exactly at time $T$, i.e. $\ola X_T\sim\pi_{\rm data}$. Modern SGMs can therefore be understood as approximate implementations of this ideal reverse-time sampler. Their accuracy is governed by three main sources of error: numerical discretization of the reverse SDE \eqref{eq:backward_sde}, mismatch between the forward terminal law $p_T$ and the
reference distribution $\pi_\infty$, and score learning error, namely the
discrepancy between the learned score approximation and the true score function $(x,t) \to \score_{T-t}(x)$ . The last
term itself combines generalization error and optimization error.

\paragraph{Score matching.} To learn the score function, one introduces a parametric model $s_\theta:\R^d\times[0,T]\to\R^d$, typically a deep neural network, and fits it by minimizing a time-averaged Fisher divergence 
\begin{align} \label{eq:score_matching}
\mathcal L(\theta)
&\eqdef
\frac12 \int_0^T  \lambda(t)
\E \left[
\|\score_\theta(\ora X_t,t)-\score_t(\ora X_t)\|^2
\right] \rmd t \eqsp, 
\end{align}  
where $\lambda:[0,T]\to\R_{>0}$ is a weighting function.
Since the marginal score $\score_t$ is generally intractable as it depends on $\pi_{\rm data}$, diffusion models are instead trained using conditional score matching, replacing $\score_t(\ora X_t)$ with the explicit conditional score $\score_{t|0}(\ora X_t\mid X_0)=\nabla\log p_{t|0}(\ora X_t\mid X_0)$.

For stochastic optimization, it is convenient to express the score matching objective in \eqref{eq:score_matching} using the conditional score and rewrite the time integral as an expectation. Let $q$ be a probability density on $(0,T]$, and let $\tau\sim q$ be independent of $(X_0,Z)$. Under standard square-integrability assumptions, 
\begin{align*}
\mathcal L(\theta)
&=
\frac12\,
\E\!\left[
\frac{\lambda(\tau)}{q(\tau)}
\|\score_\theta(\ora X_\tau,\tau)-\score_{\tau|0}(\ora X_\tau\mid X_0)\|^2
\right] - \mathcal C_\lambda \eqsp,
\end{align*}
where $\mathcal C_\lambda\geq 0$ does not depend on $\theta$. Hence, score matching and conditional score matching have the same minimizers \citep{Vincent}. The choice of $q$ does not change the population objective and only determines how time indices are sampled during training.

\section{Convergence for General Score Parameterizations}
\label{sec:general_sgd}
%\section{Convergence Analysis of Standard SGD}

In this section, we derive convergence rates of SGD for the weighted score matching objective under general score parameterizations. SGD generates a sequence of parameter iterates as follows: let $\theta^{(0)}\in\Theta$, and for all $k\in\mathbb{N}$,
\begin{equation}\label{eq:SGD}
\theta^{(k+1)} = \theta^{(k)} - \gamma_{k+1} \widehat{\nabla}_{\theta}\mathcal L \left(\theta^{(k)};\mathcal{D}^{(k+1)}\right),
\end{equation}
where $\{\gamma_k\}_{k\ge 1}$ is a sequence of positive step sizes and $\widehat{\nabla}_{\theta}\mathcal L(\theta^{(k)};\mathcal{D}^{(k+1)})$ is a stochastic estimator of the gradient computed from the mini-batch
\[
\mathcal{D}^{(k+1)}=\big\{(X^{(k+1)}_{0,b}, t^{(k+1)}_{b}, Z^{(k+1)}_{b})\big\}_{b=1}^B \eqsp,
\]
with $X^{(k+1)}_{0,b} \sim \pi_{\rm data}$, $t^{(k+1)}_{b} \sim q$, and $Z^{(k+1)}_{b} \sim \mathcal N(0,\Id_d)$ independently for $b=1,\dots,B$.
The resulting stochastic gradient estimator is given by
\begin{equation}\label{eq:gradient_estimator}
\widehat{\nabla}_{\theta}\mathcal L(\theta^{(k)};\mathcal{D}^{(k+1)})
= \frac{1}{B}\sum_{b=1}^B g^{(k+1)}_{b} \eqsp,
\end{equation}
where $g^{(k+1)}_{b}$ denotes the stochastic gradient associated with a single sample, defined in~\eqref{eq:gradient_estimator_sample}. The full procedure is detailed in Algorithm~\ref{alg:sgd}.
In what follows, we assume that 
$\sup_{t\in(0,T]}\lambda(t)/q(t)<\infty$
and $t\mapsto\lambda(t)/\sigma_t^2$ is integrable on $(0,T]$.
These conditions can be ensured by early stopping along the diffusion time horizon, which is commonly used in the literature \citep{chen2023improved} or by making appropriate design choices for the functions $\beta_t$ and $\lambda(t)$.
We establish the convergence guarantees under the following regularity assumptions on the score network.

\begin{assumption}[Score regularity]
\label{ass:score_lipschitz_smooth}
Assume that there exist measurable, square-integrable functions $L_\ell,L_s:\R^d\times[0,T]\to[0,\infty)$ such that, for all $\theta,\theta'\in\Theta$ and all $(x,t)\in\R^d\times[0,T]$,
\[
\|s_\theta(x,t)-s_{\theta'}(x,t)\|
\le L_\ell(x,t) \|\theta-\theta'\| \eqsp,
\]
and
$$
\eqsp \|\nabla_\theta s_\theta(x,t)-\nabla_\theta s_{\theta'}(x,t)\|
\le L_s(x,t) \|\theta-\theta'\| \eqsp.
$$
\end{assumption}

Assumption~\ref{ass:score_lipschitz_smooth} imposes data-dependent Lipschitz and smoothness conditions on the score network, which are weaker than requiring uniform global bounds. Similar assumptions on the score function appear in the analysis of variational autoencoders \citep{surendran2025theoretical} and policy-gradient methods \citep{papini2018stochastic}. In particular, \citet{surendran2025theoretical} show that neural networks with bounded weights and standard smooth activations, such as sigmoid, tanh, and softplus, a smooth approximation to the ReLU, satisfy such conditions. However, Assumption~\ref{ass:score_lipschitz_smooth} only controls the sensitivity of the network with respect to the parameters; it does not impose a uniform bound on the network output itself. This motivates the additional growth condition introduced next.

\begin{assumption}[Polynomial growth]
\label{ass:score_poly_growth}
%Let $\Theta \subset \mathbb{R}^{d_\theta}$.
Assume that there exist constants
$a,b\ge 0$ and $p\ge 1$ such that, for all $\theta\in\Theta$ and all
$(x,t)\in\R^d\times[0,T]$,
\[
\|s_\theta(x,t)\| \le a + b\|x\|^p \eqsp.
\]
Moreover, assume that the data distribution $\pi_{\rm data}$ has a finite $2p$-th moment.
\end{assumption}

Assumption~\ref{ass:score_poly_growth} controls the growth of the score network with respect to the input, requiring its magnitude to increase at most polynomially in $\|x\|$. This is a natural stability condition: many theoretical analyses of score-based diffusion models impose regularity or growth assumptions on the true score, and in common settings such as Gaussian or log-concave targets the score has at most linear growth in $\|x\|$ \citep{cole2024score, gao2025wasserstein,strasman2025an}. Together with Assumption~\ref{ass:score_lipschitz_smooth}, this condition ensures that the score matching objective is smooth, which is essential for establishing non-asymptotic convergence guarantees for SGD.

% \stan{peut être définir $\theta^\star$ ici non ? \textcolor{red}{Remplacer par $\mathrm{inf}_{\theta\in\Theta}\mathcal L_{\rm DSM}(\theta)$ qui ne nécessite pas de supposer l'existence de $\theta^*$ ? }}
% \sobihan{On peut simplement mettre aussi $\theta^{(n+1)}$}
\begin{theorem}
\label{th:rate_general_SGD}
Assume that Assumptions~\ref{ass:score_lipschitz_smooth} and~\ref{ass:score_poly_growth} hold. Let $(\theta^{(n)})_{n\ge 0}$ be the iterates of the SGD recursion~\eqref{eq:SGD} with
step size $\gamma_{n} = C_{\gamma}n^{-1/2}$, where $0 < C_{\gamma} \leq 1/L$.
Assume that for all $n \in \mathbb{N}$, $\tau^2 = \sup_{\theta \in \Theta} \E [\| g^{(n)} - \nabla_{\theta} \mathcal L\|^2] < +\infty$.
For all $n \geq 1$, let $J \in \{0, \ldots, n\}$ be a uniformly distributed random variable. Then,
\begin{align*}
\E\left[\left\| \nabla_\theta \mathcal L(\theta^{(J)}) \right\|^{2}\right] &\leq \frac{2 \mathcal L^{\ast}  +  L C^2_{\gamma} \tau^2 \log(\rme(n+1))/B}{C_{\gamma} \sqrt{n+1}} \eqsp,
\end{align*}
where $\mathcal L^{\ast} = \mathcal L\left(\theta^{(0)}\right) - \E\left[\mathcal L\left(\theta^{(n+1)}\right) \right]$ and $L$ is defined in \eqref{eq:smooth_L_def}.
\end{theorem} 

Theorem~\ref{th:rate_general_SGD} recovers the standard non-asymptotic
$\mathcal O(\log n/\sqrt n)$ convergence rate for nonconvex stochastic
optimization, while making explicit the dependence on the mini-batch size,
the stochastic gradient variance, and the time-dependent weighting of the score matching objective.

For a fixed time $t\in[0,T]$, an explicit contribution to the bound in
Theorem~\ref{th:rate_general_SGD} induced by the time reweighting scales as $\lambda(t)/(\sigma_t^2 q(t))$, where $\lambda(t)$ is the weighting function and $q(t)$ is the time-sampling density. This factor reveals a trade-off between the small- and large-noise regimes in the optimization error. In the case of uniform time sampling, if $\lambda(t)$ grows more slowly than $\sigma_t^2$, then this factor becomes large as $\sigma_t \to 0$. This leads to a deterioration of the bound in the low-noise regime, where accurate score estimation is particularly important.
On the other hand, if $\lambda(t)$ grows faster than $\sigma_t^2$, this contribution increasingly emphasizes the large-noise regime. At the same time, the objective places relatively more weight on large-noise levels, potentially reducing the relative emphasis on the low-noise regime. The choice $\lambda(t)=\sigma_t^2$ therefore balances these two regimes by keeping the reweighting factor controlled across time. This provides an optimization-based motivation for this weighting choice, commonly used in SGMs \citep{song2021scorebased}.

However, this result only guarantees convergence in terms of the expected squared gradient norm and does not directly control the score-approximation error, which is one of the three components of the final sampling error. Moreover, although the assumptions cover a broad class of smooth neural network architectures, they exclude commonly used non-smooth activations such as ReLU. To address these limitations, the next section focuses on overparameterized two-layer ReLU networks and uses an NTK analysis to derive direct error bounds along the SGD trajectory.

\section{Overparameterized Two-Layer Networks: An NTK Approach}

% Since this conditional expectation is not available in closed form, we use the observed variable $X_0$ as an unbiased noisy version of $\E[X_0 | X_t]$. Define
% $$
% \xi_t \eqdef X_0-f^\star(\ora X_t,t) \qquad \text{so that} \qquad \E[\xi_t\mid \ora X_t,t]=0 \eqsp.
% $$
% Then, the weighted conditional score matching objective in \eqref{eq:objective_function} can be rewritten as
% \begin{align*}
% \mathcal L_{\rm train} (\theta)
% & \eqdef  \frac12 \E \left[
% c(t) \|\net(\ora X_t,t)- (f^*(\ora X_t,t) +\xi_t)\|^2
% \right] \\
% & = \frac12 \E \left[
% c(t) \|\net(\ora X_t,t)-X_0\|^2
% \right] \eqsp.
% \end{align*}

\subsection{Setting}

We now specialize the score parameterization introduced above to an
overparameterized two-layer ReLU network and analyze the resulting SGD dynamics through a matrix-valued NTK argument. In contrast with standard NTK analyses, diffusion training involves unbounded noised inputs, vector-valued targets and an additional time variable \citep{han2024neural}.

\paragraph{Training objective with noisy target. } 
The score satisfies the identity (see, e.g., Lemma~B.8 of \citet{strasman2025an})
\[
\score_t(x)
=
\frac{1}{\sigma_t^2}
\left(
m_t \E[\ora X_0 \mid \ora X_t=x]-x
\right) \eqsp.
\]
Therefore, the score matching objective~\eqref{eq:score_matching} can be written as
\begin{align}
\mathcal L(\theta)
&=
\frac12 \int_0^T c(t)
\E \left[
\|\fnet(\ora X_t,t)- f^\star(\ora X_t,t)\|^2
\right] q(t) \rmd t \eqsp,
\label{eq:objective_function}
\end{align}
where $c(t)\eqdef \lambda(t)m_t^2/(q(t)\sigma_t^4)>0$
$,f^\star(x,t)\eqdef \E[X_0\mid \ora X_t=x]$ and
$$
\snet(x,t)
\eqdef
\frac{1}{\sigma_t^2}\bigl(m_t\fnet(x,t)-x\bigr) \eqsp.
$$
Since $f^\star$ is generally unavailable, diffusion models are trained using the noisy target $X_0$ \citep{karras2022elucidating}. To quantify the resulting noise, let $\ora \X\eqdef(\ora X_\tau,\tau)$ with $\tau\sim q$, and define
\[
\xi \eqdef X_0-f^\star(\ora \X),
\quad
\E[\xi\mid \ora \X]=0,
\quad
\tau_c^2 \eqdef \E[c(\tau)\|\xi\|^2] \eqsp,
\]
where $\tau_c^2$ denotes the weighted noise variance.

\paragraph{Neural network parameterization. }

We consider a two-layer fully-connected ReLU network with vector-valued output in $\R^d$. Since the score function depends on both the spatial variable and time, we work with the augmented input $ \x \eqdef (x,t)\in\R^{d} \times (0,T]$. Let $A\in\R^{d\times m}$ denote the fixed output-layer weights, and let $W\in\R^{m\times(d+1)}$ denote the trainable first-layer weights. We consider the network
\begin{equation}
\label{eq:net_def}
\begin{aligned}
\net (\x;W) &\eqdef \frac{1}{\sqrt m}A D_{\x,W} W \x \\
&= \frac{1}{\sqrt m}\sum_{i=1}^m A_{\cdot i}\mathbf 1_{\{W_i^\top \x\ge 0\}}W_i^\top \x
\in\R^d \eqsp,
\end{aligned}
\end{equation}
where, for each $i\in\{1,\dots,m\}$, $W_i\in\R^{d+1}$ denotes the column vector such that the $i$-th row of $W$ is $W_i^\top$, $A_{\cdot i}\in\R^d$ denotes the $i$-th column of $A$, and
$$
D_{\x,W}
\eqdef
\operatorname{diag} \left(
\mathbf 1_{\{W_1^\top \x \ge 0\}},\dots,\mathbf 1_{\{W_m^\top \x \ge 0\}}
\right)\in\R^{m\times m}
$$
is the diagonal activation matrix. Following the standard NTK parameterization, only the first-layer weights $W$ are updated during training, whereas the matrix $A$ is initialized with independent Rademacher entries and then kept fixed.

\paragraph{Stochastic Gradient Descent. }
The SGD update in \eqref{eq:SGD} specialized to this setting is given by
\begin{equation}
\label{eq:SGD_update}
\param[k+1]
=
\param[k]
-
\gamma_k
\widehat{\nabla}_{\param} \mathcal L (\param^{(k)}) \eqsp,
\end{equation}
where $\param[k]$ denotes the $k$-th iterate. The stochastic gradient estimator can be written as
\begin{equation*}
\widehat{\nabla}_W \mathcal L (\param^{(k)}) = c(t^{(k)})\,
\frac{1}{\sqrt m}\,
D_k A^\top \left(\net \left(\ora \X^{(k)};\param^{(k)}\right)-X_0^{(k)}\right) \,(\ora \X^{(k)})^\top \eqsp,
\end{equation*}
where $\ora \X^{(k)} = \bigl(\ora X_{t^{(k)}}^{(k)},t^{(k)}\bigr)$ and 
$$
D_k \eqdef D_{\ora \X^{(k)},\param[k]}
=
\operatorname{diag} \left(
\mathbf 1_{\{ \param_1^{(k)\top}  \ora \X^{(k)}\ge 0\}},
\dots,
\mathbf 1_{\{ \param_m^{(k)\top}  \ora \X^{(k)}\ge 0\}}
\right) \eqsp.
$$
The derivation is provided in Appendix~\ref{proof:sgd_udpates}.

\paragraph{Neural Tangent Kernel.}
For each iteration $k\ge 0$, we define the empirical matrix-valued neural tangent kernel $K_k:\R^{d+1}\times\R^{d+1}\to\R^{d\times d}$ associated with
\eqref{eq:net_def} by
\begin{equation}
\label{eq:empirical_ntk}
K_k(\x,\x')
\eqdef
\sum_{i=1}^m
\nabla_{\param_i}\net(\x;\param[k])
\nabla_{\param_i}\net(\x';\param[k])^\top \eqsp.
\end{equation}
Using the explicit form of the network, this kernel writes as
\begin{equation}
\label{eq:explicit_empirical_ntk}
\begin{aligned}
K_k(\x,\x')
&=
\frac{\x^\top \x'}{m}
A D_{\x,\param[k]}D_{\x',\param[k]}A^\top \eqsp.
&=
\frac{\x^\top \x'}{m}
\sum_{i=1}^m
\mathbf 1_{\{ \param_i^{(k)\top}  \x\ge 0\}}
\mathbf 1_{\{ \param_i^{(k)\top} \x'\ge 0\}}
A_{\cdot i}A_{\cdot i}^\top \eqsp.
\end{aligned}
\end{equation}
At initialization, under the Gaussian initialization of the first-layer weights and independently initialized Rademacher output weights, $K_0$ concentrates around the deterministic limiting kernel
\begin{equation}
\label{eq:limiting_ntk}
K_\infty(\x,\x')
=
(\x^\top \x')
\,\E_{w}
\left[
\mathbf 1_{\{w^\top \x\ge 0\}}
\mathbf 1_{\{w^\top \x'\ge 0\}}
\right] \Id_d \eqsp,
\end{equation}
where $w \sim \mathcal N(0,\Id_{d+1})$.
In the overparameterized regime, the empirical kernel $K_k$ remains close to its infinite-width
limit $K_\infty$ along training, yielding a linearized function-space characterization of the
network dynamics \citep{jacot2018neural, arora2019exact, du2019gradient, allen2019convergence}.

\subsection{Convergence Result}
\label{subsec:results}

Consider the following assumptions.

\begin{assumption}
\label{ass:subgaussian_data}
The data distribution $\pi_{\rm data} \in \mathcal{P}(\R^d)$ admits a density $p_0$ with respect to the Lebesgue measure and has sub-Gaussian tails: there exist constants $C_0,\nu_0>0$ such that, for all $r\ge 0$,
\begin{equation}
\mathbb P(\|X_0\|>r)\le C_0 \exp\ \left(-\frac{r^2}{2 \nu_0^2}\right) \eqsp.
\end{equation}
\end{assumption}

Assumption~\ref{ass:subgaussian_data} is standard in the analysis of SGMs and ensures that the data distribution has sufficiently light tails. This property is crucial for controlling the forward process $\ora X_t$ uniformly over $t\in[0,T]$, since these perturbed samples are used as inputs to the score network throughout the convergence analysis. This assumption is satisfied by a broad class of distributions, including bounded distributions and Gaussian mixtures.

% \stan{Il faut suppose que $c_\infty < \infty$ ou bien on écrit tout en early stopping, sinon $\beta_t$ borné inférieurement mais ça peut compliquer d'autres choses}
% \sobihan{on peut dire juste en bas qu'on suppose $c_\infty < \infty$ et c'est vérifiable avec early stopping?}
Let $\x=(x,t)\in\mathbb{R}^{d} \times [0,T]$ denote the augmented input and let $\ora \X := (\ora X_{\tau},\tau)$ be an independent draw of the random variable sampled for training. Fix a radius $R>0$ and define the localized region
\begin{equation*}
\mathcal{X}_R \eqdef \{(x,t)\in\R^d\times[0,T]:\|x\|\le R\} \eqsp,
\qquad B_R \eqdef \sqrt{R^2+T^2} \eqsp ,
\qquad c_\infty\eqdef \sup_{t\in[0,T]} c(t) \eqsp.
\end{equation*}
% In what follows, we assume that $c_\infty < \infty$. This condition can be ensured by early stopping along the diffusion time horizon, which is commonly used in the literature \citep{chen2023improved} or by making appropriate design choices for the functions $\beta_t$ and $\lambda(t)$.
Throughout this section, we assume that $c_\infty<\infty$, which, as in Section~\ref{sec:general_sgd}, can be ensured by early stopping along the diffusion time horizon.
We also define the localized weighted measure
\begin{equation}
\label{eq:def_mu_R_main}
\mu_R(\rmd x,\rmd t)
\eqdef
c(t)\,q(t)\,p_t(x)\,\mathbf 1_{\{\|x\|\le R\}}\,\rmd x\,\rmd t \eqsp,
\end{equation}
together with the associated norm $\|g\|_{L_2(\mu_R)}^2
\eqdef
\int_{\R^d\times[0,T]} \|g(x,t)\|^2\,\mu_R(\rmd x,\rmd t)$, for every measurable $g:\R^d\times[0,T]\to\R^d$. Define the approximation error at iteration $k$ by
\begin{align} \label{eq:def_pred_err_proj}
\Delta_k(\x) := f^{\star}(\x) - f_{\theta}(\x, \param[k]) \eqsp.
\end{align}
% It follows that the score matching objective after $n$ iterations writes as  
% \begin{align*}
% \mathcal L (\param^{(n)})
% &=
% \frac12 \E \left[c(\tau) \|\Delta_n(\ora \X)\|^2\right] \eqsp.
% \end{align*}
For compactness, let
\[
D_0 \eqdef \|\Delta_0\|_{L_2(\mu_R)}+\tau_c,
\,\,
\zeta_{m,\delta}
\eqdef
\sqrt{\frac{d\log(m)+\log(1/\delta)}{m}}.
\]
The following theorem provides a non-asymptotic bound on the resulting weighted
score approximation error along the SGD trajectory.

\begin{theorem}
\label{th:rate_two_layer_SGD}
Let $f_\theta$ be the neural network defined in~\eqref{eq:net_def}, and let $(\param^{(n)})_{n\ge 0}$ be the sequence generated by the SGD recursion~\eqref{eq:SGD_update}.
Suppose that Assumption~\ref{ass:subgaussian_data} holds and assume that the iterates
are stopped at some horizon $N\geq 1$.
Assume moreover that there exist constants $a,b\geq 0$ and $p\geq 1$ such that, for all $0\leq n\leq N$,
\begin{equation}
\label{eq:polynomial_growth_network}
\|f_\theta(\x,\param[n])\|
\leq
a+b\|x\|^p
\eqsp.
\end{equation}
Assume further that $m\geq\mathfrak m_R$ and that $\gamma_n=C_\gamma/(n+1)$, where $0<C_\gamma<\Gamma_R\wedge1/(2\Lambda_R)$, and $\mathfrak m_R$ and $\Gamma_R$ are defined in \eqref{eq:condition_width_sim} and~\eqref{eq:step_size_condition}, respectively.
Then, with probability at least $1-\delta$ over the initialization, for every
$1\le n\le N$ and every $R> \sigma_T\mu_Z$,
\begin{equation}
\label{eq:two_layer_SGD_bound}
\begin{aligned}
\E_{\sgd} \left[ \sqrt{\mathcal L (\param^{(n)})} \right] 
\le \mathcal E_{\rm spec} + \mathcal E_{\rm ker} + \mathcal E_{\rm sto} + \mathcal E_{\rm tail} \eqsp,
\end{aligned}
\end{equation}
where
\begin{align*}
\mathcal E_{\rm spec}
&\eqdef
\sqrt{2}
\inf_{r\in\mathbb N}
\Bigg\{
\left(
\prod_{k=0}^{n-1}(1-\gamma_k\kappa_r)
\right)
\|\Delta_0\|_{L_2(\mu_R)}
+\mathcal R(\Delta_0,r)
\Bigg\} \eqsp,
\\[0.3em]
\mathcal E_{\rm ker}
&\eqdef
C\mu_R(\mathcal X_R)B_R^2dC_\gamma
\log(N+1)\,
\zeta_{m,\delta}\,D_0 \eqsp,
\\[0.3em]
\mathcal E_{\rm sto}
&\eqdef
C\chi_{m,\delta}^2
\sqrt{c_\infty\mu_R(\mathcal X_R)}
B_R^2C_\gamma e^{\Lambda_R C_\gamma}
\times
\left(
1+\frac{1}{1-2\Lambda_R C_\gamma}
\right)^{1/2}
D_0 \eqsp,
\\[0.3em]
\mathcal E_{\rm tail}
&\eqdef
C_{\rm tail}
\left(\int_0^T 
\exp\left(
-\frac{(R-\sigma_t\mu_Z)^2}
{16(\sigma_t^2\vee m_t^2\nu_0^2)}
\right) q(\rmd t) \right)^{1/2} \eqsp,
\end{align*}
where $(\kappa_k)_{k\ge1}$ denote the eigenvalues of the limiting kernel operator $\mathsf K_\infty$, ordered in non-increasing order, and $\mathcal R(\Delta_0,r)$ is the spectral remainder of the initial error $\Delta_0$ defined in~\eqref{eq:projection_K_infty}. Moreover, $\chi_{m,\delta}$ and $C_{\rm tail}$ are defined in \eqref{eq:chi_def} and~\eqref{eq:tail_constant}, respectively. Finally, $\mu_Z \eqdef \E \| Z \|$ for $Z \sim \mathcal{N}(0,\Id_d)$ and $C>0$ denotes a universal constant, independent of $n,m,d,R,\delta$ and the initialization.
\end{theorem}

The bound in Theorem~\ref{th:rate_two_layer_SGD} decomposes into four contributions. The first term captures the optimization dynamics along the spectrum of the limiting kernel $\mathsf K_\infty$ and reflects a bias--approximation trade-off through a spectral cutoff parameter $r$. In particular, the product term captures the contraction induced by SGD along the eigendirections, while $\mathcal R(\Delta_0,r)$ quantifies the approximation error associated with projecting onto the leading eigenspaces. The second and third terms quantify the finite-width kernel approximation and stochastic fluctuations, respectively. The kernel approximation term decays with the network width $m$, whereas the stochastic term captures the intrinsic fluctuations induced by SGD and depends explicitly on the variance of the stochastic gradients. Their control requires the width condition in~\eqref{eq:condition_width_sim}, highlighting the role of overparameterization in limiting noise accumulation. Compared with existing NTK analyses for SGD~\citep{zhu2021one} and gradient-descent analyses for diffusion models~\citep{han2024neural}, our result accommodates stochastic gradients and unbounded sub-Gaussian inputs, while improving the dependence on the network width and dimension. The final term arises from the localization step and controls the probability that the noised input leaves the truncated region $\mathcal{X}_R$, which is exponentially small in $R$ under Assumption~\ref{ass:subgaussian_data}.
This contribution is specific to the present localized analysis and replaces the
bounded-input assumptions commonly used in standard NTK treatments. Overall, under suitable width and early stopping conditions, the theorem provides a direct score-approximation bound for SGMs trained via SGD, balancing stochastic fluctuations, finite-width, and localization errors.
%Overall, the theorem shows that, under suitable width and early stopping conditions, SGD for SGMs with overparameterized two-layer ReLU networks achieves a controlled trade-off between stochastic optimization, finite-width, and localization errors. Consequently, it yields a direct bound on the score-approximation error along the training trajectory.

% The last term arises from the localization argument and controls the contribution of samples for which the forward process leaves the ball of radius $R$. Assumption~\ref{ass:subgaussian_data} ensures that this term decays exponentially in $R$ and can therefore be made negligible by choosing $R$ sufficiently large. This contribution is specific to the present localized analysis and does not appear in standard NTK analyses, where the data are typically assumed to be bounded or supported on the unit sphere.

Heuristically, for a fixed time $t\in[0,T]$, the dominant terms in the bound scale as
$c(t) R^2$. If the localization radius is chosen of the same order as the
noise level, i.e., $R\propto \sigma_t$, this suggests the balanced scaling
$c(t)\propto 1/\sigma_t^2$. Similarly to the reweighting analysis in the
previous section, taking $\lambda(t)\propto\sigma_t^2$ and uniform time
sampling, $q(t)\propto 1$, yields, under the present parametrization, $c(t)\propto m_t^2/\sigma_t^2$, which matches this scaling up to the signal factor $m_t^2$.
This corresponds to a signal-to-noise-ratio weighting, which is widely used in diffusion models and has been instrumental in achieving state-of-the-art results in image and video generation \citep{dhariwal2021diffusion,rombach2022high,ho2022video,hang2023efficient}. Different works use different parametrizations, for instance predicting the
clean data, the noise, or the score directly. These choices modify the prefactor appearing in front of the loss, but the underlying choice $\lambda(t)=\sigma_t^2$ remains a common and natural reweighting. It is also worth noting that \citet{han2024neural} analyze the case $c(t)=1$,
which corresponds in our parametrization to $\lambda(t)=\sigma_t^4$ and matches the unweighted MSE objective used in their NTK analysis. In Section~\ref{sec:experiments}, this choice yields weaker empirical performance in our experiment.
Our analysis therefore provides an optimization-based justification for the
widely used signal-to-noise-ratio weighting scheme.

% \paragraph{Sketch of proof.} Because the noised input $\ora X_t=m_tX_0+\sigma_t Z$ is unbounded, we first localize the analysis to a truncated region $\mathcal{X}_R$, decomposing it into a localized $L_2(\mu_R)$ error and an exponentially small tail term. The NTK analysis is then carried out on the event that all SGD inputs remain in $\mathcal X_R$. On this event, we derive the recursion
% $$
% \Delta_{k+1}
% =
% (\Id-\gamma_k\mathsf K_k)\Delta_k - v_k + \varepsilon_k,
% $$
% where $\mathsf K_k$ is the empirical NTK, $v_k$ is the SGD noise, and $\varepsilon_k$ is the nonlinear linearization error. Comparing
% $\mathsf K_k$ with its infinite-width limit $\mathsf K_\infty$ yields four
% terms: contraction, kernel drift, stochastic fluctuation, and nonlinear remainder. The contraction term is controlled spectrally through $\mathsf K_\infty$; the
% stochastic term is handled by martingale estimates; and the kernel drift and
% remainder are controlled by sign-flip bounds together with VC-dimension
% arguments that provide uniform concentration over activation patterns. This
% yields a non-asymptotic bound in which the leading term is the contraction of
% the limiting NTK and the lower-order terms capture finite-width effects, stochastic gradients, and feature drift. A detailed proof is given in Appendix \ref{app:NTK_proof}.

\paragraph{Sketch of proof.}
Since the noised inputs $\ora X_t=m_tX_0+\sigma_tZ$ are unbounded, we
first localize the analysis to $\mathcal X_R$ and control its complement
through sub-Gaussian tail bounds. On the resulting good event, the
prediction error satisfies
\[
\Delta_{k+1}
=
(\Id-\gamma_k\mathsf K_k)\Delta_k-v_k+\varepsilon_k,
\]
where $\mathsf K_k$ is the empirical NTK, $v_k$ is the SGD noise, and
$\varepsilon_k$ is the nonlinear remainder. Introducing the
infinite-width kernel $\mathsf K_\infty$ yields four contributions:
spectral contraction, kernel approximation error, stochastic SGD
fluctuations, and linearization error. The contraction is controlled
through the spectrum of $\mathsf K_\infty$, the stochastic fluctuations
through martingale estimates, and the remaining terms through
initialization concentration and uniform sign-flip bounds based on VC
arguments. Details are deferred to
Appendix~\ref{app:NTK_proof}.

\subsection{Generation Error}
Fix $\epsilon\in(0,T)$ an early stopping parameter and a $K$ discretization grid of $[0,T - \epsilon]$, that is
$0=t_0<\cdots<t_K=T-\epsilon$, with
$h\eqdef\max_k(t_{k+1}-t_k)$.
We denote by $\bar X^{(n,h)}$ the Euler--Maruyama discretization of the
reverse dynamics~\eqref{eq:backward_sde}, initialized from $\pi_\infty$
and with $\score_{T-t}$ replaced by the learned score
$\snet^{(n)}(\cdot,T-t)$ trained via stochastic gradient descent. We study the generation error given by
$
\mathcal W_2\!\left(
\mathcal L(\bar X_{T-\epsilon}^{(n,h)}),\pi_{\rm data}
\right) \eqsp.
$
To quantify error propagation through the learned reverse dynamics, let
\begin{equation} \label{eq:gamma_def}
\Gamma_{n,h}(\epsilon)
\eqdef
\sup_{0\le s\le T-\epsilon}
\exp\!\left(
\int_s^{T-\epsilon}\ell_r^{(n,h)}\,\rmd r
\right)
\eqsp.
\end{equation}
where $\ell^{(n,h)}$ is an admissible coupling-averaged one-sided
Lipschitz coefficient in
Assumption~\ref{ass:mean_stability_sampling} (Appendix~\ref{sec:sampling_error}).
Lemma~\ref{lem:network_osl} guarantees the existence of such an integrable coefficient for the ReLU network \eqref{eq:net_def}, hence $\Gamma_{n,h}(\epsilon)<\infty$.
\begin{theorem}
\label{th:generation_error_main}
Under the assumptions of Lemma~\ref{lem:sampling_error_osl}, conditionally on the trained network,
\begin{multline*}
\mathcal W_2\!\left( \mathcal L(\bar X_{T-\epsilon}^{(n,h)}),\pi_{\rm data} \right)
\le \\
\mathcal E_{\rm early} + \Gamma_{n,h}(\epsilon)\, \mathcal W_2(p_T,\pi_\infty) 
+ C_{n,\epsilon}\sqrt h + \eqsp  2\sqrt2\,\Gamma_{n,h}(\epsilon)\left(\int_\epsilon^T \tfrac{\beta_t^2}{\lambda(t)} \,\rmd t
\right)^{1/2} \sqrt{\mathcal L(\param[n])} \eqsp  \eqsp,
\end{multline*}
where $\mathcal E_{\rm early}^2 = (1-m_\epsilon)^2\E\|X_0\|^2
+d\sigma_\epsilon^2$ and $C_{n,\epsilon}<\infty$ is independent of $h$.
\end{theorem}
The four terms respectively capture early stopping, forward mixing
(initialization mismatch), time discretization, and score learning error. The factor $\Gamma_{n,h}(\epsilon)$ governs the propagation of perturbations through the learned reverse dynamics. This finite-horizon bound does not require contractivity, and neither $\Gamma_{n,h}(\epsilon)$ nor $C_{n,\epsilon}$ is asserted to be uniform in $T$. We next specialize to a more restrictive contractive regime in the Variance-Preserving case, yielding a more interpretable deterministic propagation bound into which the expected score-learning guarantee can be directly inserted.
\begin{corollary}
\label{cor:generation_error_vp}
Suppose that the same conditions as for Theorem \ref{th:rate_two_layer_SGD} hold and that $\alpha=1$, $\beta\equiv1$,
$\pi_\infty=\mathcal N(0,\Id_d)$, and assume moreover that, for some 
$\kappa>0$,
$$
\int_s^{T-\epsilon}\ell_r^{(n,h)}\,\rmd r \le -2 \kappa(T-\epsilon-s), \qquad 0\le s\le T-\epsilon \eqsp.
$$
Then, for a mesh satisfying \eqref{eq:vp_mesh_condition} and $\epsilon\le1$,
\begin{multline*}
E_{\sgd}\!\left[
\mathcal W_2\!\left(
\mathcal L(\bar X_{T-\epsilon}^{(n,h)}),\pi_{\rm data}
\right)
\right]
\lesssim \\
 \sqrt{\epsilon\bigl(\E\|X_0\|^2+d\bigr)} +
\rme^{-T}\mathcal W_2(\pi_{\rm data},\pi_\infty) 
+ \bar C \sqrt h  
 + \left( \int_\epsilon^T \frac{\rme^{-2\kappa(t-\epsilon)}}{\lambda(t)} \,\rmd t \right)^{1/2} \E_{\sgd}\!\left[\sqrt{\mathcal L(\param[n])}\right]
\end{multline*}
where $\bar C <\infty$ and does not depend on grid refinement.
\end{corollary}
All proofs are deferred to Appendix~\ref{sec:proof_generation_main}. The bound separates the four sources of generation error as above. In the
contractive regime, the first two scale respectively as
$\mathcal O(\sqrt{\epsilon})$ and $\mathcal O(\rme^{-T})$, while the
learning contribution is directly controlled by
$\E_{\sgd}[\sqrt{\mathcal L(\param[n])}]$ and can therefore be bounded
using Theorem~\ref{th:rate_two_layer_SGD}. The discretization error is of
order $\sqrt h$. These rates are consistent with existing sampling
guarantees~\citep{strasman2025an,beyler2025convergence}. The role of the contractivity assumption is to control the propagation
of these errors along the reverse dynamics. Importantly, unlike much of the stability literature, the contractivity condition is imposed on the learned reverse dynamics rather than on the unknown true score.

\section{Experiments}
\label{sec:experiments}

We study the effect of the reweighting factor in the score matching loss. This experiment is intended as an empirical illustration of the qualitative effect of the weighting schedule in a standard diffusion architecture. We train a score-based generative model following
\citet{song2021scorebased}, parameterizing the score function with the
NCSN++ architecture \citep{song2021scorebased}. The model is trained on
the Leeds Butterfly dataset \citep{wang2009learning} for $200{,}000$
iterations. Samples are generated using the Euler--Maruyama discretization
with $1{,}000$ steps. %The experiments were conducted using NVIDIA RTX 6000 GPUs with 48 GB of VRAM.

We examine the effect of the loss-weighting schedule $\lambda(t)$ in the score matching objective by considering five choices:
$\lambda(t) \in \{1, \sigma_t, \sigma_t^2, \sigma_t^3, \sigma_t^4\}$ with uniform time sampling.
For each schedule, we train a separate model and evaluate generation 
quality using the Fréchet Inception Distance (FID) \citep{heusel2017gans}, keeping all other training and sampling settings fixed. This allows us to isolate the effect of the loss-weighting schedule on generation quality.

The results are reported in Figure~\ref{fig:fid_butterflies}. Since the optimization error contributes to the score-learning component of the generation error, the choice of $\lambda(t)$ can influence the quality of the learned score function and, consequently, the quality of generated samples.
Among the schedules considered, $\lambda(t) = \sigma_t^2$
achieves the lowest FID. This result is consistent with the qualitative
behavior suggested by our theoretical analysis. Indeed, under uniform
time sampling, the reweighting factor appearing in the bound is proportional to
$\lambda(t)/\sigma_t^2$. For $\lambda(t)=1$ and
$\lambda(t)=\sigma_t$, this factor scales respectively as
$1/\sigma_t^2$ and $1/\sigma_t$, and therefore becomes large in the
low-noise regime as $t\to0$. This amplifies the contribution of small
noise levels to the optimization bound and leads to a deterioration of
the guarantee in this regime, which is consistent with the higher FID
observed for these schedules.
Conversely, for $\lambda(t) = \sigma_t^3$ and $\lambda(t) = \sigma_t^4$, 
the weighting places relatively more emphasis on the high-noise regime, 
i.e., as $t \to T$, concentrating more of the optimization effort on 
large-noise levels. As a result, the relative contribution of small 
noise levels is reduced, leading to a less balanced weighting across 
noise levels. This is consistent with the degraded sample quality 
observed for these schedules compared to $\sigma_t^2$.

\begin{figure}[t]
  \begin{center}
  \includegraphics[width=0.43\textwidth]{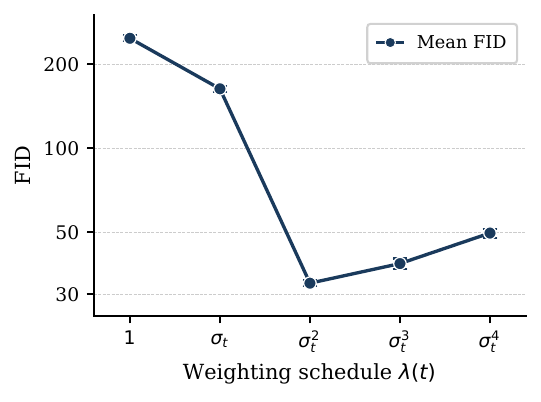} 
  \vskip -0.1in
  \caption{Effect of the loss weighting schedule $\lambda(t)$ on generation quality evaluated by FID (lower is better) on the Leeds Butterfly dataset. Results are averaged over five independent runs, with error bars indicating the standard deviation.}
  \vskip -0.2in
  \label{fig:fid_butterflies}
  \end{center}
\end{figure}

The schedule $\lambda(t) = \sigma_t^2$ provides a balanced trade-off between these two regimes: it avoids overweighting the high-noise regime, unlike $\sigma_t^3$ and $\sigma_t^4$, while compensating for the low-noise singularity, thereby mitigating gradient explosion and keeping the optimization bound controlled. This is consistent with the signal-to-noise ratio weighting discussed in Section~\ref{subsec:results}, where $\lambda(t) = \sigma_t^2$ has been empirically shown to yield state-of-the-art generative performance. Overall, the empirical results are consistent with the theoretical analysis, which suggests that $\lambda(t) = \sigma_t^2$ balances the optimization difficulty across noise levels. While this experiment is illustrative rather than exhaustive, it supports the view that this weighting is a natural choice from an optimization perspective.

\section{Discussion}

This paper studies the optimization and score-approximation errors of score-based generative models trained with SGD. For general score parameterizations, a standard non-convex SGD analysis yields a stationarity guarantee for the weighted score-matching objective, while making explicit the role of the time-dependent reweighting. For overparameterized two-layer ReLU networks, our NTK analysis gives a direct bound on the score-approximation error along the SGD trajectory. This connects the optimization dynamics of denoising score matching to the score-learning term appearing in sampling guarantees, clarifying how training error can affect the final sampling error.
A main implication of our analysis is the role of the reweighting schedule. Our results provide theoretical guidance, from an optimization perspective, for weighting choices used in practice.
One limitation is that our NTK analysis focuses on overparameterized two-layer ReLU networks. Although this setting helps us understand the optimization dynamics of diffusion models, extending the theory to deep U-Net architectures remains an important open question. 

A natural direction for future work is to extend our analysis to kinetic noising processes. Although kinetic SGMs often improve generation empirically \citep{dockhorn2022score}, existing stability analyses do not explain this benefit \citep{conforti_kl,strasman2026wasserstein}. One possible mechanism lies in the learning dynamics. \citet{dockhorn2022score} report smoother learned scores and introduce hybrid score matching (HSM), which preserves the population objective while reducing stochastic-gradient variance. This suggests a possible implicit regularization effect through optimization.

\bibliographystyle{abbrvnat}
\bibliography{ref}

\begin{thebibliography}{48}
\providecommand{\natexlab}[1]{#1}
\providecommand{\url}[1]{\texttt{#1}}
\expandafter\ifx\csname urlstyle\endcsname\relax
  \providecommand{\doi}[1]{doi: #1}\else
  \providecommand{\doi}{doi: \begingroup \urlstyle{rm}\Url}\fi

\bibitem[Allen-Zhu et~al.(2019)Allen-Zhu, Li, and Song]{allen2019convergence}
Z.~Allen-Zhu, Y.~Li, and Z.~Song.
\newblock A convergence theory for deep learning via over-parameterization.
\newblock In \emph{International Conference on Machine Learning}, pages 242--252. PMLR, 2019.

\bibitem[Anderson(1982)]{anderson1982reverse}
B.~D.~O. Anderson.
\newblock Reverse-time diffusion equation models.
\newblock \emph{Stochastic Processes and their Applications}, 12\penalty0 (3):\penalty0 313--326, 1982.

\bibitem[Arora et~al.(2019)Arora, Du, Hu, Li, Salakhutdinov, and Wang]{arora2019exact}
S.~Arora, S.~S. Du, W.~Hu, Z.~Li, R.~R. Salakhutdinov, and R.~Wang.
\newblock On exact computation with an infinitely wide neural net.
\newblock In \emph{Advances in Neural Information Processing Systems}, volume~32, 2019.

\bibitem[Benton et~al.(2024)Benton, Bortoli, Doucet, and Deligiannidis]{benton2024nearly}
J.~Benton, V.~D. Bortoli, A.~Doucet, and G.~Deligiannidis.
\newblock Nearly \$d\$-linear convergence bounds for diffusion models via stochastic localization.
\newblock In \emph{International Conference on Learning Representations}, 2024.

\bibitem[Beyler and Bach(2025)]{beyler2025convergence}
E.~Beyler and F.~Bach.
\newblock Convergence of deterministic and stochastic diffusion-model samplers: A simple analysis in wasserstein distance.
\newblock \emph{arXiv preprint arXiv:2508.03210}, 2025.

\bibitem[Bottou et~al.(2018)Bottou, Curtis, and Nocedal]{bottou2018optimization}
L.~Bottou, F.~E. Curtis, and J.~Nocedal.
\newblock Optimization methods for large-scale machine learning.
\newblock \emph{SIAM {R}eview}, 60\penalty0 (2):\penalty0 223--311, 2018.

\bibitem[Chen et~al.(2023{\natexlab{a}})Chen, Lee, and Lu]{chen2023improved}
H.~Chen, H.~Lee, and J.~Lu.
\newblock Improved analysis of score-based generative modeling: User-friendly bounds under minimal smoothness assumptions.
\newblock In \emph{International Conference on Machine Learning}, volume 202, pages 4735--4763. PMLR, 2023{\natexlab{a}}.

\bibitem[Chen et~al.(2023{\natexlab{b}})Chen, Huang, Zhao, and Wang]{chen2023score}
M.~Chen, K.~Huang, T.~Zhao, and M.~Wang.
\newblock Score approximation, estimation and distribution recovery of diffusion models on low-dimensional data.
\newblock In \emph{International Conference on Machine Learning}, pages 4672--4712. PMLR, 2023{\natexlab{b}}.

\bibitem[Chen et~al.(2023{\natexlab{c}})Chen, Chewi, Li, Li, Salim, and Zhang]{chen2023sampling}
S.~Chen, S.~Chewi, J.~Li, Y.~Li, A.~Salim, and A.~Zhang.
\newblock Sampling is as easy as learning the score: theory for diffusion models with minimal data assumptions.
\newblock In \emph{International Conference on Learning Representations}, 2023{\natexlab{c}}.

\bibitem[Cole and Lu(2024)]{cole2024score}
F.~Cole and Y.~Lu.
\newblock Score-based generative models break the curse of dimensionality in learning a family of sub-gaussian probability distributions.
\newblock In \emph{International Conference on Learning Representations}, 2024.

\bibitem[Conforti et~al.(2025)Conforti, Durmus, and Silveri]{conforti_kl}
G.~Conforti, A.~Durmus, and M.~G. Silveri.
\newblock Kl convergence guarantees for score diffusion models under minimal data assumptions.
\newblock \emph{SIAM Journal on Mathematics of Data Science}, 7\penalty0 (1):\penalty0 86--109, 2025.

\bibitem[Dhariwal and Nichol(2021)]{dhariwal2021diffusion}
P.~Dhariwal and A.~Nichol.
\newblock Diffusion models beat gans on image synthesis.
\newblock In \emph{Advances in Neural Information Processing Systems}, volume~34, pages 8780--8794, 2021.

\bibitem[Dockhorn et~al.(2022)Dockhorn, Vahdat, and Kreis]{dockhorn2022score}
T.~Dockhorn, A.~Vahdat, and K.~Kreis.
\newblock Score-based generative modeling with critically-damped langevin diffusion.
\newblock In \emph{International Conference on Learning Representations}, 2022.

\bibitem[Du et~al.(2019)Du, Lee, Li, Wang, and Zhai]{du2019gradient}
S.~Du, J.~Lee, H.~Li, L.~Wang, and X.~Zhai.
\newblock Gradient descent finds global minima of deep neural networks.
\newblock In \emph{International Conference on Machine Learning}, pages 1675--1685. PMLR, 2019.

\bibitem[Dupuis et~al.(2025)Dupuis, Shariatian, Haddouche, Durmus, and Simsekli]{dupuisalgorithm}
B.~Dupuis, D.~Shariatian, M.~Haddouche, A.~O. Durmus, and U.~Simsekli.
\newblock Algorithm-and data-dependent generalization bounds for diffusion models.
\newblock In \emph{Advances in Neural Information Processing Systems}, 2025.

\bibitem[Fu and Lee(2025)]{fu2025approximation}
G.~Fu and W.~S. Lee.
\newblock Approximation and generalization abilities of score-based neural network generative models for sub-gaussian distributions.
\newblock In \emph{Advances in Neural Information Processing Systems}, 2025.

\bibitem[Gao et~al.(2025)Gao, Nguyen, and Zhu]{gao2025wasserstein}
X.~Gao, H.~M. Nguyen, and L.~Zhu.
\newblock Wasserstein convergence guarantees for a general class of score-based generative models.
\newblock \emph{Journal of Machine Learning Research}, 26\penalty0 (43):\penalty0 1--54, 2025.

\bibitem[Ghadimi and Lan(2013)]{ghadimi2013stochastic}
S.~Ghadimi and G.~Lan.
\newblock Stochastic first-and zeroth-order methods for nonconvex stochastic programming.
\newblock \emph{SIAM journal on optimization}, 23\penalty0 (4):\penalty0 2341--2368, 2013.

\bibitem[Giraud(2021)]{giraud2021highdim}
C.~Giraud.
\newblock \emph{Introduction to High-Dimensional Statistics}.
\newblock Chapman and Hall/CRC, 2021.

\bibitem[Gong et~al.(2023)Gong, Li, Feng, Wu, and Kong]{gong2022diffuseq}
S.~Gong, M.~Li, J.~Feng, Z.~Wu, and L.~Kong.
\newblock Diffuseq: Sequence to sequence text generation with diffusion models.
\newblock In \emph{International Conference on Learning Representations}, 2023.

\bibitem[Han et~al.(2024)Han, Razaviyayn, and Xu]{han2024neural}
Y.~Han, M.~Razaviyayn, and R.~Xu.
\newblock Neural network-based score estimation in diffusion models: Optimization and generalization.
\newblock In \emph{International Conference on Learning Representations}, 2024.

\bibitem[Hang et~al.(2023)Hang, Gu, Li, Bao, Chen, Hu, Geng, and Guo]{hang2023efficient}
T.~Hang, S.~Gu, C.~Li, J.~Bao, D.~Chen, H.~Hu, X.~Geng, and B.~Guo.
\newblock Efficient diffusion training via min-snr weighting strategy.
\newblock In \emph{Proceedings of the IEEE/CVF international conference on computer vision}, pages 7441--7451, 2023.

\bibitem[Haussmann and Pardoux(1986)]{haussmann1986time}
U.~G. Haussmann and E.~Pardoux.
\newblock Time reversal of diffusions.
\newblock \emph{The Annals of Probability}, 14\penalty0 (4):\penalty0 1188--1205, 1986.

\bibitem[Heusel et~al.(2017)Heusel, Ramsauer, Unterthiner, Nessler, and Hochreiter]{heusel2017gans}
M.~Heusel, H.~Ramsauer, T.~Unterthiner, B.~Nessler, and S.~Hochreiter.
\newblock Gans trained by a two time-scale update rule converge to a local nash equilibrium.
\newblock In \emph{Advances in Neural Information Processing Systems}, volume~30, 2017.

\bibitem[Ho et~al.(2020)Ho, Jain, and Abbeel]{ho2020denoising}
J.~Ho, A.~Jain, and P.~Abbeel.
\newblock Denoising diffusion probabilistic models.
\newblock In \emph{Advances in Neural Information Processing Systems}, 2020.

\bibitem[Ho et~al.(2022)Ho, Salimans, Gritsenko, Chan, Norouzi, and Fleet]{ho2022video}
J.~Ho, T.~Salimans, A.~Gritsenko, W.~Chan, M.~Norouzi, and D.~J. Fleet.
\newblock Video diffusion models.
\newblock In \emph{Advances in Neural Information Processing Systems}, volume~35, pages 8633--8646, 2022.

\bibitem[Jacot et~al.(2018)Jacot, Gabriel, and Hongler]{jacot2018neural}
A.~Jacot, F.~Gabriel, and C.~Hongler.
\newblock Neural tangent kernel: Convergence and generalization in neural networks.
\newblock In \emph{Advances in Neural Information Processing Systems}, volume~31, 2018.

\bibitem[Karras et~al.(2022)Karras, Aittala, Aila, and Laine]{karras2022elucidating}
T.~Karras, M.~Aittala, T.~Aila, and S.~Laine.
\newblock Elucidating the design space of diffusion-based generative models.
\newblock In A.~H. Oh, A.~Agarwal, D.~Belgrave, and K.~Cho, editors, \emph{Advances in Neural Information Processing Systems}, 2022.

\bibitem[Li et~al.(2022)Li, Yang, Chang, Chen, Feng, Xu, Li, and Chen]{li2022srdiff}
H.~Li, Y.~Yang, M.~Chang, S.~Chen, H.~Feng, Z.~Xu, Q.~Li, and Y.~Chen.
\newblock Srdiff: Single image super-resolution with diffusion probabilistic models.
\newblock \emph{Neurocomputing}, 479:\penalty0 47--59, 2022.

\bibitem[Li et~al.(2023)Li, Li, Zhang, and Bian]{li2023generalization}
P.~Li, Z.~Li, H.~Zhang, and J.~Bian.
\newblock On the generalization properties of diffusion models.
\newblock In \emph{Advances in Neural Information Processing Systems}, volume~36, pages 2097--2127, 2023.

\bibitem[Lugmayr et~al.(2022)Lugmayr, Danelljan, Romero, Yu, Timofte, and Van~Gool]{lugmayr2022repaint}
A.~Lugmayr, M.~Danelljan, A.~Romero, F.~Yu, R.~Timofte, and L.~Van~Gool.
\newblock Repaint: Inpainting using denoising diffusion probabilistic models.
\newblock In \emph{Proceedings of the IEEE/CVF Conference on Computer Vision and Pattern Recognition}, pages 11461--11471, 2022.

\bibitem[Papini et~al.(2018)Papini, Binaghi, Canonaco, Pirotta, and Restelli]{papini2018stochastic}
M.~Papini, D.~Binaghi, G.~Canonaco, M.~Pirotta, and M.~Restelli.
\newblock Stochastic variance-reduced policy gradient.
\newblock In \emph{International Conference on Machine Learning}, pages 4026--4035. PMLR, 2018.

\bibitem[Robbins and Monro(1951)]{robbins1951stochastic}
H.~Robbins and S.~Monro.
\newblock A stochastic approximation method.
\newblock \emph{The {A}nnals of {M}athematical {S}tatistics}, pages 400--407, 1951.

\bibitem[Rombach et~al.(2022)Rombach, Blattmann, Lorenz, Esser, and Ommer]{rombach2022high}
R.~Rombach, A.~Blattmann, D.~Lorenz, P.~Esser, and B.~Ommer.
\newblock High-resolution image synthesis with latent diffusion models.
\newblock In \emph{Proceedings of the IEEE/CVF Conference on Computer Vision and Pattern Recognition}, pages 10684--10695, 2022.

\bibitem[Sohl-Dickstein et~al.(2015)Sohl-Dickstein, Weiss, Maheswaranathan, and Ganguli]{sohldickstein2015deep}
J.~Sohl-Dickstein, E.~Weiss, N.~Maheswaranathan, and S.~Ganguli.
\newblock Deep unsupervised learning using nonequilibrium thermodynamics.
\newblock In \emph{International Conference on Machine Learning}, pages 2256--2265. PMLR, 2015.

\bibitem[Song and Ermon(2019)]{song2019generative}
Y.~Song and S.~Ermon.
\newblock Generative modeling by estimating gradients of the data distribution.
\newblock In \emph{Advances in Neural Information Processing Systems}, 2019.

\bibitem[Song et~al.(2021)Song, Sohl-Dickstein, Kingma, Kumar, Ermon, and Poole]{song2021scorebased}
Y.~Song, J.~Sohl-Dickstein, D.~P. Kingma, A.~Kumar, S.~Ermon, and B.~Poole.
\newblock Score-based generative modeling through stochastic differential equations.
\newblock In \emph{International Conference on Learning Representations}, 2021.

\bibitem[St{\'e}phanovitch et~al.(2025)St{\'e}phanovitch, Aamari, and Levrard]{stephanovitch2025generalization}
A.~St{\'e}phanovitch, E.~Aamari, and C.~Levrard.
\newblock Generalization bounds for score-based generative models: a synthetic proof.
\newblock \emph{arXiv preprint arXiv:2507.04794}, 2025.

\bibitem[Strasman et~al.(2025{\natexlab{a}})Strasman, Ocello, Boyer, Corff, and Lemaire]{strasman2025an}
S.~Strasman, A.~Ocello, C.~Boyer, S.~L. Corff, and V.~Lemaire.
\newblock An analysis of the noise schedule for score-based generative models.
\newblock \emph{Transactions on Machine Learning Research}, 2025{\natexlab{a}}.
\newblock ISSN 2835-8856.

\bibitem[Strasman et~al.(2025{\natexlab{b}})Strasman, Surendran, Boyer, Corff, Lemaire, and Ocello]{strasman2026wasserstein}
S.~Strasman, S.~Surendran, C.~Boyer, S.~L. Corff, V.~Lemaire, and A.~Ocello.
\newblock Wasserstein convergence of critically damped langevin diffusions.
\newblock In \emph{Advances in Neural Information Processing Systems}, 2025{\natexlab{b}}.

\bibitem[Strasman et~al.(2026)Strasman, Cardoso, Le~Corff, Lemaire, and Ocello]{strasman2026forgetting}
S.~Strasman, G.~V. Cardoso, S.~Le~Corff, V.~Lemaire, and A.~Ocello.
\newblock On forgetting and stability of score-based generative models.
\newblock \emph{arXiv preprint arXiv:2601.21868}, 2026.

\bibitem[Surendran et~al.(2025)Surendran, Godichon-Baggioni, and Le~Corff]{surendran2025theoretical}
S.~Surendran, A.~Godichon-Baggioni, and S.~Le~Corff.
\newblock Theoretical convergence guarantees for variational autoencoders.
\newblock In \emph{International Conference on Artificial Intelligence and Statistics}, pages 3547--3555. PMLR, 2025.

\bibitem[Vershynin(2026)]{vershynin2026hdp}
R.~Vershynin.
\newblock \emph{High-Dimensional Probability: An Introduction with Applications in Data Science}.
\newblock Cambridge Series in Statistical and Probabilistic Mathematics. Cambridge University Press, 2 edition, 2026.
\newblock ISBN 9781009490641.

\bibitem[Vincent(2011)]{Vincent}
P.~Vincent.
\newblock A connection between score matching and denoising autoencoders.
\newblock \emph{Neural Computation}, 23\penalty0 (7):\penalty0 1661--1674, 2011.

\bibitem[Wang et~al.(2009)Wang, Markert, Everingham, et~al.]{wang2009learning}
J.~Wang, K.~Markert, M.~Everingham, et~al.
\newblock Learning models for object recognition from natural language descriptions.
\newblock In \emph{BMVC}, volume~1, page~2, 2009.

\bibitem[Yakovlev and Puchkin(2025)]{yakovlev2025generalization}
K.~Yakovlev and N.~Puchkin.
\newblock Generalization error bound for denoising score matching under relaxed manifold assumption.
\newblock In \emph{Conference on Learning Theory}, pages 5824--5891. PMLR, 2025.

\bibitem[Zhang et~al.(2026)Zhang, Xu, Zhou, Fazel, and Du]{zhang2025convergence}
Y.~Zhang, W.~Xu, M.~Zhou, M.~Fazel, and S.~S. Du.
\newblock Convergence dynamics of over-parameterized score matching for a single gaussian.
\newblock In \emph{International Conference on Learning Representations}, 2026.

\bibitem[Zhu and Xu(2021)]{zhu2021one}
H.~Zhu and J.~Xu.
\newblock One-pass stochastic gradient descent in overparametrized two-layer neural networks.
\newblock In \emph{International Conference on Artificial Intelligence and Statistics}, pages 3673--3681. PMLR, 2021.

\end{thebibliography}

%%%%%%%%%%%%%%%%%%%%%%%%%%%%%%%%%%%%%%%%%%%%%%%%%%%%%%%%%%%%

\clearpage
\appendix

\startcontents[appendix]

\section*{Contents}
\printcontents[appendix]{}{1}{}

\newpage

\paragraph{Additional notation.}
Throughout this appendix, square integrability of a function $h:\R^d\times[0,T]\to\R$ is understood with respect to the distribution of $(\ora X_\tau,\tau)$.

\section{Details and Proofs for the General SGD Analysis}

\subsection{Score Matching and Conditional Score Matching}

We provide here additional details on the equivalence between the score matching and conditional score matching formulations introduced in Section \ref{sec:background}.
Recall that, under standard square-integrability assumptions,
\[
\score_t(\ora X_t)
=
\E\left[
\cscore{t}(\ora X_t\mid X_0)
\,\middle|\,
\ora X_t
\right].
\]
Therefore, by the Pythagorean identity for conditional expectations,
\begin{align}
\E\left[
\|\score_\theta(\ora X_t,t)
-
\cscore{t}(\ora X_t\mid X_0)\|^2
\right]
&=
\E\left[
\|\score_\theta(\ora X_t,t)
-
\score_t(\ora X_t)\|^2
\right]
+
C_t \eqsp,
\end{align}
where
\[
C_t
\eqdef
\E\left[
\|\cscore{t}(\ora X_t\mid X_0)
-
\score_t(\ora X_t)\|^2
\right]
\]
does not depend on $\theta$. Integrating with respect to the weighting function
$\lambda$ shows that the two objectives differ only by a constant independent
of $\theta$:
\begin{equation}
\label{eq:conditional_score_matching_appendix}
\mathcal L_{\rm C}(\theta)
\eqdef
\frac12
\int_0^T
\lambda(t)
\E\left[
\|\score_\theta(\ora X_t,t)
-
\cscore{t}(\ora X_t\mid X_0)\|^2
\right]
\rmd t
=
\mathcal L(\theta)+C_\lambda,
\end{equation}
where
\[
C_\lambda
\eqdef
\frac12
\int_0^T
\lambda(t)C_t\,\rmd t
\]
is independent of $\theta$. Consequently, the two objectives have the same
minimizers \citep{Vincent}.
Moreover, since
\[
\ora X_t\mid X_0
\sim
\mathcal N(m_tX_0,\sigma_t^2\Id_d),
\]
the conditional score is explicitly given by
\[
\cscore{t}(\ora X_t\mid X_0)
=
-\frac{\ora X_t-m_tX_0}{\sigma_t^2}
=
-\frac{Z}{\sigma_t},
\]
which yields the stochastic gradient estimator used in the following subsection.

\subsection{SGD for Score Matching}

In this section, we provide additional details on the stochastic gradient computation for score matching.
At each iteration $k \in \mathbb{N}_*$, we independently sample
\[
X_0^{(k)} \sim \pi_{\rm data},
\qquad
Z^{(k)} \sim \mathcal N(0,\Id_d),
\qquad
t^{(k)} \sim q \eqsp,
\]
and define
\[
\ora X_{t^{(k)}}^{(k)}
=
m_{t^{(k)}}X_0^{(k)}+\sigma_{t^{(k)}}Z^{(k)},
\qquad
\ora \X^{(k)}
\eqdef
\bigl(\ora X_{t^{(k)}}^{(k)}, t^{(k)}\bigr)\in\R^d\times[0,T] \eqsp.
\]
The stochastic gradient estimator at iteration $k+1$ for a sample $b$ is given by
\begin{align}
\label{eq:gradient_estimator_sample}
g^{k+1}_{b} = \frac{\lambda(t_b^{(k+1)})}{q(t_b^{(k+1)})} \Big(\nabla_{\theta}s_{\theta^{(k)}}(\ora \X^{(k+1)}_{b})\Big)^{\top}
\Big(s_{\theta^{(k)}}(\ora \X^{(k+1)}_{b})+\frac{Z^{(k+1)}_{b}}{\sigma_{t^{(k+1)}_{b}}}\Big) \eqsp.
\end{align}

\begin{algorithm}[H]
\caption{SGD for Score Matching}
\label{alg:sgd}
\begin{algorithmic}[1]
\State \textbf{Inputs:} Number of iterations $n$, step sizes $\{\gamma_{k}\}_{k \geq 1}$, batch size $B$
\State Initialize $\theta^{(0)}$
\For{$k=0,\dots,n-1$}
    \State Sample $X_{0,b}^{(k+1)} \sim \pi_{\rm data}$ for all $1 \le b \le B$.
    \State Sample $t_b^{(k+1)} \sim q$ for all $1 \le b \le B$.
    \State Sample $Z_b^{(k+1)} \sim \mathcal N(0,\Id_d)$ for all $1 \le b \le B$.
    \State Compute $X_{t_b^{(k+1)}}^{(k+1)} = m_{t_b^{(k+1)}}X_{0,b}^{(k+1)} + \sigma_{t_b^{(k+1)}}Z_b^{(k+1)}$ for all $1 \le b \le B$.
    \State Compute the stochastic gradient using \eqref{eq:gradient_estimator} and \eqref{eq:gradient_estimator_sample}.
    \State Update the parameters
    \[
    \theta^{(k+1)} = \theta^{(k)} - \gamma_{k+1} \widehat{\nabla}_{\theta}\mathcal L \left(\theta^{(k)};\mathcal{D}^{(k+1)}\right)
    \]
\EndFor
\end{algorithmic}
\end{algorithm}

\subsection{Proof of Theorem \ref{th:rate_general_SGD}}

\begin{proof}
Let $\mathcal L(\theta)=\mathbb E[\ell(\theta;t,\ora X_t,X_0)]$ and $\tilde c(t)=\lambda(t)/q(t)$ with
\[
\ell(\theta;t,x,x_0) := \frac{\tilde c(t)}{2}
\left\| s_\theta(x,t)-\frac{m_t x_0-x}{\sigma_t^2} \right\|^2 \eqsp.
\]
Then
\[
\nabla_\theta \ell(\theta;t,x,x_0)
=
\tilde c(t)\,(\nabla_\theta s_\theta(x,t))^\top
\left(
s_\theta(x,t)-\frac{m_t x_0-x}{\sigma_t^2}
\right) \eqsp.
\]

For all $\theta,\theta'\in\Theta$, we have:
\[
\begin{aligned}
\nabla_\theta \ell(\theta;t,x,x_0)-\nabla_\theta \ell(\theta';t,x,x_0) &=
\tilde c(t)\Big[ (\nabla_\theta s_\theta(x,t)-\nabla_\theta s_{\theta'}(x,t))^\top
\Big( s_\theta(x,t)-\frac{m_t x_0-x}{\sigma_t^2} \Big) \\
&\qquad\qquad +(\nabla_\theta s_{\theta'}(x,t))^\top
\big(s_\theta(x,t)-s_{\theta'}(x,t)\big) \Big] \eqsp.
\end{aligned}
\]
Hence, by the triangle inequality,
\[
\begin{aligned}
\left\|
\nabla_\theta \ell(\theta;t,x,x_0)-\nabla_\theta \ell(\theta';t,x,x_0)
\right\| &\le \tilde c(t) \left\| \nabla_\theta s_\theta(x,t)-\nabla_\theta s_{\theta'}(x,t) \right\| \left\| s_\theta(x,t)-\frac{m_t x_0-x}{\sigma_t^2}
\right\| \\
&\qquad + \tilde c(t) \left\| \nabla_\theta s_{\theta'}(x,t) \right\| \left\| s_\theta(x,t)-s_{\theta'}(x,t) \right\| \eqsp.
\end{aligned}
\]
Using Assumptions \ref{ass:score_lipschitz_smooth} and \ref{ass:score_poly_growth}, we obtain
\[
\begin{aligned}
&\left\|
\nabla_\theta \ell(\theta;t,x,x_0)-\nabla_\theta \ell(\theta';t,x,x_0)
\right\| \\
&\le \tilde c(t)L_s(x,t) \left(a+b\|x\|^p+\frac{\|x-m_t x_0\|}{\sigma_t^2}\right)\|\theta-\theta'\| + \tilde c(t)L_\ell(x,t)^2\|\theta-\theta'\| \eqsp.
\end{aligned}
\]
Using Young's inequality, for $x=m_t x_0+\sigma_t z$ we have
\[
L_s(x,t)\frac{\|x-m_t x_0\|}{\sigma_t^2}
\le
\frac{L_s(x,t)^2}{2} + \frac{\|x-m_t x_0\|^2}{2\sigma_t^4} = \frac{L_s(x,t)^2}{2} + \frac{\|z\|^2}{2\sigma_t^2}
\eqsp.
\]
Therefore,
\[
\begin{aligned}
&\left\| \nabla_\theta \ell(\theta;t,x,x_0) - \nabla_\theta \ell(\theta';t,x,x_0) \right\| \\
&\le \tilde c(t) \left[ L_s(x,t)\left(a+b\|x\|^p\right) + \frac{L_s(x,t)^2}{2} + \frac{\|z\|^2}{2\sigma_t^2} + L_\ell(x,t)^2 \right] \|\theta-\theta'\| \eqsp.
\end{aligned}
\]
Taking expectations then yields
\[
\|\nabla \mathcal L(\theta) - \nabla \mathcal L(\theta')\| \le L \|\theta-\theta'\| \eqsp,
\]
where
\begin{align}
\label{eq:smooth_L_def}
L &= \mathbb E\!\left[
\tilde c(t) \left( L_s(\ora X_t,t)\left(a+b\|\ora X_t\|^p\right) + \frac{L_s(\ora X_t,t)^2}{2} + L_\ell(\ora X_t,t)^2 + \frac{\|Z\|^2}{2\sigma_t^2} \right) \right] \eqsp.
\end{align}
This proves that $\mathcal L$ is $L$-smooth. We now establish the convergence rate of stochastic gradient descent. Since the update is based on the mini-batch estimator $\widehat{\nabla}_{\theta} \mathcal L(\theta^{(k)};\mathcal{D}^{(k+1)})$, the remainder of the proof follows the argument of Theorem 2.1 in \cite{ghadimi2013stochastic}, adapted to the mini-batch setting, and relies on the unbiasedness and variance control of this estimator.
Using the $L$-smoothness of $\mathcal L$, we have
\begin{align*}
\mathcal L(\theta^{(k+1)})
&\le \mathcal L(\theta^{(k)}) + \left\langle \nabla_\theta \mathcal L(\theta^{(k)}), \theta^{(k+1)}-\theta^{(k)} \right\rangle
+ \frac{L}{2}\|\theta^{(k+1)}-\theta^{(k)}\|^2  \\
&= \mathcal L(\theta^{(k)}) - \gamma_{k+1} \left\langle \nabla_\theta \mathcal L(\theta^{(k)}), \widehat \nabla_\theta \mathcal L(\theta^{(k)};\mathcal D^{(k+1)}) \right\rangle + \frac{L\gamma_{k+1}^2}{2} \left\| \widehat \nabla_\theta \mathcal L(\theta^{(k)};\mathcal D^{(k+1)}) \right\|^2 \eqsp,
\end{align*}
where $\mathcal{D}^{(k+1)}$ corresponds to the mini-batch of data used to compute the gradient estimator at iteration $k+1$. For all $k \geq 0$, let
\[
\mathcal{F}_{k} = \sigma \left( \theta^{(0)}, \{\mathcal{D}^{(i)}\}_{1 \leq i \leq k } \right) = \sigma\left(\theta^{(0)},\left\{\left(X^{(i)}_{0,b},\,t^{(i)}_b,\,Z^{(i)}_b\right)_{b=1}^B\right\}_{1 \leq i \leq k}\right)
\]
be the filtration generated by the initialization and the mini-batches up to iteration $k$.
Taking conditional expectation with respect to $\mathcal F_k$, using the unbiasedness of the mini-batch gradient estimator and the variance bound
\[
\E\left[ \left\| \widehat \nabla_\theta \mathcal L(\theta^{(k)};\mathcal D^{(k+1)}) - \nabla_\theta \mathcal L(\theta^{(k)}) \right\|^2 \middle| \mathcal F_k \right]
\le \frac{\tau^2}{B} \eqsp,
\]
we obtain
\begin{align*}
\E\left[ \mathcal L(\theta^{(k+1)}) \middle| \mathcal F_k \right]
&\le
\mathcal L(\theta^{(k)}) - \gamma_{k+1} \left\| \nabla_\theta \mathcal L(\theta^{(k)}) \right\|^2 + \frac{L\gamma_{k+1}^2}{2} \left\| \nabla_\theta \mathcal L(\theta^{(k)}) \right\|^2 + \frac{L\gamma_{k+1}^2}{2B} \tau^2 .
\end{align*}
Taking expectations and summing over $k=0,\ldots,n$ gives
\begin{align*}
\sum_{k=0}^n \left( \gamma_{k+1} - \frac{L\gamma_{k+1}^2}{2}
\right) \E\left[ \left\| \nabla_\theta \mathcal L(\theta^{(k)}) \right\|^2 \right]
&\le
\mathcal L(\theta^{(0)}) - \E\left[ \mathcal L(\theta^{(n+1)}) \right] + \frac{L\tau^2}{2B} \sum_{k=0}^n \gamma_{k+1}^2 \eqsp.
\end{align*}
Since $\gamma_{k+1}\le 1/L$, we have
\[
\gamma_{k+1}-\frac{L\gamma_{k+1}^2}{2} \ge \frac{\gamma_{k+1}}{2} \eqsp.
\]
It follows that
\[
\sum_{k=0}^n \gamma_{k+1} \E\left[ \left\| \nabla_\theta \mathcal L(\theta^{(k)}) \right\|^2 \right]
\le 2\left( \mathcal L(\theta^{(0)}) - \E[ \mathcal L(\theta^{(n+1)}) ] \right) + \frac{L\tau^2}{B} \sum_{k=0}^n \gamma_{k+1}^2 \eqsp.
\]
Since the step-size sequence $(\gamma_{k+1})_{k\geq0}$ is non-increasing, we have
\[
\gamma_{n+1} \sum_{k=0}^n
\mathbb{E}\left[ \left\| \nabla_\theta \mathcal L (\theta^{(k)}) \right\|^2 \right]
\le
\sum_{k=0}^n \gamma_{k+1}\mathbb{E}\left[ \left\| \nabla_\theta \mathcal L(\theta^{(k)}) \right\|^2 \right] \eqsp.
\]
Therefore, let $J \in \{0, \ldots, n\}$ be a uniformly distributed random variable,
independent of the SGD randomness. Then,
\[
\mathbb{E}\left[ \left\| \nabla_\theta \mathcal L(\theta^{(J)}) \right\|^2 \right]
= \frac{1}{n+1} \sum_{k=0}^n \mathbb{E}\left[\left\| \nabla_\theta \mathcal L(\theta^{(k)}) \right\|^2 \right] \eqsp.
\]
Combining the previous inequalities yields
\[
\E\left[ \left\| \nabla_\theta \mathcal L(\theta^{(J)}) \right\|^2
\right]
\le \frac{2\left( \mathcal L(\theta^{(0)}) - \E[ \mathcal L(\theta^{(n+1)}) ] \right) + \frac{L\tau^2}{B} \sum_{k=0}^n \gamma_{k+1}^2}{(n+1)\gamma_{n+1}} \eqsp.
\]
Using $\gamma_{k+1}=C_\gamma(k+1)^{-1/2}$ yields
\[
(n+1)\gamma_{n+1} = C_\gamma\sqrt{n+1},
\qquad
\sum_{k=0}^n \gamma_{k+1}^2 \le C_\gamma^2 \log(\rme(n+1)) \eqsp.
\]
Hence,
\[
\E\left[ \left\| \nabla_\theta \mathcal L(\theta^{(J)}) \right\|^2 \right]
\le \frac{2\left( \mathcal L(\theta^{(0)}) - \E [ \mathcal L(\theta^{(n+1)}) ] \right) + L C_\gamma^2\tau^2\log(\rme(n+1))/B}{C_\gamma\sqrt{n+1}} \eqsp.
\]
It remains to verify that $L<\infty$. Since $\widetilde c(t)=\lambda(t)/q(t)$ is uniformly bounded, Assumptions~\ref{ass:score_lipschitz_smooth} and~\ref{ass:score_poly_growth}, together with Cauchy--Schwarz, ensure the integrability of the terms involving $L_s$, $L_\ell$, and $\|\ora X_t\|^p$ in~\eqref{eq:smooth_L_def}.
Moreover,
\[
\E\left[
\frac{\lambda(\tau)}{q(\tau)}
\frac{\|Z\|^2}{\sigma_\tau^2}
\right]
=
d\int_0^T \frac{\lambda(t)}{\sigma_t^2}\,\rmd t
<\infty.
\]
Hence, $L<\infty$, which concludes the proof.
% which concludes the proof.
% where we used
% $\mathcal L(\theta^\star)\le
% \mathcal L(\theta^{(n+1)})$.
\end{proof}

\section{Proof of the main NTK theorem} \label{app:NTK_proof}

\paragraph{Probability space and conditioning.}
The analysis involves three sources of randomness: the initialization
$(A,\param^{(0)})$, the SGD sample iterations
$$
\X^{(k)}:=(X_0^{(k)},Z^{(k)},t^{(k)}), \qquad k\ge 0 \eqsp,
$$
and independent fresh samples drawn from $\ora \X \eqdef (\ora X_{\tau}, \tau)$
used to evaluate the population quantities. Here, $\ora \X$ has the same law as a
training sample input, but is independent of both the initialization and the
SGD sample stream. To avoid overloading notation, random variables sampled
during SGD carry the iteration index in superscript, whereas fresh independent
samples used only for evaluation are denoted with an arrow. We also introduce the corresponding filtrations, denoted by
$$
\mathcal I \eqdef \sigma(A,\param^{(0)}) \eqsp,
\qquad
\mathcal F_n \eqdef \sigma(\X^{(0)},\dots,\X^{(n-1)}) \eqsp,
\qquad
\mathcal G_n \eqdef \sigma \left( \mathcal I \cup \mathcal F_n \right) \eqsp.
$$
It follows that, for a fixed trained parameter $\param^{(n)}$, the score matching loss can be written as
\[
\mathcal L(\param^{(n)})
\eqdef
\frac12 \E \left[
c(\tau) 
\|\Delta_n( \ora \X)\|^2
\,\middle|  \mathcal G_n
\right] \qquad
\ora \X=(\ora X_\tau,\tau) \eqsp.
\]
Throughout the rest of the proof, as in \citet{zhu2021one}, we condition implicitly on $\mathcal{I}$. When it is clear from the context, expectations are to be understood as with respect to the SGD sample only, or we specify
\[
\E_{\sgd} [ \cdot ] = \E [ \cdot | \mathcal I ] \eqsp.
\]

\paragraph{Localized score matching loss.}

We first reduce the score matching loss to a localized $L_2$ error plus a tail
remainder. We introduce the truncated weighted measure
\begin{equation} \label{eq:def_mu_r}
\mu_R(\rmd \x)
\eqdef
c(t) q(t) p_t(x) \mathbf 1_{\{\|x\|\le R\}} \rmd x \rmd t \eqsp , 
\qquad \x=(x,t) \in \R^{d} \times [0,T] \eqsp.
\end{equation}
We define
\begin{align}
\label{eq:def_B_R}
\mathcal{X}_R
\eqdef
\{(x,t)\in\R^d\times[0,T]:\|x\|\le R\},
\qquad
B_R\eqdef \sqrt{R^2+T^2} \eqsp, 
\end{align}
and we let
\begin{align*}
\mu_R(\mathcal{X}_R) \eqdef \int_0^T c(t) q(t) \mathbb P(\|\ora X_t\|\le R) \rmd t \eqsp.%= \E \left[ c(\tau) \mathbf 1_{\{\|\ora X_{\tau} \|\le R\}} \right] \eqsp,
\end{align*}

\begin{lemma}
For all $n\ge 0$,
\begin{equation}
\label{eq:loss_reduction}
\sqrt{\mathcal L(\param^{(n)})}
\le
\frac{1}{\sqrt{2}}\|\Delta_n\|_{L_2(\mu_R)}
+
\left(
\frac{1}{2}
\int_0^T
c(t)q(t)\,
\E\left[
\|\Delta_n(\ora \X_t)\|^2
\mathbf 1_{\{\|\ora X_t\|>R\}}
\right]
\rmd t
\right)^{1/2}
\eqsp.
\end{equation}
\end{lemma}

\begin{proof}
By definition,
\begin{align*}
2\mathcal L(\param^{(n)})
&=
\E\left[
c(\tau)\|\Delta_n(\ora \X)\|^2
\mathbf 1_{\{\|\ora X_\tau\|\leq R\}}
\right]
+
\E\left[
c(\tau)\|\Delta_n(\ora \X)\|^2
\mathbf 1_{\{\|\ora X_\tau\|>R\}}
\right]
\\
&=
\|\Delta_n\|_{L_2(\mu_R)}^2
+
\E\left[
c(\tau)\|\Delta_n(\ora \X)\|^2
\mathbf 1_{\{\|\ora X_\tau\|>R\}}
\right]
\eqsp.
\end{align*}
Using $\sqrt{a+b}\leq\sqrt a+\sqrt b$ for $a,b\geq0$, we obtain
\begin{align*}
\sqrt{2\mathcal L(\param^{(n)})}
&\leq
\|\Delta_n\|_{L_2(\mu_R)}
+
\left(
\E\left[
c(\tau)\|\Delta_n(\ora \X)\|^2
\mathbf 1_{\{\|\ora X_\tau\|>R\}}
\right]
\right)^{1/2}\\
&=
\|\Delta_n\|_{L_2(\mu_R)}
+
\left(
\int_0^T
c(t)q(t)\,
\E\left[
\|\Delta_n(\ora \X_t)\|^2
\mathbf 1_{\{\|\ora X_t\|>R\}}
\right]
\rmd t
\right)^{1/2}
\eqsp.
\end{align*}
Dividing both sides by $\sqrt{2}$ yields \eqref{eq:loss_reduction}.
\end{proof}

\begin{lemma}
Suppose Assumption~\ref{ass:subgaussian_data} holds and assume that there
exist constants $a,b\geq 0$ and $p\geq 1$ such that, for all
$0\leq n\leq N$,
\begin{equation}
\label{eq:lemma_polynomial_growth_network}
\|f_\theta(\x,\param[n])\|
\leq
a+b\|x\|^p
\eqsp.
\end{equation}
For all $r>0$, let
\begin{equation}
\label{eq:M0_r}
M_{0,r}
:=
C_0(2\nu_0^2)^{r/2}
\Gamma\left(\frac r2+1\right)
\eqsp,
\end{equation}
and, for all $r\geq 1$ and $t\in[0,T]$, let
\begin{equation}
\label{eq:Mt_r}
M_{t,r}
:=
2^{r-1}
\left(
m_t^r M_{0,r}
+
2^{r/2}\sigma_t^r
\frac{\Gamma((d+r)/2)}{\Gamma(d/2)}
\right)
\eqsp.
\end{equation}
Then, for all $0\leq n\leq N$ and $R>\sigma_T\mu_Z$,
\begin{align}
&\int_0^T
c(t)q(t)\,
\E\left[
\|\Delta_n(\ora \X_t)\|^2
\mathbf 1_{\{\|\ora X_t\|>R\}}
\right]
\rmd t
\nonumber\\
&\qquad\leq
\sqrt{1+C_0}
\int_0^T
c(t)q(t)
\left[
2\sqrt{M_{0,4}}
+
4\sqrt{2}\,
\sqrt{a^4+b^4M_{t,4p}}
\right]
\exp\left\{
-\frac{(R-\sigma_t\mu_Z)^2}
{16\max(\sigma_t^2,m_t^2\nu_0^2)}
\right\}
\rmd t
\eqsp.
\label{eq:population_tail_explicit}
\end{align}
\end{lemma}

\begin{proof}
We first establish the moment bounds. For all $r>0$, Assumption~\ref{ass:subgaussian_data} yields
\begin{align*}
\E[\|X_0\|^r]
&=
r\int_0^\infty
s^{r-1}\mathbb P(\|X_0\|>s)\rmd s
\\
&\leq
C_0 r
\int_0^\infty
s^{r-1}
\exp\left(-\frac{s^2}{2\nu_0^2}\right)
\rmd s
\\
&=
C_0(2\nu_0^2)^{r/2}
\Gamma\left(\frac r2+1\right) = M_{0,r}
\eqsp.
\end{align*}
Moreover, since $\ora X_t=m_t X_0+\sigma_t Z$ with $Z\sim\mathcal N(0,\Id_d)$, we have, for all $r\geq 1$,
\begin{align*}
\E[\|\ora X_t\|^r]
&\leq
2^{r-1}
\left(
m_t^r\E[\|X_0\|^r]
+
\sigma_t^r\E[\|Z\|^r]
\right) \eqsp, \\
\E[\|Z\|^r]
&=
2^{r/2}
\frac{\Gamma((d+r)/2)}{\Gamma(d/2)}
\eqsp.
\end{align*}
Thus,
\begin{align*}
\E[\|\ora X_t\|^r]
&\leq
2^{r-1}
\left(
m_t^rM_{0,r}
+
2^{r/2}\sigma_t^r
\frac{\Gamma((d+r)/2)}{\Gamma(d/2)}
\right) = M_{t,r}
\eqsp.
\end{align*}
We now control the tail of the approximation error. Since $\Delta_n(u) = f^\star(u)-f_\theta(u,\param[n])$, we have
\[
\|\Delta_n(u)\|^2
\leq
2\|f^\star(u)\|^2
+
2\|f_\theta(u,\param[n])\|^2
\eqsp.
\]
Therefore,
\begin{align}
&\int_0^T
c(t)q(t)
\E\left[
\|\Delta_n(\ora \X_t)\|^2
\mathbf 1_{\{\|\ora X_t\|>R\}}
\right]\rmd t
\nonumber\\
&\quad\leq
2\int_0^T
c(t)q(t)
\E\left[
\|f^\star(\ora \X_t)\|^2
\mathbf 1_{\{\|\ora X_t\|>R\}}
\right]\rmd t 
+
2\int_0^T
c(t)q(t)
\E\left[
\|f_\theta(\ora \X_t,\param[n])\|^2
\mathbf 1_{\{\|\ora X_t\|>R\}}
\right]\rmd t
\eqsp.
\label{eq:tail_decomp}
\end{align}
For the first term, since $f^\star(\ora \X_t)=\E[X_0\mid\ora X_t]$, conditional Jensen's inequality gives
\[
\|f^\star(\ora \X_t)\|^2
\leq
\E[\|X_0\|^2\mid\ora X_t]
\eqsp.
\]
Since $\mathbf 1_{\{\|\ora X_t\|>R\}}$ is measurable with respect to
$\ora X_t$, it follows that
\begin{align*}
\E\left[
\|f^\star(\ora \X_t)\|^2
\mathbf 1_{\{\|\ora X_t\|>R\}}
\right]
&\leq
\E\left[
\|X_0\|^2
\mathbf 1_{\{\|\ora X_t\|>R\}}
\right]
\\
&\leq
\E[\|X_0\|^4]^{1/2}
\mathbb P(\|\ora X_t\|>R)^{1/2}
\\
&\leq
M_{0,4}^{1/2}
\mathbb P(\|\ora X_t\|>R)^{1/2}
\eqsp,
\end{align*}
where the second inequality follows from Cauchy--Schwarz.
For the second term, using \eqref{eq:lemma_polynomial_growth_network} and Cauchy--Schwarz,
\begin{align*}
\E\left[
\|f_\theta(\ora \X_t,\param[n])\|^2
\mathbf 1_{\{\|\ora X_t\|>R\}}
\right]
&\leq
\E\left[
(a+b\|\ora X_t\|^p)^4
\right]^{1/2}
\mathbb P(\|\ora X_t\|>R)^{1/2}
\eqsp.
\end{align*}
Using $(x+y)^4\leq 8(x^4+y^4)$ for $x,y\geq0$, we obtain
\[
(a+b\|\ora X_t\|^p)^4
\leq
8\left(
a^4+b^4\|\ora X_t\|^{4p}
\right)
\eqsp,
\]
and therefore
\begin{align*}
\E\left[
\|f_\theta(\ora \X_t,\param[n])\|^2
\mathbf 1_{\{\|\ora X_t\|>R\}}
\right]
&\leq
2\sqrt{2}\,
\left(a^4+b^4\E[\|\ora X_t\|^{4p}]\right)^{1/2}
\,
\mathbb P(\|\ora X_t\|>R)^{1/2}
\\
&\leq
2\sqrt{2}\,
\left(
a^4+b^4M_{t,4p}
\right)^{1/2}
\,
\mathbb P(\|\ora X_t\|>R)^{1/2}
\eqsp,
\end{align*}
where the last inequality follows from
$\E[\|\ora X_t\|^{4p}]\leq M_{t,4p}$.
Combining the previous bounds with \eqref{eq:tail_decomp} yields
\begin{align}
&\int_0^T
c(t)q(t)
\E\left[
\|\Delta_n(\ora \X_t)\|^2
\mathbf 1_{\{\|\ora X_t\|>R\}}
\right]\rmd t
\leq
\int_0^T
c(t)q(t)
\left[
2\sqrt{M_{0,4}}
+
4\sqrt{2}\,
\sqrt{a^4+b^4M_{t,4p}}
\right]
\mathbb P(\|\ora X_t\|>R)^{1/2}
\rmd t
\eqsp.
\label{eq:tail_before_subgaussian}
\end{align}
Finally, by Lemma~\ref{lem:tail_bound_X},
\[
\mathbb P(\|\ora X_t\|>R)^{1/2}
\leq
\sqrt{1+C_0}
\exp\left\{
-\frac{(R-\sigma_t\mu_Z)^2}
{16\max(\sigma_t^2,m_t^2\nu_0^2)}
\right\}
\eqsp.
\]
Substituting this inequality into
\eqref{eq:tail_before_subgaussian} yields
\eqref{eq:population_tail_explicit}.
\end{proof}

\paragraph{Localized SGD analysis.} Fix an iteration horizon $n\ge 1$ and a radius $R>0$, and define
\begin{equation}
\label{eq:good_event}
\Omega_{n,R}
:=
\bigcap_{k=0}^{n-1}
\left\{
\|\ora X_{t^{(k)}}^{(k)}\| \le R
\right\} \eqsp.
\end{equation}
By the union bound,
\[
\mathbb{P} (\Omega_{n,R}^c)
\le
\sum_{k=0}^{n-1} \mathbb{P} \left(\|\ora X_{t^{(k)}}^{(k)}\|>R\right)
=
n \int_0^T q(t) \mathbb{P} (\|\ora X_t\|>R) \rmd t \eqsp,
\]
which is controlled by Lemma~\ref{lem:tail_bound_X}. The NTK analysis is carried out on the good event $\Omega_{n,R}$. We split
\begin{align*}
\E_{\sgd}\!\left[\sqrt{\mathcal L(\param^{(n)})}\right]
&=
\E_{\sgd}\!\left[\sqrt{\mathcal L(\param^{(n)})} 
\mathbf 1_{\Omega_{n,R}}\right]
+
\E_{\sgd}\!\left[\sqrt{\mathcal L(\param^{(n)})} 
\mathbf 1_{\Omega_{n,R}^c}\right] \eqsp.
\end{align*}
The first term is controlled by the localized error. Indeed, by
\eqref{eq:loss_reduction},
\begin{align*}
\E_{\sgd} \left[
\sqrt{\mathcal L(\param^{(n)})}\mathbf 1_{\Omega_{n,R}}
\right]
& \le
\frac{1}{\sqrt{2}}
\E_{\sgd}\!\left[
\|\Delta_n\|_{L_2(\mu_R)}
\mathbf 1_{\Omega_{n,R}}
\right] +
\E_{\sgd}\!\left[ \left(
\frac{1}{2}
\int_0^T
c(t)q(t)\,
\E\left[
\|\Delta_n(\ora \X_t)\|^2
\mathbf 1_{\{\|\ora X_t\|>R\}}
\right]
\rmd t
\right)^{1/2} \right]\eqsp.
\end{align*}
To control the complement event, using a similar argument as above, we obtain
\begin{align*}
\E_{\sgd} \left[
\sqrt{\mathcal L(\param^{(n)})}\mathbf 1_{\Omega_{n,R}^c}
\right]
&\le
\widetilde C_{\rm tail}
\int_0^T
q(t)
\exp\left\{
-\frac{(R-\sigma_t\mu_Z)^2}
{16\max(\sigma_t^2,m_t^2\nu_0^2)}
\right\}
\rmd t \eqsp,
\end{align*}
where
\begin{equation}
\label{eq:tail_constant_tilde}
\widetilde C_{\rm tail}
\eqdef
(1+C_0)N\left(\int_0^T
c(t)q(t)
\left[
M_{0,2}
+
2a^2
+
2b^2M_{t,2p}
\right]
\rmd t
\right)^{1/2}
\eqsp.
\end{equation}
Therefore,
\begin{equation}
\begin{aligned} \label{eq:equality_to_finsh_thm_2_layers_bis}
&\E_{\sgd} \left[
\sqrt{\mathcal L(\param^{(n)})}
\right]
\le
\frac{1}{\sqrt{2}}\E_{\sgd} \left[
\|\Delta_n\|_{L_2(\mu_R)} \mathbf 1_{\Omega_{n,R}}
\right] +
\widetilde C_{\rm tail}
\int_0^T
q(t)
\exp\left\{
-\frac{(R-\sigma_t\mu_Z)^2}
{16\max(\sigma_t^2,m_t^2\nu_0^2)}
\right\}
\rmd t \\
& \qquad +
\left(\frac{1}{2}\sqrt{1+C_0}
\int_0^T
c(t)q(t)
\left[
2\sqrt{M_{0,4}}
+
4\sqrt{2}\,
\sqrt{a^4+b^4M_{t,4p}}
\right]
\exp\left\{
-\frac{(R-\sigma_t\mu_Z)^2}
{16\max(\sigma_t^2,m_t^2\nu_0^2)}
\right\}
\rmd t \right)^{1/2} \eqsp.
\end{aligned}
\end{equation}
Let
\begin{equation}
\label{eq:tail_moment_constant}
H_{\infty}
\eqdef
\sup_{t\in[0,T]}
\left[
2\sqrt{M_{0,4}}
+
4\sqrt{2}\,
\sqrt{a^4+b^4M_{t,4p}}
\right]
<\infty.
\end{equation}
Since $c(t)\leq c_\infty$ for every $t\in[0,T]$, the last term in
\eqref{eq:equality_to_finsh_thm_2_layers_bis} is bounded by
\begin{align*}
&
\left(
\frac{1}{2}\sqrt{1+C_0}
\int_0^T
c(t)q(t)
\left[
2\sqrt{M_{0,4}}
+
4\sqrt{2}\,
\sqrt{a^4+b^4M_{t,4p}}
\right]
\exp\left\{
-\frac{(R-\sigma_t\mu_Z)^2}
{16\max(\sigma_t^2,m_t^2\nu_0^2)}
\right\}
\rmd t
\right)^{1/2}
\\
&\qquad\leq
\left(
\frac{1}{2}\sqrt{1+C_0}\,
c_\infty H_{\infty}
\right)^{1/2}
\left(
\int_0^T
q(t)
\exp\left\{
-\frac{(R-\sigma_t\mu_Z)^2}
{16\max(\sigma_t^2,m_t^2\nu_0^2)}
\right\}
\rmd t
\right)^{1/2}.
\end{align*}
Moreover, since $q$ is a probability density and the exponential term is
bounded by one,
\[
\int_0^T
q(t)
\exp\left\{
-\frac{(R-\sigma_t\mu_Z)^2}
{16\max(\sigma_t^2,m_t^2\nu_0^2)}
\right\}
\rmd t
\leq 1.
\]
Therefore,
\begin{align*}
\widetilde C_{\rm tail}
\int_0^T
q(t)
\exp\left\{
-\frac{(R-\sigma_t\mu_Z)^2}
{16\max(\sigma_t^2,m_t^2\nu_0^2)}
\right\}
\rmd t
&\leq
\widetilde C_{\rm tail}
\left(
\int_0^T
q(t)
\exp\left\{
-\frac{(R-\sigma_t\mu_Z)^2}
{16\max(\sigma_t^2,m_t^2\nu_0^2)}
\right\}
\rmd t
\right)^{1/2}.
\end{align*}
Hence, defining
\begin{equation}
\label{eq:tail_constant}
C_{\rm tail}
\eqdef
\widetilde C_{\rm tail}
+
\left(
\frac{1}{2}\sqrt{1+C_0}\,
c_\infty H_{\infty}
\right)^{1/2},
\end{equation}
we obtain
\begin{align}
\label{eq:equality_to_finsh_thm_2_layers}
\E_{\sgd}\left[
\sqrt{\mathcal L(\param^{(n)})}
\right]
&\leq
\frac{1}{\sqrt{2}}
\E_{\sgd}\left[
\|\Delta_n\|_{L_2(\mu_R)}
\mathbf 1_{\Omega_{n,R}}
\right]
+
C_{\rm tail}
\left(
\int_0^T
q(t)
\exp\left\{
-\frac{(R-\sigma_t\mu_Z)^2}
{16\max(\sigma_t^2,m_t^2\nu_0^2)}
\right\}
\rmd t
\right)^{1/2}.
\end{align}

\paragraph{Localized error decomposition}

We introduce the integral operators associated with the kernels $K_k$ and $K_\infty$, restricted to the localized region $\mathcal X_R$. In particular, for every measurable function $g:\mathbb R^{d} \times [0,T] \to\mathbb R^d$,
\begin{align*}
 (\mathsf K_k g)(\x)
& \eqdef
\mathbb E \left[c(\tau)\,K_k(\x,\ora \X)\,g(\ora \X) \mathbf 1_{\{\|\ora X_\tau\| \leq R\}} \right] \eqsp, \qquad \ora \X=(\ora X_\tau,\tau) \eqsp,
\qquad \tau\sim q \eqsp, \\
(\mathsf K_\infty g)(\x)
& \eqdef
\E \left[c(\tau)\,K_\infty(\x,\ora \X) g(\ora \X) \mathbf 1_{\{\|\ora X_\tau\|\leq R\}} \right],
\qquad
\ora \X=(\ora X_\tau,\tau) \eqsp, \qquad \tau\sim q  \eqsp.
\end{align*}
We also introduce the following operators that will be useful throughout the proof:
\begin{align} \label{eq:notation_kernel}
\mathsf{P}_k :=  \mathsf{I}-\gamma_k \mathsf{K}_\infty \eqsp ,\qquad
\mathsf{Q}_k :=  \mathsf{I} -\gamma_k \mathsf{K}_k \eqsp, \qquad
\mathsf{D}_k := \mathsf{Q}_k - \mathsf{P}_k \eqsp,
\end{align}
where $\mathsf{I}$ should be understood as the identity on $L_2(\mu_R)$. In the following, all norms of vector-valued functions $f: \R^{d} \times [0,T] \to \R^d$ are taken in $L_2(\mu_R;\R^d)$, where $\mu_R$ is defined in \eqref{eq:def_mu_r}:
\[
\|f\|^2_{L_2(\mu_R)} =\int \|f(\x)\|^2 \mu_R(\rmd \x) \eqsp, \qquad \x=(x,t)\in\R^{d} \times [0,T] \eqsp.
\]
For a bounded linear operator $\mathsf T$ on $L_2(\mu_R;\R^d)$, we denote
\[
\|\mathsf T\|_{\rm op} =\sup_{\|f\|\le 1}\|\mathsf T f\|_{L_2(\mu_R)} \eqsp.
\]

\subsection{Error decomposition}

\begin{lemma}
\label{lem:delta_rec_matrix}
On the event $\Omega_{n+1,R}$, for all $0\le k\le n$ and all 
$\x\in\R^{d} \times [0,T]$, 
\begin{align}
\label{eq:delta_rec}
\Delta_{k+1}(\x)
= \bigl( \mathsf I -\gamma_k \mathsf{K}_k\bigr) \Delta_k(\x)
- v_k(\x,\ora \X^{(k)})
+ \varepsilon_k(\x,\ora \X^{(k)}) \eqsp,
\end{align}
where
\begin{align*}
    v_k(\x,\ora \X^{(k)})
& \eqdef
\gamma_k \ck K_k(\x, \ora \X^{(k)})\bigl[\Delta_k(\ora \X^{(k)}) + \xi_k\bigr] \un_{\{ \|\ora X^{(k)}_{t^{(k)}} \| \leq R \}}
-\gamma_k (\mathsf K_k\Delta_k)(\x) \eqsp, \\
    \varepsilon_k(\x, \ora \X^{(k)})
& \eqdef 
f_{\theta}(\x, \param[k]) - f_{\theta}(\x, \param[k+1])
+
\\
& \qquad \qquad \gamma_k \ck\, K_k(\x,\ora \X^{(k)})
\Bigl[f^{\star}(\ora \X^{(k)})+\xi_k-f_{\theta}(\ora \X^{(k)}, \param[k])\Bigr] \eqsp.
\end{align*}
\end{lemma}

\begin{proof}
For all $0\le k\le n$, by definition of $\Delta_k$,
\[
\Delta_{k+1}(\x)-\Delta_k(\x)
=
\net(\x,\param[k]) - \net(\x,\param[k+1]) \eqsp.
\]
% Using the definition of $\varepsilon_k(\x,\ora \X^{(k)})$,
% $$
% \net(\x,\param[k]) - \net(\x,\param[k+1])
% =
% -\gamma_k \ck K_k(\x,\ora \X^{(k)})
% \Bigl[f^\star(\ora \X^{(k)})+\xi_k-\net(\ora \X^{(k)},\param[k])\Bigr]
% +\varepsilon_k(\x,\ora \X^{(k)}) \eqsp.
% $$
On $\Omega_{n+1,R}$,
\begin{align*}
\net(\x,\param[k]) - \net(\x,\param[k+1])
&= -\gamma_k \ck K_k(\x,\ora \X^{(k)})
\Bigl[f^\star(\ora \X^{(k)})+\xi_k-\net(\ora \X^{(k)},\param[k])\Bigr] \\
&\quad \times \un_{\{ \| \ora X^{(k)}_{t^{(k)}} \| \leq R \}}
+\varepsilon_k(\x,\ora \X^{(k)}) \eqsp .
\end{align*}
Since
\[
f^\star(\ora \X^{(k)})+\xi_k-\net(\ora \X^{(k)},\param[k])
= \Delta_k(\ora \X^{(k)})+\xi_k \eqsp,
\]
it follows that
\begin{align*}
\Delta_{k+1}(\x)
=
\Delta_k(\x)
-\gamma_k \ck K_k(\x,\ora \X^{(k)}) \left[ \Delta_k(\ora \X^{(k)}) + \xi_k \right]
+\varepsilon_k(\x,\ora \X^{(k)}) \eqsp,
\end{align*}
and therefore,
\begin{align*}
\Delta_{k+1}(\x)
=
\Delta_k(\x)
- v_k(\x,\ora \X^{(k)}) - \gamma_k(\mathsf K_k\Delta_k)(\x)
+\varepsilon_k(\x,\ora \X^{(k)}) \eqsp.
\end{align*}
Grouping the remaining terms  yields \eqref{eq:delta_rec}.
\end{proof}

\begin{lemma} 
\label{lem:operator_norm_bound}
Let $\mathsf K_k$ and $\mathsf K_\infty$ be the integral
operators on $L_2(\mu_R;\R^d)$ associated with the kernels $K_k$ and $K_\infty$. If
\begin{align}
\label{eq:step_size_condition}
0\le \gamma_k\le \min \left\{ \frac{2m}{\|A\|^2 B_R^2 \mu_R( \mathcal{X}_R
)} ; \frac{1}{ B_R^2 \mu_R(\mathcal{X}_R
)} \right\} \eqdef \Gamma_R \eqsp,
\end{align}
then, for any $0 \leq k \leq n$
\[ 
\|\mathsf{P}_k\|_{\rm op} \leq 1 \quad \text{and} \quad \|\mathsf Q_k\|_{\rm op} \leq 1 \eqsp. 
\]
\end{lemma}

\begin{proof}

\emph{Step 1: operator bound on $L_2(\mu_R, \R^d)$. } \\
Let $\mathsf H$ be either $\mathsf K_\infty$ or $\mathsf K_k$, and let
$H(\x,\x')$ denote the kernel associated with the corresponding integral operator.  Then, for every
$g\in L_2(\mu_R,\R^d)$ and any $\x \in \mathcal{X}_R$, by the triangle inequality
\begin{align*}
\|(\mathsf H g)(\x)\| & = \left\|\int H(\x,\x') g(\x') \mu_R(\rmd \x') \right\|
\le
\int \|H(\x,\x')\| \|g(\x')\| \mu_R( \rmd \x') \eqsp.
\end{align*}

Let $\sup_{\x, \x' \in \mathcal{X}_R} \|H(\x,\x')\| = M_H$, so that using Cauchy--Schwarz,
$$
\|(\mathsf H g)(\x)\| \leq M_H \int \|g(\x')\| \mu_R( \rmd \x') \leq M_H \mu_R(\mathcal{X}_R)^{1/2} \|g\|_{L_2(\mu_R)} \eqsp.
$$
Therefore, integrating $\x$ with respect to $\mu_R$ yields,
$$
\|\mathsf H g\|_{L_2(\mu_R)}^2
\le
M_H^2 \mu_R( \mathcal{X}_R)^2 \|g\|_{L_2(\mu_R)}^2 \eqsp,
$$
and therefore, 
$$
\|\mathsf H \|_{\rm op}
\le
M_H \mu_R(\mathcal{X}_R)  \eqsp.
$$

\emph{Step 2: uniform kernel bounds.} \\
We now control $M_H = \|H(\x,\x')\|$ in both cases. For any $\x=(x,t),\x'=(x',t')\in \mathcal{X}_R$, we have
$$
\|\x\|^2=\|x\|^2+t^2\le R^2+T^2 = B_R^2 \eqsp,
$$
so that,
$$
\| K_\infty(\x,\x') \|
\le
|\x^\top \x'|
\le
B^2_R \eqsp.
$$
Moreover, using that $\|D_{\x,\param[k]}\| \le 1$,
$$
\|K_k(\x,\x')\|
\le
\frac{\|A\|^2}{m} B_R^2 \eqsp.
$$
\emph{Step 3: non-expansiveness condition.} \\
Using the bounds obtained in Step 2, we deduce that
$$
\|\mathsf K_\infty\|_{\rm op}
\le B_R^2 \mu_R(\mathcal{X}_R),
\qquad \text{ and } \qquad
\|\mathsf K_k\|_{\rm op}
\le
\frac{\|A\|^2}{m} B_R^2\,\mu_R(\mathcal{X}_R) \eqsp.
$$
Moreover, both $\mathsf K_\infty$ and $\mathsf K_k$ are self-adjoint positive semidefinite operators on $L_2(\mu_R;\mathbb R^d)$. Therefore their spectra are contained in
$ [0,\|\mathsf K_\infty\|_{\rm op}]$ and $[0,\|\mathsf K_k\|_{\rm op}]$ respectively. It follows that
$$
\|\mathsf P_k\|_{\rm op}
=
\|I-\gamma_k \mathsf K_\infty\|_{\rm op}
\le 1 \eqsp,
$$
whenever
$$
0\le \gamma_k\le \frac{2}{B_R^2\,\mu_R(\mathcal{X}_R)} \eqsp,
$$
and similarly
$$
\|\mathsf Q_k\|_{\rm op}
=
\|I-\gamma_k \mathsf K_k\|_{\rm op}
\le 1 \eqsp,
$$
whenever
$$
0\le \gamma_k\le \frac{2m}{\|A\|^2 B_R^2 \mu_R(\mathcal{X}_R)} \eqsp.
$$
\end{proof}

\begin{lemma}
For all $k\ge 0$, assume that $0\le \gamma_k \le \Gamma_R$. Then, on the event $\Omega_{n+1,R}$, for all $0 \le k \le n$,
\begin{multline}
\|\Delta_{k+1}\|_{L_2(\mu_R)} \le
\left\|\left(\prod_{s=0}^{k} \mathsf{P}_s\right) \Delta_0 \right\|_{L_2(\mu_R)}
+ \sum_{r=0}^{k}\|\mathsf{D}_r\|_{\rm op} \|\Delta_0\|_{L_2(\mu_R)}
\\+ \left\|
\sum_{r=0}^{k}\left(\prod_{s=r+1}^{k} \mathsf{Q}_s\right) v_r
\right\|_{L_2(\mu_R)}
+ \sum_{r=0}^{k}\|\varepsilon_r\|_{L_2(\mu_R)} \eqsp.
\label{eq:error_decomp_bound}
\end{multline}
\end{lemma}

\begin{proof}
Using the notation in \eqref{eq:notation_kernel}, \eqref{eq:delta_rec} can be written as
$$
\Delta_{k+1} = \mathsf{Q}_k   \Delta_k - v_k + \varepsilon_k \eqsp.
$$
and therefore,
$$
\Delta_{k+1}
= \left(\prod_{s=0}^{k} \mathsf{Q}_s\right) \Delta_0
- \sum_{r=0}^{k}\left(\prod_{s=r+1}^{k} \mathsf{Q}_s\right)  v_r
+ \sum_{r=0}^{k}\left(\prod_{s=r+1}^{k} \mathsf{Q}_s\right)  \varepsilon_r \eqsp.
$$
Using that,
$$
\prod_{s=0}^k \mathsf{Q}_s - \prod_{s=0}^k \mathsf{P}_s = \sum_{r = 0}^k \left( \prod_{i=r+1}^k \mathsf{Q}_i \right) \left( \mathsf{Q}_r - \mathsf{P}_r\right) \left( \prod_{j=0}^{r-1} \mathsf{P}_j \right) \eqsp.
$$
Next, since for any $r$, $\mathsf{D}_r=\mathsf{Q}_r-\mathsf{P}_r$, 
% we have the identity
% $$
% \prod_{s=0}^{k} \mathsf{Q}_s
% =
% \prod_{s=0}^{k} \mathsf{P}_s
% +
% \sum_{r=0}^{k}
% \left(\prod_{i=r+1}^{k} \mathsf{Q}_i\right)
% \mathsf{D}_r
% \left(\prod_{j=0}^{r-1} \mathsf{P}_j\right),
% $$
% and therefore
\begin{align*}
\Delta_{k+1}
= \left(\prod_{s=0}^{k} \mathsf{P}_s\right) \Delta_0
+ \sum_{r=0}^{k}
\left(\prod_{i=r+1}^{k} \mathsf{Q}_i\right)\mathsf{D}_r
\left(\prod_{j=0}^{r-1} \mathsf{P}_j\right) \Delta_0 - \sum_{r=0}^{k}\left(\prod_{s=r+1}^{k} \mathsf{Q}_s\right) v_r
+ \sum_{r=0}^{k}\left(\prod_{s=r+1}^{k} \mathsf{Q}_s\right) \varepsilon_r \eqsp.
\end{align*}
Taking norms and applying the triangle inequality yields
\begin{align*}
\|\Delta_{k+1}\|_{L_2(\mu_R)}
&\le
\left\|\left(\prod_{s=0}^{k} \mathsf{P}_s\right) \Delta_0 \right\|_{L_2(\mu_R)}
+ \sum_{r=0}^{k}
\left\|
\left(\prod_{i=r+1}^{k} \mathsf{Q}_i\right)\mathsf{D}_r
\left(\prod_{j=0}^{r-1} \mathsf{P}_j\right) \Delta_0
\right\|_{L_2(\mu_R)} \\
&\quad
+ \left\|
\sum_{r=0}^{k}\left(\prod_{s=r+1}^{k} \mathsf{Q}_s\right) v_r
\right\|_{L_2(\mu_R)}
+ \sum_{r=0}^{k}
\left\|\left(\prod_{s=r+1}^{k} \mathsf{Q}_s\right) \varepsilon_r\right\|_{L_2(\mu_R)} \eqsp.
\end{align*}

Since $\gamma_k \leq \Gamma_R$, Lemma \ref{lem:operator_norm_bound} gives $\|\mathsf{Q}_s\|_{\rm op} \le 1$ and  $\|\mathsf{P}_s\|_{\rm op} \le 1$. Hence,
\begin{align*}
\left\|
\left(\prod_{i=r+1}^{k} \mathsf{Q}_i\right)\mathsf{D}_r
\left(\prod_{j=0}^{r-1} \mathsf{P}_j\right) \Delta_0
\right\|_{L_2(\mu_R)}
&\le \|\mathsf{D}_r\|_{\rm op}  \|\Delta_0\|_{L_2(\mu_R)} \eqsp, \\
\left\|\left(\prod_{s=r+1}^{k} \mathsf{Q}_s\right) \varepsilon_r\right\|_{L_2(\mu_R)}
&\le \|\varepsilon_r\|_{L_2(\mu_R)} \eqsp.
\end{align*}
Applying these bounds to the inequality for $\|\Delta_{k+1}\|_{L_2(\mu_R)}$ yields \eqref{eq:error_decomp_bound} and concludes the proof.
\end{proof}

It now remains to bound each term in the above decomposition.
For all $n\ge 0$, define
\begin{align*}
\mathcal T_{1,n} &\eqdef \left\|\left(\prod_{k=0}^{n}\mathsf P_k\right)\Delta_0\right\|_{L_2(\mu_R)},\\
\mathcal T_{2,n} &\eqdef \sum_{k=0}^{n}\|\mathsf D_k\|_{\rm op} \|\Delta_0\|_{L_2(\mu_R)},\\
\mathcal T_{3,n} &\eqdef \left\|
\sum_{k=0}^{n}\left(\prod_{r=k+1}^{n}\mathsf Q_r\right)v_k
\right\|_{L_2(\mu_R)},\\
\mathcal T_{4,n} &\eqdef \sum_{k=0}^{n}\|\varepsilon_k\|_{L_2(\mu_R)}.
\end{align*}
Taking expectation with respect to the SGD sample stream, on the event $\Omega_{n+1,R}$, conditionally on the
initialization, yields
\begin{equation}
\label{eq:expected_squared_error_decomp}
\E_{\sgd} \left[\|\Delta_{n+1}\|_{L_2(\mu_R)} \un_{\Omega_{n+1,R}} \right]
\le
\sum_{\ell=1}^4
\E_{\sgd} \left[\mathcal T_{\ell,n} \un_{\Omega_{n+1,R}} \right] \eqsp.
\end{equation}

\subsection{First term}

For any function $g \in L_2(\mu_R;\mathbb R^d)$, we denote by $\mathcal{R}(g,r)$ the $L_2(\mu_R;\mathbb R^d)$-norm of the projection of $g$ onto the subspace spanned by the eigenvectors $\{e_i\}_{i=r+1}^\infty$, given by
\begin{align}
\label{eq:projection_K_infty}
\mathcal{R}(g,r)
\eqdef
\left(\sum_{i=r+1}^{\infty} \langle g,e_i\rangle^2\right)^{1/2} \eqsp.
\end{align}

\begin{lemma}
\label{lem:first_term_bound}
Assume that $0 \le \gamma_k \le \Gamma_R \leq 1 /\bigl(B_R^2 \mu_R(\mathcal{X}_R)\bigr)$ for all $k \ge 0$. Then,
$$
\mathcal T_{1,n}  = \left\|\left(\prod_{k=0}^{n}\mathsf{P}_k\right) \Delta_0 \right\|_{L_2(\mu_R)}
\le \inf_{r\in\mathbb{N}}\left\{
\left(\prod_{k=0}^{n}(1-\gamma_k\kappa_r)\right) \|\Delta_0\|_{L_2(\mu_R)} + \mathcal{R}(\Delta_0,r)
\right\} \eqsp.
$$

Notice that $\mathcal T_{1,n}$ is deterministic conditionally on the initialization and does not depend on the SGD samples. Therefore,
$$
\E_{\sgd}\!\left[ \mathcal T_{1,n}\mathbf 1_{\Omega_{n+1,R}}
\right] \le \mathcal T_{1,n} \eqsp.
$$
\end{lemma}

\begin{proof}
Fix $n\ge 0$. Let $\{(\kappa_i,e_i)\}_{i\ge 1}$ denote the eigenpairs of $\mathsf K_\infty$, with
$\kappa_1\ge \kappa_2\ge \cdots \ge 0$, and write the spectral expansion
\[
\Delta_0=\sum_{i=1}^{\infty}\langle \Delta_0,e_i\rangle\,e_i \eqsp.
\]
Since $\mathsf{P}_s = I-\gamma_s\mathsf K_\infty$, we have $\mathsf{P}_s e_i=(1-\gamma_s\kappa_i)e_i$, and hence
\[
\left(\prod_{k=0}^{n}\mathsf{P}_k\right) \Delta_0
= \sum_{i=1}^{\infty}\rho_i(n) \langle \Delta_0,e_i\rangle e_i \eqsp,
\qquad
\rho_i(n):=\prod_{k=0}^{n}\bigl(1-\gamma_k\kappa_i\bigr) \eqsp.
\]
The assumption on the step sizes $\gamma_k$ implies that $0\le 1-\gamma_k\kappa_i\le 1$ for all $i$ and all $k\le n$, hence
$0\le \rho_i(n)\le 1$. Moreover, since $\kappa_i\ge \kappa_r$ for $i\le r$ and $u\mapsto 1-\gamma_k u$ is decreasing,
we have $\rho_i(n)\le \rho_r(n)$ for all $i\le r$. Therefore, for any $r\in\mathbb{N}$,
\begin{align*}
\left\|
\left(\prod_{k=0}^{n}\mathsf P_k\right)\Delta_0
\right\|_{L_2(\mu_R)}
&\le \left\| \sum_{i=1}^{r} \rho_i(n)\langle \Delta_0,e_i\rangle e_i \right\|_{L_2(\mu_R)}
+ \left\| \sum_{i=r+1}^{\infty} \rho_i(n)\langle \Delta_0,e_i\rangle e_i \right\|_{L_2(\mu_R)} \\
&\le \rho_r(n) \left\| \sum_{i=1}^{r} \langle \Delta_0,e_i\rangle e_i \right\|_{L_2(\mu_R)} + \left\| \sum_{i=r+1}^{\infty} \langle \Delta_0,e_i\rangle e_i \right\|_{L_2(\mu_R)} \\
&\le \rho_r(n)\|\Delta_0\|_{L_2(\mu_R)} + \mathcal R(\Delta_0,r) \eqsp.
\end{align*}
Since this holds for every $r\in\mathbb{N}$, taking the infimum over $r$ concludes the proof.
\end{proof}

\subsection{Second term}

For a matrix-valued kernel $G:\mathcal{X}_R\times\mathcal{X}_R\to\R^{d\times d}$, define
\[
\|G\|_{\infty,R}
\eqdef
\sup_{\x,\x'\in\mathcal{X}_R}\|G(\x,\x')\|
\]
and let
\begin{equation} \label{eq:def_loss_k}
\loss_k
\eqdef
\net \left(\ora \X^{(k)};\param^{(k)}\right)-X_0^{(k)}
=
\net \left(\ora \X^{(k)};\param^{(k)}\right)
-
\left(f^\star \left(\ora \X^{(k)}\right)+\xi_k\right) \eqsp.
\end{equation}

% \[
% \begin{aligned}
% \E\left[\mathcal T_{2,k}\right]
% &\le
% \|\Delta_0\|_{L_2(\mu_R)}
%  3C d B^2\sqrt{\frac{d+\log(1/\delta)}{m}}
% \sum_{r=0}^k \gamma_r \\
% &\quad+
% \|\Delta_0\|_{L_2(\mu_R)}
% \, \frac{4d^{3/2}B^3}{\sqrt{2\pi m}}
% \sum_{r=0}^{k} \gamma_r \sum_{\ell=0}^{r-1}\gamma_{\ell}
% \Bigl(
% \E \big[c(t^{(\ell)}) \|\Delta_{\ell}(\ora \X_t^{(\ell)})\|\big]+\tau
% \Bigr)
% \eqsp.
% \end{aligned}
% \]

\begin{lemma}
\label{lem:second_term_bound}
For all $n \ge 0$, there exists a universal constant $C>0$ such that, with probability at least $1-\delta$ over the initialization,
% \textcolor{red}{$$
% \begin{aligned}
% \E\left[\mathcal T_{2,k}\right]
% &\le
% \|\Delta_0\|_{L_2(\mu_R)}
%  3C d B_R^2\sqrt{\frac{d+\log(1/\delta)}{m}}
% \sum_{r=0}^k \gamma_r \\
% &\quad+
% \|\Delta_0\|_{L_2(\mu_R)}
% \, c_\infty \frac{4d^{3/2}B_R^3}{\sqrt{2\pi m}} 
% \sum_{r=0}^{k} \gamma_r \sum_{\ell=0}^{r-1}\gamma_{\ell}
% \Bigl(
% \E \big[ \|\Delta_{\ell}(\ora \X_t^{(\ell)})\|\big]+\tau
% \Bigr)
% \eqsp.
% \end{aligned}
% $$}
\begin{align*}
\E_{\sgd} \left[ \mathcal T_{2,n} \un_{\Omega_{n+1,R}} \right]
& \leq C \mu_R(\mathcal{X}_R) B_R^2
\|\Delta_0\|_{L_2(\mu_R)}
\Bigg[
\frac{d^{3/2} B_R}{\sqrt m}
\E_{\sgd} \left[
\left(\sum_{k=0}^{n} \gamma_k \eta_k \right)
\un_{\Omega_{n+1,R}}
\right] \\
&\qquad\qquad
+ d \sqrt{ \frac{d\log(m)+\log(1/\delta)}{m}}
\sum_{k=0}^{n} \gamma_k 
\Bigg] \eqsp,
\end{align*}
where
\begin{align*}
\eta_k \eqdef 
\sum_{r=0}^{k-1} \gamma_r c(t^{(r)}) \|\loss_r\| \eqsp.
\end{align*}
\end{lemma}

\begin{proof}
For all $k \ge 0$, since $\mathsf D_k = \gamma_k(\mathsf K_\infty -\mathsf K_k )$,
$$
\|\mathsf D_k\|_{\rm op}
\le
\gamma_k \|\mathsf{K}_k- \mathsf{K}_\infty\|_{\rm op} \eqsp,
$$
and therefore
\begin{align*}
\|\mathsf{K}_k- \mathsf{K}_\infty\|_{\rm op} %& \leq \mu_R(\mathcal{X}_R) \sup_{\x , \x' \in \mathcal{X}_R} \| K_r -K_\infty \| \\
& \leq \mu_R(\mathcal{X}_R) \| K_k -K_\infty \|_{\infty,R}  \eqsp.
\end{align*}
On the event $\Omega_{n+1,R}$, Lemma~\ref{lem:kernel_Kk_K0_bound}
applies for every $0\le k\le n$. Combining it with
Lemma~\ref{lem:kernel_K0_Kinf_bound}, both with confidence parameter
$\delta/2$, we obtain for $m\ge d\log(em)+\log(1/\delta)$, there exist universal constants (which may vary from line to line) $C,C' > 0$ such that, with probability at least $1-\delta$
\begin{align*}
\|K_k-K_\infty\|_{\infty,R}
& \le 2dB_R^2 \left( \frac{2B_R\sqrt d}{\sqrt{2\pi m}} \eta_k
+ C \sqrt{\frac{d+\log(2/\delta)}{m}} \right) 
+ C'B_R^2 \sqrt{\frac{d \log(m)+\log(2/\delta)}{m}}  \\
& \le C B_R^2
\left(
\frac{d^{3/2}B_R}{\sqrt m} \eta_k
+
d \sqrt{\frac{d \log(m)+\log(1/\delta)}{m}} \right) \eqsp.
\end{align*}
Therefore,
$$
\|\mathsf D_k\|_{\rm op}
\le \gamma_k \mu_R(\mathcal{X}_R) C B_R^2
\left(
\frac{d^{3/2}B_R}{\sqrt m}\eta_k
+
d \sqrt{\frac{d \log(m)+\log(1/\delta)}{m}} \right) \eqsp.
$$
Hence,
\begin{align*}
\mathcal T_{2,n} & = \sum_{k=0}^{n}\|\mathsf D_k\|_{\rm op} \|\Delta_0\|_{L_2(\mu_R)} \\
& \leq C \mu_R(\mathcal{X}_R) B_R^2 \|\Delta_0\|_{L_2(\mu_R)} \left( \frac{d^{3/2}B_R}{\sqrt m} \sum_{k=0}^{n} \gamma_k \eta_k + d \sqrt{\frac{d \log(m)+\log(1/\delta)}{m}}  \sum_{k=0}^{n} \gamma_k \right) \eqsp,
\end{align*}
taking expectation with respect to the SGD randomness finishes the proof.
% yields
% \begin{align*}
% \E \left[ \mathcal T_{2,n}^2 \un_{\Omega_{n+1,R}} \right]
% & \leq C \mu_R^2(\mathcal{X}_R) B_R^4
% \|\Delta_0\|_{L_2(\mu_R)}^2
% \Bigg[
% \frac{d^3 B_R^2}{m}
% \E \left[
% \left(\sum_{k=0}^{n} \gamma_k \eta_k \right)^2
% \un_{\Omega_{n,R}}
% \right] \\
% &\qquad
% + d^2 \frac{d\log(m)+\log(1/\delta)}{m}
% \left(\sum_{k=0}^{n} \gamma_k \right)^2
% \Bigg] \eqsp,
% \end{align*}
% which concludes the proof.
\end{proof}

\begin{lemma}
\label{lem:kernel_Kk_K0_bound}
There exists a universal constant $C>0$ such that, with probability at least $1-\delta$ over the initialization, on $\Omega_{n+1,R}$ for all $0 \le k \le n$,
% \begin{align*}
% \|K_k - K_0\|_{\infty} 
% \le
% 2dB^2\frac{2B\sqrt{d}}{\sqrt{2\pi m}} \eta_k + 2C dB^2 \sqrt{\frac{d+\log(1/\delta)}{m}} \eqsp,
% \end{align*}
% where
% \begin{align*}
% \eta_k \eqdef 
% \sum_{r=0}^{k-1} \gamma_r c(t^{(r)}) \|\loss_r\| \eqsp.
% \end{align*}
\begin{align*}
\|K_k - K_0\|_{\infty,R} 
\le
2dB_R^2 \left( \frac{2B_R\sqrt{d}}{\sqrt{2\pi m}} \eta_k + C  \sqrt{\frac{d+\log(1/\delta)}{m}} \right) \eqsp,
\end{align*}
where
\begin{align*}
\eta_k \eqdef 
\sum_{r=0}^{k-1} \gamma_r c(t^{(r)}) \|\loss_r\| \eqsp.
\end{align*}
\end{lemma}

\begin{proof}
For all $0 \le k \le n$, $1 \le i \le m$, and $\x\in\mathcal{X}_R$, set
\[
I_i^k(\x)=\mathbf{1}\{\param_i^{(k)\top}\x\ge 0\}.
\]
Then, for all $\x,\x'\in\mathcal{X}_R$,
\[
\begin{aligned}
\| K_k(\x,\x') - K_0(\x,\x') \|
&=
\Biggl\|
\frac{\x^\top \x'}{m}
\sum_{i=1}^m
\bigl(I_i^k(\x)I_i^k(\x')-I_i^0(\x)I_i^0(\x')\bigr)
A_{\cdot i}A_{\cdot i}^\top
\Biggr\| \\
&\le
\frac{B_R^2}{m}
\sum_{i=1}^m
\|A_{\cdot i}A_{\cdot i}^\top\|
\bigl|I_i^k(\x)I_i^k(\x')-I_i^0(\x)I_i^0(\x')\bigr| \\
&\le
\frac{dB_R^2}{m}
\sum_{i=1}^m
\Bigl(
|I_i^k(\x')-I_i^0(\x')|
+
|I_i^k(\x)-I_i^0(\x)|
\Bigr) \\
&\le
\frac{dB_R^2}{m}\bigl(S_k(\x) + S_k(\x')\bigr).
\end{aligned}
\]
Taking the supremum over $\x, \x'\in \mathcal{X}_R$ yields
$$
\|K_k-K_0\|_{\infty,R}
\le
\frac{2dB_R^2}{m} \|S_k\|_{\infty,R} \eqsp.
$$
The result then follows from Lemma~\ref{lem:bound:S_k}.
\end{proof}

\begin{lemma}
\label{lem:kernel_K0_Kinf_bound}
If $m\ge d\log(em)+\log(1/\delta)$, there exists a universal constant $C>0$ such that, with probability at least $1-\delta$ over the initialization $(A, \param[0])$,
$$
\|K_0 - K_\infty\|_{\infty,R}
\leq C B_R^2 \sqrt{\frac{ \left( d \log (m)+\log(1/\delta) \right)}{m}} \eqsp.
$$
\end{lemma}

\begin{proof}
For all $1 \le i \le m$ and $\x\in\mathcal{X}_R$, set
\[
I_i^0(\x)=\mathbf{1}\{\param_i^{(0)\top}\x\ge 0\},
\qquad
p(\x,\x')=\E_{\param\sim\mathcal N(0,\Id_{d+1})}
\left[
\mathbf{1}\{\param^\top \x\ge 0\}
\mathbf{1}\{\param^\top \x'\ge 0\}
\right].
\]
Since $\x\in\mathcal X_R$, we have $\|\x\|\le B_R$
\begin{align*}
\|K_0(\x,\x')-K_\infty(\x,\x')\|
&\le
B_R^2
\Biggl\|
\frac{1}{m}\sum_{i=1}^m
I_i^0(\x)I_i^0(\x')
A_{\cdot i}A_{\cdot i}^\top
-
p(\x,\x')\Id_d
\Biggr\|.
\end{align*}
Moreover, by independence of $\param_i^{(0)}$ and $A$, we get
\[
\E\left[
\mathbf{1}\{\param^\top \x\ge 0\}
\mathbf{1}\{\param^\top \x'\ge 0\}
A_{\cdot i}A_{\cdot i}^\top
\right]
=
p(\x,\x')\Id_d .
\]
Therefore, by Lemma~\ref{lem:cov_proba_bound}, if
$m\ge d\log(\rme m)+\log(1/\delta)$, then with probability at least $1-\delta$,
\[
\Biggl\|
\frac{1}{m}\sum_{i=1}^m
I_i^0(\x)I_i^0(\x')
A_{\cdot i}A_{\cdot i}^\top
-
p(\x,\x')\Id_d
\Biggr\|
\le
C\sqrt{\frac{d\log(m)+\log(1/\delta)}{m}} .
\]
Combining the last two bounds concludes the proof.
\end{proof}

\subsection{Third term}

\begin{lemma}
\label{lem:controle_v_term}
Assume that $\gamma_k \leq \Gamma_R $ for every $k\ge 0$. Then, there exists a universal constant $C>0$ such that, with probability at least $1-\delta$,
% \begin{align*}
%  \E \left[ \mathcal{T}_{3,k}^2 \right] & \leq  \frac{\|A\|^4}{m^2} B_R^4 \mu_R(\mathcal{X}_R) c_\infty^2 \sum_{r=0}^{k}
% \gamma_r^2 
% \E \left[
% \|\Delta_r(\ora\X^{(r)})+\xi_r\|^2 \mathbf{1}_{\{ \|\ora X^{(r)}_{t^{(r)}} \| \leq R \}} 
% \right] \eqsp.
% \end{align*}
\begin{align*}
\E_{\sgd} \left[ \mathcal{T}_{3,n} \un_{\Omega_{n+1,R}} \right]
&\leq C
\chi_{m,\delta}^2
\sqrt{c_\infty \mu_R(\mathcal{X}_R)} B_R^2
\sqrt{\sum_{k=0}^{n}
\gamma_k^2 
\E_{\sgd} \left[
c(t^{(k)}) \|\Delta_k(\ora\X^{(k)})+\xi_k\|^2 \un_{\Omega_{k+1,R}}
\right]}
\end{align*}
where
\begin{equation}
\label{eq:chi_def}
\chi_{m,\delta}
\eqdef
1+\sqrt{\frac{d}{m}}
+\sqrt{\frac{\log(1/\delta)}{m}}.
\end{equation}
\end{lemma}

% \begin{align}
% \E\left[\mathcal T_{3,k}\right] & = \E \left[ \left\|
% \sum_{r=0}^{k}\left(\prod_{s=r+1}^{k} \mathsf{Q}_s\right) v_r
% \right\|_{L_2(\mu_R)} \right] \notag \\
% & \leq \sqrt{\sum_{r=0}^k \gamma_r^2 
% \mathbb E \Big[ c(t^{(r)})^2
% \|\Delta_r(\ora \X^{(r)})+\xi_r\|^2 \int \left\|K_r(\x,\ora \X^{(r)}) \right\|^2 \mu_R( \rmd \x) 
% \Big]}
% \label{eq:control_term_3} \\
% % & \leq \sqrt{\sum_{r=0}^k \gamma_r^2 
% % \sigma_r^2 \sup_{z \in \R^{d} \times [0,T]} \int \left\| K_r(\x,z) \right\|^2_{\rm op} \mu_R( \rmd \x) 
% % \Big]} \eqsp. \notag
% \end{align}

\begin{proof}
Recall that the SGD filtration is given by $\mathcal{F}_n \eqdef \sigma \left((X_0^{(0)},Z^{(0)},t^{(0)}),\dots,(X_0^{(n-1)},Z^{(n-1)},t^{(n-1)})\right)
$ and define for $k \ge 1$, $ h_k:=\mathsf Q_k q_{k-1}$ and
$$
q_k=v_k \mathbf{1}_{\Omega_{k,R}}  +  h_k \eqsp.
$$
Hence, on $\Omega_{n+1,R}$ all the indicators above are equal to one and
$$
q_n:=\sum_{k=0}^{n}\left(\prod_{r=k+1}^{n} \mathsf Q_r\right) v_k,
\qquad \text{and} \qquad
\mathcal{T}_{3,n} \mathbf{1}_{\Omega_{n+1,R}} = \| q_n \|_{L_2(\mu_R)} \mathbf{1}_{\Omega_{n+1,R}} \eqsp .
$$

Therefore
\begin{align*}
\mathbb E_{\sgd}\|q_n\|_{L_2(\mu_R)}^2
& =
\mathbb E_{\sgd}\|v_n\mathbf{1}_{\Omega_{n,R}}+h_n\|^2_{L_2(\mu_R)} \\
& =
\mathbb E_{\sgd} \left[ \mathbf{1}_{\Omega_{n,R}} \|v_n \|^2_{L_2(\mu_R)} \right] +\mathbb E_{\sgd}\|h_n\|^2_{L_2(\mu_R)} +2\mathbb E_{\sgd} \left[ \mathbf{1}_{\Omega_{n,R}} \langle v_n,h_n\rangle_{L_2(\mu_R)}  \right] \eqsp.
\end{align*}

\paragraph{Upper bound on $\mathbb E_{\sgd}\|v_k\|^2_{L_2(\mu_R)}$ for all $0\le k \le n$.}

Recalling that
$$
v_k(\x)
=
\gamma_k\,c(t^{(k)})\,K_k(\x,\ora\X^{(k)})
\left(\Delta_k(\ora\X^{(k)})+\xi_k\right) \mathbf{1}_{\{ \|\ora X^{(k)}_{t^{(k)}} \| \leq R \}}
-
\gamma_k (\mathsf K_k\Delta_k)(\x) \eqsp,
$$
we get
\begin{align*}
\mathbb E_{\sgd} \left[ v_k(\x)\mid \mathcal F_{k} \right]
&=
\gamma_k 
\mathbb E_{\sgd} \left[
c(t^{(k)}) K_k(\x,\ora\X^{(k)})
\left( \Delta_k(\ora\X^{(k)})+\xi_k \right)  \mathbf{1}_{\{ \|\ora X^{(k)}_{t^{(k)}} \| \leq R \}} \middle| \mathcal F_{k} \right] -
\gamma_k (\mathsf K_k\Delta_k)(\x) \eqsp,
\end{align*}
because conditionally on the initialization, the maps $K_k(\cdot,\cdot)$ and
$\Delta_k(\cdot)$ are $\mathcal F_{k}$-measurable, since $\param[k]$ is
determined by the samples up to iteration $k-1$. Moreover, by definition of $\xi_k$
$$
\mathbb E_{\sgd}[\xi_k\mid \ora\X^{(k)},\mathcal F_{k}]=0 \eqsp.
$$
Using that $\mathcal F_{k} \subset \sigma( \ora \X^{(k)} , \mathcal F_{k})$ we get
\begin{align*}
&\mathbb E_{\sgd}\left[
c(t^{(k)}) K_k(\x,\ora\X^{(k)})
\left(\Delta_k(\ora\X^{(k)})+\xi_k\right) \mathbf{1}_{\{ \|\ora X^{(k)}_{t^{(k)}} \| \leq R \}}
\middle| \mathcal F_{k}
\right] \\
&\quad=
\mathbb E_{\sgd}\left[
c(t^{(k)}) K_k(\x,\ora\X^{(k)}) \mathbf{1}_{\{ \|\ora X^{(k)}_{t^{(k)}} \| \leq R \}}
 \left( \Delta_k(\ora\X^{(k)})+  \mathbb E_{\sgd} \left[ \xi_k\mid \ora\X^{(k)} , \mathcal F_{k}\right] \right)
 \middle|  \mathcal F_{k}
\right] \\
&\quad=
\mathbb E_{\sgd} \left[
c(t^{(k)}) K_k(\x,\ora\X^{(k)}) \Delta_k(\ora\X^{(k)}) \mathbf{1}_{\{ \|\ora X^{(k)}_{t^{(k)}} \| \leq R \}}
\middle| \mathcal F_{k}
\right] \\
& \quad =
(\mathsf K_k\Delta_k)(\x) \eqsp.
\end{align*}

Therefore,
\begin{equation} \label{eq:martingal_identity_vk}
\mathbb E_{\sgd} [v_k(\x)\mid \mathcal F_{k}] = 0 \eqsp.
\end{equation}
Setting,
$$
Y_k(\x)
\eqdef
c(t^{(k)}) K_k(\x,\ora\X^{(k)})
\left( \Delta_k(\ora\X^{(k)})+\xi_k \right) \mathbf{1}_{\{ \|\ora X^{(k)}_{t^{(k)}} \| \leq R \}} \eqsp.
$$
we proved that
$$
v_k(\x)
=
\gamma_k\Bigl(
Y_k(\x)-\mathbb E_{\sgd} [Y_k(\x)\mid \mathcal F_{k}]
\Bigr) \eqsp.
$$
Hence, using the conditional variance bound
\begin{align*}
\mathbb E_{\sgd} \left[\|v_k\|^2_{L_2(\mu_R)} \middle| \mathcal F_{k}\right]
&=
\gamma_k^2\,
\mathbb E_{\sgd} \left[
\big\|Y_k-\mathbb E_{\sgd} [Y_k \mid \mathcal F_{k}]\big\|^2_{L_2(\mu_R)}
\middle| \mathcal F_{k}
\right] \\
&\le
\gamma_k^2
\mathbb E_{\sgd} \left[
\|Y_k\|^2_{L_2(\mu_R)}
 \middle| \mathcal F_{k}
\right] \eqsp,
\end{align*}
so that
\begin{align*}
\|Y_k\|^2_{L_2(\mu_R)}
&=
\int
\left\|
c(t^{(k)}) K_k(\x,\ora\X^{(k)})
\left(\Delta_k(\ora\X^{(k)})+\xi_k\right)
\mathbf{1}_{\{ \|\ora X^{(k)}_{t^{(k)}} \| \leq R \}} \right\|^2
\mu_R( \rmd \x) \\
&\le
c(t^{(k)})^2
\|\Delta_k(\ora\X^{(k)})+\xi_k\|^2 \mathbf{1}_{\{ \|\ora X^{(k)}_{t^{(k)}} \| \leq R \}}
\int
\|K_k(\x,\ora\X^{(k)})\|^2
\mu_R(\rmd \x) \eqsp.
\end{align*}
Multiplying the above by the $\mathcal{F}_k-$ measurable indicator $\mathbf{1}_{\Omega_{k,R}}$ and taking expectation, we obtain
$$
\mathbb E_{\sgd} \left[ \mathbf{1}_{\Omega_{k,R}} \|v_k\|^2_{L_2(\mu_R)} \right]
\le
\gamma_k^2 
\E_{\sgd}  \left[
c(t^{(k)})^2
\|\Delta_k(\ora\X^{(k)})+\xi_k\|^2  \mathbf{1}_{\Omega_{k+1,R}}
\int
\|K_k(\x,\ora\X^{(k)})\|^2
 \mu_R(\rmd \x)
\right] \eqsp.
$$
where we used that $ \mathbf{1}_{\Omega_{k,R}} \mathbf{1}_{\{ \|\ora X^{(k)}_{t^{(k)}} \| \leq R \}} =  \mathbf{1}_{\Omega_{k+1,R}}$.

\paragraph{Recursive expression for $\mathbb E_{\sgd} \|q_n\|^2_{L_2(\mu_R)}$.}

Moreover, since $h_n$ and $\mathbf{1}_{\Omega_{n,R}}$ is $\mathcal F_{n}$-measurable and using \eqref{eq:martingal_identity_vk}, we obtain
\begin{align*}
\mathbb E_{\sgd} \left[ \mathbf{1}_{\Omega_{n,R}} \langle v_n,h_n\rangle_{L_2(\mu_R)} \right]
& =
\mathbb E_{\sgd} \Big[ \mathbf{1}_{\Omega_{n,R}} \mathbb E_{\sgd}[\langle v_n,h_n\rangle_{L_2(\mu_R)} \mid \mathcal F_{n}]\Big] \\
& =
\mathbb E_{\sgd} \Big[ \mathbf{1}_{\Omega_{n,R}} \big\langle \mathbb E_{\sgd} [v_n\mid \mathcal F_{n}], h_n\big\rangle_{L_2(\mu_R)}  \Big] \\
& =0 \eqsp.
\end{align*}
Hence
$$
\E_{\sgd} \|q_n\|^2_{L_2(\mu_R)}
=
\E_{\sgd} \left[ \mathbf{1}_{\Omega_{n,R}} \|v_n\|^2_{L_2(\mu_R)} \right] +\E_{\sgd} \|h_n\|^2_{L_2(\mu_R)} \eqsp.
$$
Using Lemma \ref{lem:operator_norm_bound} with $\gamma_k \leq \Gamma_R$, we have that $\|\mathsf Q_n\|_{\rm op} \le 1$, and therefore
$$
\|h_n\|_{L_2(\mu_R)} =\|\mathsf Q_n  q_{n-1}\|_{L_2(\mu_R)} \le \|q_{n-1}\|_{L_2(\mu_R)} \eqsp,
$$
so that
$$
\mathbb E_{\sgd} \|q_n\|^2_{L_2(\mu_R)}
\le
\E_{\sgd} \left[ \mathbf{1}_{\Omega_{n,R}} \|v_n\|^2_{L_2(\mu_R)} \right] +\E_{\sgd} \|q_{n-1}\|^2_{L_2(\mu_R)} \eqsp.
$$
Therefore,
$$
\mathbb E_{\sgd} \|q_n\|^2_{L_2(\mu_R)}
\le
\sum_{k=0}^{n}\mathbb E_{\sgd} \left[ \mathbf{1}_{\Omega_{k,R}} \|v_k\|^2_{L_2(\mu_R)} \right]  \eqsp.
$$

\paragraph{Final bound.} 
By Jensen's inequality,
$$
\E_{\sgd} [\mathcal T_{3,n} \un_{\Omega_{n+1,R}}]
\le \mathbb E_{\sgd} \|q_n\|_{L_2(\mu_R)}
\le \left(\mathbb E_{\sgd} \|q_n\|^2_{L_2(\mu_R)}\right)^{1/2} \eqsp.
$$
Hence,
\begin{align*}
\E_{\sgd} \left[ \mathcal{T}_{3,n} \un_{\Omega_{n+1,R}} \right]
\le
\sqrt{\sum_{k=0}^{n}
\gamma_k^2 
\E_{\sgd} \left[
c(t^{(k)})^2
\|\Delta_k(\ora\X^{(k)})+\xi_k\|^2
\un_{\Omega_{k+1,R}}  \int
\|K_k(\x,\ora\X^{(k)})\|^2
 \mu_R( \rmd \x)
\right]} \eqsp.
\end{align*}
%Equation \eqref{eq:control_term_3} follows from Jensen's inequality.
Moreover, 
\begin{align*}
\sup_{z\in\mathcal{X}_R}
\int
\|K_k(\x,z)\|^2 \mu_R( \rmd \x) \leq \frac{\|A\|^4}{m^2} B_R^4 \mu_R(\mathcal{X}_R) \eqsp.
\end{align*}
Therefore
\begin{align*}
 \E_{\sgd} \left[ \mathcal{T}_{3,n} \un_{\Omega_{n+1,R}} \right] & \leq  \frac{\|A\|^2}{m} B_R^2 \mu_R(\mathcal{X}_R)^{1/2} \sqrt{\sum_{k=0}^{n}
\gamma_k^2 
\E_{\sgd} \left[
c(t^{(k)})^2
\|\Delta_k(\ora\X^{(k)})+\xi_k\|^2 \un_{\Omega_{k+1,R}} 
\right]} \eqsp.
\end{align*}
Using $c_\infty = \sup_{t \in [0,T]} c(t)$, we obtain
\begin{align*}
\E_{\sgd} \left[ \mathcal{T}_{3,n} \un_{\Omega_{n+1,R}} \right] & \leq \frac{\|A\|^2}{m} \sqrt{c_\infty \mu_R(\mathcal{X}_R)} B_R^2 \sqrt{\sum_{k=0}^{n}
\gamma_k^2 
\E_{\sgd} \left[
c(t^{(k)}) \|\Delta_k(\ora\X^{(k)})+\xi_k\|^2 \un_{\Omega_{k+1,R}}
\right]} \eqsp,
\end{align*}
Since $A$ has i.i.d. Rademacher entries, by \citep[Theorem~4.4.3]{vershynin2026hdp}, with probability at least $1-\delta$,
\[
\|A\| \le C\left(\sqrt m+\sqrt d+\sqrt{\log(1/\delta)}\right) \eqsp,
\]
for some universal constant $C>0$. Therefore,
\[
\frac{\|A\|^2}{m}
\le
C\left(
1+\sqrt{\frac d m}
+\sqrt{\frac{\log(1/\delta)}{m}}
\right)^2 \eqsp,
\]
which concludes the proof.
\end{proof}

We let 
$$
\sigma_k^2
\eqdef
\E_{\sgd}\!\left[
\|\Delta_k\|_{L_2(\mu_R)}^2
\mathbf 1_{\Omega_{k,R}}
\right] +\tau_c^2 \eqsp,
$$
where $\tau_c^2 =\E\!\left[ c(\tau)\|\xi\|^2
\right]$ and $ \xi= X_0-f^\star(\ora X_\tau,\tau)$. For the $k$-th SGD sample, we write
$$
\xi_k
\eqdef
X_0^{(k)}-f^\star(\ora \X^{(k)}),
$$
so that $\E[\xi_k\mid \ora \X^{(k)}]=0$ and therefore
\begin{equation} \label{eq:delta_sigma_relation}
\E_{\sgd} \left[
c(t^{(k)})
\|\Delta_k(\ora\X^{(k)})+\xi_k\|^2
\mathbf 1_{\Omega_{k+1,R}}
\right] \le
\E_{\sgd}\!\left[
\|\Delta_k\|_{L_2(\mu_R)}^2
\mathbf 1_{\Omega_{k,R}}
\right]
+
\E \left[c(\tau)\|\xi\|^2\right]
=
\sigma_k^2 \eqsp.
\end{equation}

\begin{lemma}
\label{lem:sigma_recursion}
Assume that $\gamma_k < \Gamma_R$. Then, for all $k\geq 0$, we have
$$
\sigma_{k+1}^2
\le
\left(
1+
\gamma_k B_R^2\sqrt{c_\infty\mu_R(\mathcal X_R)}
\left(
d+\frac{\|A\|^2}{m}
\right)
\right)^2
\sigma_k^2 \eqsp.
$$
% $$
% \sigma_{k+1}^2 \le (1 + 2B_R\sqrt{c_\infty \mu_R(\mathcal{X}_R)} \gamma_k)^2 \sigma_k^2 \eqsp.
% $$
\end{lemma}

\begin{proof}

From Lemma~\ref{lem:delta_rec_matrix} 
$$
\Delta_{k+1} \un_{\Omega_{k+1,R}} = \left( \mathsf{Q}_k   \Delta_k - v_k + \varepsilon_k \right) \un_{\Omega_{k+1,R}} \eqsp.
$$
Therefore, we obtain
\begin{align}
\label{eq:decomposition_sigma_T1T6}
\|\Delta_{k+1}\|^2_{L_2(\mu_R)} \un_{\Omega_{k+1,R}} 
&\le \left( T_1 + T_2 + T_3 + 2T_4 + 2T_5 + 2T_6 \right) \un_{\Omega_{k+1,R}} \eqsp,
\end{align}
where
\begin{align*}
T_1 &:= \|\mathsf Q_k \Delta_k\|^2_{L_2(\mu_R)} \eqsp, \\
T_2 &:= \|v_k\|^2_{L_2(\mu_R)} \eqsp, \\
T_3 &:= \|\varepsilon_k\|^2_{L_2(\mu_R)} \eqsp, \\
T_4 &:= \|\mathsf Q_k \Delta_k\|_{L_2(\mu_R)} \|v_k\|_{L_2(\mu_R)} \eqsp, \\
T_5 &:= \|\mathsf Q_k \Delta_k\|_{L_2(\mu_R)} \|\varepsilon_k\|_{L_2(\mu_R)} \eqsp, \\
T_6 &:= \|v_k\|_{L_2(\mu_R)} \|\varepsilon_k\|_{L_2(\mu_R)} \eqsp.
\end{align*}
We now bound these terms in expectation.
First, by the step-size condition and Lemma \ref{lem:operator_norm_bound},
\begin{align}
\label{eq:T1_bound}
\mathbb E_{\sgd} \left[T_1 \un_{\Omega_{k+1,R}} \right] 
%&= \mathbb E \left[\|\mathsf Q_r \Delta_r\|^2_{L_2(\mu_R)}\right] 
\le \E_{\sgd}\!\left[ \|\Delta_k\|_{L_2(\mu_R)}^2 \un_{\Omega_{k,R}} \right]
%\le \mathbb E \left[\|\Delta_k\|^2_{L_2(\mu_R)} \right] 
= \sigma_k^2-\tau_c^2
\le \sigma_k^2 \eqsp.
\end{align}
where we used that $\Omega_{k+1,R}\subset\Omega_{k,R}$. For the second term,
\begin{align*}
%\label{eq:T2_bound}
\E_{\sgd} \left[ T_2\mathbf 1_{\Omega_{k+1,R}} \right]
& = \E_{\sgd}\!\left[ \|v_k\|_{L_2(\mu_R)}^2 \un_{\Omega_{k+1,R}}
\right]  \\
%& \mathbb E \left[\|v_k\|^2_{L_2(\mu_R)}\right]
& \le \gamma_k^2 
\mathbb E_{\sgd} \left[
c(t^{(k)})^2 \|\Delta_k(\ora\X^{(k)})+\xi_k\|^2 \un_{\Omega_{k+1,R}}
\int \|K_k(\x,\ora\X^{(k)})\|^2
\mu_R(d\x) \right] \eqsp.
%&\le c_\infty B_R^2 \mu_R(\mathcal{X}_R) \gamma_k^2 \sigma_k^2 \eqsp.
\end{align*}
On the event $\Omega_{k+1,R}$, we have $\ora\X^{(k)}\in\mathcal X_R$. Hence,
$$
\int \|K_k(\x,\ora\X^{(k)})\|^2 \mu_R(\rmd \x)
\le \frac{\|A\|^4}{m^2} B_R^4 \mu_R(\mathcal X_R) \eqsp,
$$
and since $c(t)^2\le c_\infty c(t)$, and using \eqref{eq:delta_sigma_relation}
\begin{align}
\label{eq:T2_bound}
\E_{\sgd} \left[ T_2\mathbf 1_{\Omega_{k+1,R}} \right]
\le \gamma_k^2 c_\infty
\frac{\|A\|^4}{m^2}
B_R^4
\mu_R(\mathcal X_R)
\sigma_k^2 \eqsp.
\end{align}
For $T_3$, we use a control similar to \eqref{eq:var_eps_bound}, on $\Omega_{k+1,R}$, for every $\x\in\mathcal X_R$,
$$
\|\varepsilon_k(\x,\ora\X^{(k)})\| \le \gamma_k c(t^{(k)})dB_R^2 \|\loss_k\| \eqsp,
$$
where we used $\widetilde S_k(\x)\le m$ as well. Hence,
$$
\|\varepsilon_k\|_{L_2(\mu_R)}^2
\mathbf 1_{\Omega_{k+1,R}}
\le \gamma_k^2 c(t^{(k)})^2 d^2B_R^4 \mu_R(\mathcal X_R) \|\loss_k\|^2 \un_{\Omega_{k+1,R}} \eqsp.
$$
By definition of $\loss_k$ in \eqref{eq:def_loss_k},
$$
\|\loss_k\|^2 = \|\Delta_k(\ora\X^{(k)})+\xi_k\|^2 \eqsp.
$$
Using the inequality $c(t)^2 \le c_\infty c(t)$ and \eqref{eq:delta_sigma_relation}, we obtain
\begin{align}
\label{eq:T3_bound}
\E_{\sgd}\!\left[ T_3\mathbf 1_{\Omega_{k+1,R}} \right]
&\le \gamma_k^2 c_\infty d^2B_R^4 \mu_R(\mathcal X_R) \sigma_k^2 \eqsp.
\end{align}
For the cross terms, Cauchy-Schwarz gives
\begin{align*}
\E_{\sgd} \left[ T_4\mathbf 1_{\Omega_{k+1,R}} \right]
&\le \gamma_k \frac{\|A\|^2}{m} B_R^2 \sqrt{c_\infty\mu_R(\mathcal X_R)} \sigma_k^2 \eqsp.
\end{align*}
Similarly,
\begin{align*}
\E_{\sgd} \left[ T_5\mathbf 1_{\Omega_{k+1,R}} \right] &\le \gamma_k
dB_R^2 \sqrt{c_\infty\mu_R(\mathcal X_R)} \sigma_k^2 \eqsp.
\end{align*}
Finally,
\begin{align*}
\E_{\sgd} \left[ T_6\mathbf 1_{\Omega_{k+1,R}} \right]
&\le \gamma_k^2 c_\infty d \frac{\|A\|^2}{m} B_R^4 \mu_R(\mathcal X_R) \sigma_k^2 \eqsp.
\end{align*}
Combining the previous upper bounds yields,
$$
\sigma_{k+1}^2
\le
\left(
1+
\gamma_k B_R^2\sqrt{c_\infty\mu_R(\mathcal X_R)}
\left(
d+\frac{\|A\|^2}{m}
\right)
\right)^2
\sigma_k^2 \eqsp,
$$
which concludes the proof.
% --------------------
% Substituting these bounds into \eqref{eq:decomposition_sigma_T1T6}, we obtain
% \begin{align*}
% \sigma_{k+1}^2
% &\le
% \sigma_k^2
% + B_R^2 c_\infty \mu_R(\mathcal{X}_R) \gamma_k^2 \sigma_k^2
% + B_R^2 c_\infty \mu_R(\mathcal{X}_R) \gamma_k^2 \sigma_k^2 \\
% &\qquad
% +2B_R\sqrt{c_\infty \mu_R(\mathcal{X}_R)} \gamma_k \sigma_k^2
% +2B_R\sqrt{c_\infty \mu_R(\mathcal{X}_R)} \gamma_k \sigma_k^2 + 2B_R^2 c_\infty \mu_R(\mathcal{X}_R) \gamma_k^2 \sigma_k^2 \eqsp.
% \end{align*}
% Hence,
% \begin{align*}
% \sigma_{k+1}^2
% &\le
% \Bigl(1 + 4B_R\sqrt{c_\infty \mu_R(\mathcal{X}_R)} \gamma_k + 4 B_R^2 c_\infty \mu_R(\mathcal{X}_R) \gamma_k^2 \Bigr) \sigma_k^2 \eqsp.
% \end{align*}
% Therefore,
% \[
% \sigma_{k+1}^2
% \le (1 + 2B_R \sqrt{c_\infty \mu_R(\mathcal{X}_R)} \gamma_k)^2\sigma_k^2 \eqsp,
% \]
% which concludes the proof.
\end{proof}

\begin{lemma}
\label{lem:controle_third_term}
Assume that $\gamma_k \leq \Gamma_R$ for every $k\ge 0$ and that $2 \Lambda_R C_\gamma < 1$. Then, there exists a universal constant $C>0$ such that, with probability at least $1-\delta$,
\begin{align*}
\E\left[\mathcal T_{3,n} \un_{\Omega_{n+1,R}} \right]
&\le C\left(
1+\sqrt{\frac d m}
+\sqrt{\frac{\log(1/\delta)}{m}}
\right)^2
\sqrt{c_\infty \mu_R(\mathcal{X}_R)} B_R^2 C_{\gamma} e^{\Lambda_R C_{\gamma}} \left(1+\frac{1}{1-2\Lambda_R C_\gamma}\right)^{1/2} \sigma_0 \eqsp,
\end{align*}
where $\Lambda_R = B^2_R\sqrt{c_\infty \mu_R(\mathcal{X}_R)}\left(d+\|A\|^2/m\right)$.
\end{lemma}

% Let 
% $$
% \Lambda_R \eqdef C B_R^2\sqrt{c_\infty \mu_R(\mathcal X_R)}\left(d+\frac{\|A\|_{\rm op}^2}{m}\right).
% $$
% \begin{proof}
% Recall that $\gamma_k=C_\gamma/(k+1)$, and that by Lemma~\ref{lem:sigma_recursion}, $\sigma_{k+1} \le (1 + \Lambda_R \gamma_k) \sigma_k$. Hence
% \begin{align*}
% \sigma_k \le \prod_{r=0}^{k-1} \left(1+\frac{\Lambda_R C_\gamma}{r+1}\right)\sigma_0 & \le \exp \left\{ \sum_{r=0}^{k-1}\frac{\Lambda_R C_{\gamma}}{r+1} \right\}\sigma_0 
% \leq \rme^{\Lambda_R C_{\gamma}} (k+1)^{\Lambda_R C_\gamma} \sigma_0
% \end{align*}
% Therefore,
% $$
% \gamma_k\sigma_k
% \le
% C_\gamma \rme^{\Lambda_R C_\gamma} (k+1)^{\Lambda_R C_\gamma-1}\sigma_0 .
% $$

% \end{proof}

\begin{proof}
For all $k \geq 0$, we have
\begin{align*}
\gamma_k \sigma_k
&\le \frac{C_{\gamma}}{k+1}
\prod_{r=0}^{k-1}\left(1+\frac{\Lambda_R C_{\gamma}}{r+1}\right)\sigma_0 \\
&\le \frac{C_{\gamma}}{k+1}\exp\!\biggl(\sum_{r=0}^{k-1}\frac{\Lambda_R C_{\gamma}}{r+1}\biggr)\sigma_0 \\
&\le \frac{C_{\gamma}}{k+1}\exp\!\bigl(\Lambda_R C_{\gamma}(\log(k+1)+1)\bigr)\sigma_0 \\
&\le C_{\gamma} e^{\Lambda_R C_{\gamma}} (k+1)^{\Lambda_R C_{\gamma}-1} \sigma_0 .
\end{align*}
Using Lemma~\ref{lem:controle_v_term} together with Lemma~\ref{lem:sigma_recursion}, we obtain
\begin{align*}
\E_{\sgd} \left[ \mathcal{T}_{3,n} \un_{\Omega_{n+1,R}} \right]
&\leq C
\left(
1+\sqrt{\frac d m}
+\sqrt{\frac{\log(1/\delta)}{m}}
\right)^2
\sqrt{c_\infty \mu_R(\mathcal{X}_R)} B_R^2 \left(\sum_{k=0}^{n}
\gamma_k^2 \sigma_k^2 \right)^{1/2} \\
&\leq C
\left(
1+\sqrt{\frac d m}
+\sqrt{\frac{\log(1/\delta)}{m}}
\right)^2
\sqrt{c_\infty \mu_R(\mathcal{X}_R)} B_R^2 C_{\gamma} e^{\Lambda_R C_{\gamma}} \sigma_0
\left(\sum_{k=0}^n (k+1)^{2\Lambda_R C_{\gamma}-2}\right)^{1/2} \eqsp.
\end{align*}
Since $C_{\gamma}<\frac{1}{2\Lambda_R}$,
\[
\sum_{k=0}^{\infty}(k+1)^{2\Lambda_R C_{\gamma}-2}
\le 1+\int_1^\infty x^{2\Lambda_R C_{\gamma}-2} \rmd x
= 1+\frac{1}{1-2 \Lambda_R C_{\gamma}} \eqsp,
\]
which concludes the proof.
\end{proof}

\subsection{Fourth term}

\begin{lemma}
\label{lem:controle_epsilon_term}
On $\Omega_{n+1,R}$, there exists a universal constant $C>0$ such that, for all $n\ge0$, with probability at least $1-\delta$,
% $$
% \E_{\sgd} \left[ \mathcal T_{4,n} \un_{\Omega_{n+1,R}} \right] \le C \sqrt{2} \gamma_k \sigma_k d B_R^2 \sqrt{c_\infty \mu_R( \mathcal{X}_R)}
% \left( \frac{4B_R\sqrt{d}}{\sqrt{2\pi m}}
% \mathbb{E}_{\sgd} \left[\eta_{k+1}\right] + 2C \sqrt{\frac{d+\log(1/\delta)}{m}}
% \right)^{1/2} \eqsp,
% $$
\begin{align*}
&\E_{\sgd}\left[\mathcal T_{4,n} \un_{\Omega_{n+1,R}} \right] 
\le C d B_R^2 \sqrt{c_\infty \mu_R( \mathcal{X}_R)} \sum_{k=0}^{n}\gamma_k\sigma_k
\left( \frac{4B_R\sqrt{d}}{\sqrt{2\pi m}}
\mathbb{E}_{\sgd} \left[\eta_{k+1} \un_{\Omega_{n+1,R}}\right] + 2C \sqrt{\frac{d+\log(1/\delta)}{m}}
\right)^{1/2}  \eqsp.
\end{align*}
% {\color{gray}
% $$
% \sum_{k=0}^n \|\varepsilon_k\|_{L_2(\mu_R)}
% \le
% d B^2_R \sqrt{\mu_R(\mathcal{X}_R)} \sum_{k=0}^n \gamma_k c(t^{(k)}) 
% \|\loss_k\| \left( \frac{4 B_R \sqrt{d}}{\sqrt{2\pi m}} \eta_{k+1} + 2C \sqrt{\frac{d+\log(1/\delta)}{m}} \right) \eqsp,
% $$
% }
where
\begin{align*}
\eta_{k+1} \eqdef 
\sum_{r=0}^{k} \gamma_r c(t^{(r)}) \|\loss_r\| \eqsp.
\end{align*}
and 
$$\loss_k \eqdef \net \left(\ora \X^{(k)};\param^{(k)}\right)-X_0^{(k)} = \net \left(\ora \X^{(k)};\param^{(k)}\right) - \left(f^\star \left(\ora \X^{(k)}\right)+\xi_k\right)$$.
\end{lemma}

\begin{proof}
Recall that for all $0 \le k \le n$,
\begin{align*}
\varepsilon_k(u, \ora \X^{(k)}) 
& = f_{\theta}\left(\x, \param[k] \right) - f_{\theta} \left(\x, \param[k+1] \right) 
- \gamma_k c(t^{(k)})  K_k(\x, \ora \X^{(k)}) \loss_k  \\
& = f_{\theta}\left(\x, \param[k] \right) - f_{\theta} \left(\x, \param[k+1] \right) 
- \gamma_k \frac{c(t^{(k)})}{m}  (\x^\top \ora \X^{(k)}) A D_{\x, \param[k]} D_{\ora \X^{(k)}, \param[k]} A^\top \loss_k \eqsp.
\end{align*}
and the SGD weight update is given by
$$
\param[k+1]-\param[k]
=
- \gamma_k c(t^{(k)})  \frac{1}{\sqrt m} D_{\ora \X^{(k)}, \param[k]} A^\top \loss_k (\ora \X^{(k)})^\top \eqsp.
$$
Moreover, using \eqref{eq:net_def},
\begin{align*}
\net(\x,\param[k+1]) - \net(\x,\param[k]) & = \frac{1}{\sqrt m} A \left( D_{\x,\param[k+1]} \param[k+1] - D_{\x,\param[k]} \param[k] \right) \x \\
& = \frac{1}{\sqrt m} A \left(  D_{\x,\param[k+1]} - D_{\x,\param[k]} \right) \param[k+1]  \x \\
& \quad + \frac{1}{\sqrt m} A  D_{\x,\param[k]} \left( \param[k+1] -  \param[k] \right) \x \eqsp,
\end{align*}
and therefore,
\begin{align} \label{eq:expression_vareps}
 \varepsilon_k(u, \ora \X^{(k)})  =  \frac{1}{\sqrt m} A \left(  D_{\x,\param[k]} - D_{\x,\param[k+1]} \right) \param[k+1]  \x \eqsp.
\end{align}
Set
$$
u_k(\x)
\eqdef
\left(D_{\x,\param[k]}-D_{\x,\param[k+1]}\right)\param[k+1]\x
\in \mathbb R^m.
$$
Then, for every $i\in \{1, \cdots, m \}$,
$$
(u_k(\x))_i
=
\Bigl(
\mathbf 1_{\{(\param[k]_i)^\top \x\ge 0\}}
-
\mathbf 1_{\{(\param[k+1]_i)^\top \x\ge 0\}}
\Bigr)\,
(\param[k+1]_i)^\top \x \eqsp.
$$
Hence $(u_k(\x))_i=0$ whenever
$$
\mathbf 1_{\{(\param[k]_i)^\top \x\ge 0\}}
=
\mathbf 1_{\{(\param[k+1]_i)^\top \x\ge 0\}} \eqsp,
$$
hence, $u_k(\x)_i = 0$ for every $i \notin \widetilde O_k$, with
$$
\widetilde O_k(\x)
=
\left\{
i\in[m]:
\mathbf 1_{\{(\param[k]_i)^\top \x\ge 0\}}
\neq
\mathbf 1_{\{(\param[k+1]_i)^\top \x\ge 0\}}
\right\} \quad \widetilde S_k (\x) = \left| \widetilde O_k(\x) \right| \eqsp.
$$
By definition of $u_k(u)$, we have
$$
\varepsilon_k(\x,\ora \X^{(k)}) = \frac{1}{\sqrt{m}}
 A u_k(u) = \frac{1}{\sqrt{m}} \sum_{i\in \widetilde O_k(\x)} A_{\cdot i}u_{k,i}(\x) \eqsp.
$$
Taking norm,
\begin{equation} \label{eq:bound_on_eps}
\|\varepsilon_k(\x,\ora \X^{(k)}) \|
\le \frac{1}{\sqrt m }
\sum_{i\in \widetilde O_k(\x)} \|A_{\cdot i}\| |u_{k,i}(\x)|
\end{equation}
Now let $i\in \widetilde O_k(\x)$. Since the signs of $(\param[k]_i)^\top \x$ and $(\param[k+1]_i)^\top \x$ differ, we have
$$
|(\param[k+1]_i)^\top \x|
\le
|(\param[k+1]_i-\param[k]_i)^\top \x|
\le
\|\param[k+1]_i-\param[k]_i\| \|\x\|\eqsp.
$$
Using the row-wise SGD update from Lemma~\ref{lem:sgd_updates},
$$
\param[k+1]_i-\param[k]_i
=
-\gamma_k c(t^{(k)}) \frac{1}{\sqrt m}
\mathbf 1_{\{(\param[k]_i)^\top \ora \X^{(k)}\ge 0\}}
A_{\cdot i}^\top \loss_k \,\ora \X^{(k)} \eqsp,
$$
we obtain
$$
\|\param[k+1]_i-\param[k]_i\|
\le
\gamma_k c(t^{(k)}) \frac{1}{\sqrt m}
|A_{\cdot i}^\top \loss_k|\,\|\ora \X^{(k)}\| \eqsp.
$$
Therefore, for $i\in \widetilde O_k(\x)$
$$
|u_{k,i}(\x)|
\le
\gamma_k c(t^{(k)}) \frac{1}{\sqrt m}
|A_{\cdot i}^\top \loss_k| \|\ora \X^{(k)}\| \|\x\| \eqsp.
$$
Plugging this bound in \eqref{eq:bound_on_eps} yields
$$
\|\varepsilon_k(\x,\ora \X^{(k)}) \|
\le \gamma_k c(t^{(k)})\frac{ \|\ora \X^{(k)}\| \|\x\| }{ m }
\sum_{i\in \widetilde O_k(\x)} \|A_{\cdot i}\| |A_{\cdot i}^\top \loss_k| \eqsp.
$$
Now using Cauchy-Schwarz and the fact that the entries of $A$ are Rademacher 
$$
\|\varepsilon_k(\x,\ora \X^{(k)}) \|
\le \gamma_k c(t^{(k)})\frac{ \|\ora \X^{(k)}\| \|\x\| }{ m } d \widetilde S_k(\x)  \| \loss_k \|
\eqsp.
$$
Moreover, on the truncated event $\Omega_{n+1,R}$, $\|\ora \X^{(k)}\|\le B_R$ and $\|\x\|\le B_R$, so
\begin{equation} \label{eq:var_eps_bound}
\|\varepsilon_k(\x,\ora \X^{(k)})\|
\le
\gamma_k c(t^{(k)}) \frac{d B_R^2}{m}
\widetilde S_k(\x) \|\loss_k\| \eqsp.
\end{equation}
Next, note that
$$
\widetilde O_k(\x)\subset O_k(\x)\cup O_{k+1}(\x) \eqsp,
$$
with $O_k(\x)$ defined as in \eqref{eq:sign_flip}. Indeed, if the sign flips between steps $k$ and $k+1$, then at least one of these two signs differs from the sign at initialization. Hence
$$
\widetilde S_k(\x)\le S_k(\x)+S_{k+1}(\x) \eqsp,
$$
with $S_k(\x)$ defined as \eqref{eq:sign_flip}. 
Therefore, on $\Omega_{n+1,R}$
\begin{align*}
\|\varepsilon_k(\x,\ora \X^{(k)})\|
&\le
\gamma_k c(t^{(k)}) dB_R^2
\left(
\frac{\|S_{k} \|_{\infty,R}+\|S_{k+1} \|_{\infty,R}}{m}
\right)
\|\loss_k\| \eqsp.
\end{align*}
Using $c(t)\le c_\infty$, taking expectation and applying the Cauchy--Schwarz inequality, we obtain
\begin{align*}
\mathbb{E}_{\sgd}\left[ \|\varepsilon_k \|_{L_2(\mu_R)} \un_{\Omega_{n+1,R}}  \right]
\le \gamma_k \sigma_k d B_R^2 \sqrt{c_\infty \mu_R( \mathcal{X}_R)}
\left( \mathbb{E}_{\sgd} \left[ \left( \frac{\|S_k\|_{\infty,R}+\|S_{k+1}\|_{\infty,R}}{m}
\right)^2 \un_{\Omega_{n+1,R}} \right] \right)^{1/2} \eqsp,
\end{align*}
where we used that $\mathbb{E}_{\sgd} \left[ c(t^{(k)}) \| \loss_k \|^2 \un_{\Omega_{n+1,R}} \right] \leq \sigma_k^2$. Using the fact that $\|S_k\|_{\infty,R}\le m$, we have
\begin{align*}
\mathbb{E}_{\sgd}\left[ \|\varepsilon_k\|_{L_2(\mu_R)} \un_{\Omega_{n+1,R}} \right]
\le \sqrt{2} \gamma_k \sigma_k d B_R^2 \sqrt{c_\infty \mu_R(\mathcal{X}_R)}
\left( \mathbb{E}_{\sgd} \left[ \frac{\|S_k\|_{\infty,R}+\|S_{k+1}\|_{\infty,R}}{m} \un_{\Omega_{n+1,R}} \right] \right)^{1/2} \eqsp.
\end{align*}
Using Lemma \ref{lem:bound:S_k} and the fact that $\eta_k \leq \eta_{k+1}$, we get that, with probability at least $1-\delta$, there exist a universal constant $C>0$, such that
\begin{align*}
\frac{\|S_k\|_{\infty,R} + \|S_{k+1}\|_{\infty,R}}{m}
\le
\frac{4B_R\sqrt{d}}{\sqrt{2\pi m}} \eta_{k+1} + 2C \sqrt{\frac{d+\log(1/\delta)}{m}} % \eqdef \frac{\mathfrak{S}_k}{m} 
\eqsp.
\end{align*}
Therefore, there exists $C>0$ such that
\begin{align*}
\mathbb{E}_{\sgd}\left[ \|\varepsilon_k\|_{L_2(\mu_R)} \un_{\Omega_{n+1,R}} \right] &\le C \sqrt{2} \gamma_k \sigma_k d B_R^2 \sqrt{c_\infty \mu_R( \mathcal{X}_R)}
\left( \frac{4B_R\sqrt{d}}{\sqrt{2\pi m}}
\mathbb{E}_{\sgd} \left[\eta_{k+1} \un_{\Omega_{n+1,R}}\right] + 2C \sqrt{\frac{d+\log(1/\delta)}{m}}
\right)^{1/2} \eqsp.
\end{align*}

% Using Lemma \ref{lem:bound:S_k} and the fact that $\eta_k \leq \eta_{k+1}$, we get that, with probability at least $1-\delta$,
% \begin{align*}
% \frac{\|S_k\|_\infty + \|S_{k+1}\|_\infty}{m}
% \le
% \frac{4B_R\sqrt{d}}{\sqrt{2\pi m}} \eta_{k+1} + 2C \sqrt{\frac{d+\log(1/\delta)}{m}} % \eqdef \frac{\mathfrak{S}_k}{m} 
% \eqsp.
% \end{align*}
% Hence, for every $\x\in\mathcal{X}_R$,
% $$
% \|\varepsilon_k(\x,\ora \X^{(k)})\|
% \le
% \gamma_k c(t^{(k)})dB_R^2
% \left[
% \frac{4B_R\sqrt d}{\sqrt{2\pi m}}\eta_{k+1}
% +
% 2C\sqrt{\frac{d+\log(1/\delta)}{m}}
% \right]
% \|\loss_k\| \eqsp.
% $$
% % $$
% % \|\varepsilon_k(\x,\ora \X^{(k)})\|
% % \le
% % \gamma_k c(t^{(k)}) \frac{d B^2}{m}
% % \mathfrak{S}_k \|\loss_k\| \eqsp.
% % $$
% The right-hand side does not depend on $\x$, so 
% % $$
% % \|\varepsilon_k\|_{L_2(\mu_R)}
% % \le
% % \gamma_k c(t^{(k)}) \frac{d B^2}{m}
% % \mathfrak{S}_k \|\loss_k\| \sqrt{\mu_R(\R^{d+1})} \eqsp.
% % $$
% $$
% \|\varepsilon_k\|_{L_2(\mu_R)}
% \le
% \gamma_k c(t^{(k)})dB_R^2
% \left[
% \frac{4B_R\sqrt d}{\sqrt{2\pi m}}\eta_{k+1}
% +
% 2C\sqrt{\frac{d+\log(1/\delta)}{m}}
% \right]
% \|\loss_k\|
% \sqrt{\mu_R(\mathcal{X}_R)} \eqsp.
% $$
% Then, summing over $k = 0, \cdots, n$ gives the result.

% % $$
% % \sum_{r=0}^k \|\varepsilon_r\|_{L_2(\mu_R)}
% % \le
% % d B^2 \sqrt{\mu_R(\R^{d+1})} \sum_{r=0}^k \gamma_r c(t^{(r)}) 
% % \mathfrak{S}_r \|\loss_r\| \left( \frac{4B\sqrt{d}}{\sqrt{2\pi m}} \eta_{r+1} + 2C \sqrt{\frac{d+\log(1/\delta)}{m}} \right) \eqsp.
% % $$
\end{proof}

\subsection{Final bound}

We prove by induction that for every $0\le n\le N$,
\begin{equation} \label{eq:induction}
\E_{\sgd}\!\left[
\|\Delta_n\|_{L_2(\mu_R)}
\mathbf 1_{\Omega_{n,R}}
\right] \le S_n \eqsp,
\end{equation}
with
$$
S_n
\eqdef
2 \left( E_{1,n}+E_{2,N}+E_{3,N} \right) \eqsp,
$$
where
\begin{align*}
E_{1,n}
& \eqdef
\inf_{r\in\mathbb N}
\left\{
\left(\prod_{k=0}^{n-1}(1-\gamma_k\kappa_r)\right)
\|\Delta_0\|_{L_2(\mu_R)}
+
\mathcal R(\Delta_0,r)
\right\} \\
E_{2,N}
& \eqdef
C \mu_R(\mathcal X_R)  B_R^2
d C_\gamma\log(N+1)
\sqrt{\frac{d\log(m)+\log(1/\delta)}{m}}  \sigma_0 \\
E_{3,N}
& \eqdef
C \left( 1+\sqrt{\frac d m}
+\sqrt{\frac{\log(1/\delta)}{m}} \right)^2
\sqrt{c_\infty\mu_R(\mathcal X_R)}
B_R^2 C_\gamma
\rme^{\Lambda_R C_\gamma}
\left(
1+\frac{1}{1-2\Lambda_R C_\gamma}
\right)^{1/2}
\sigma_0  \eqsp.
\end{align*}
% Once this is proven, Theorem \ref{th:rate_two_layer_SGD} follows from \eqref{eq:equality_to_finsh_thm_2_layers}.
% The bound~\eqref{eq:induction} provides a uniform control of the iterates
% up to the stopping horizon $N$ and will be used below to close the
% bootstrap argument.
Since $0\le 1-\gamma_k\kappa_r\le 1$, for every $0\le n\le N$,
\[
E_{1,n}
\le
\inf_{r\in\mathbb N}
\left\{
\|\Delta_0\|_{L_2(\mu_R)}
+\mathcal R(\Delta_0,r)
\right\} \eqsp.
\]
Moreover, since $\Delta_0\in L_2(\mu_R)$,
$\mathcal R(\Delta_0,r)\to 0$ as $r\to\infty$. Hence,
\[
E_{1,n}\le \|\Delta_0\|_{L_2(\mu_R)} \eqsp.
\]
Hence,
\[
S_n
\le
\widetilde S_N
\eqdef
2\left(
\|\Delta_0\|_{L_2(\mu_R)}
+
E_{2,N}
+
E_{3,N}
\right)
\eqsp.
\]
The quantity $\widetilde S_N$ will be used only as a uniform control of
the intermediate iterates in the bootstrap estimates.

The base case $n=0$ is immediate. Now suppose that \eqref{eq:induction} holds for all $0\le k\le n$. We prove that it also holds for $n+1$.
By \eqref{eq:expected_squared_error_decomp}, we obtain
\begin{align*}
\E_{\sgd} \left[\|\Delta_{n+1}\|_{L_2(\mu_R)} \un_{\Omega_{n+1,R}} \right]
\le
\E_{\sgd} [ \left(\mathcal T_{1,n}+\mathcal T_{2,n}+\mathcal T_{3,n}+\mathcal T_{4,n}\right)\un_{\Omega_{n+1,R}}] \eqsp.
\end{align*}
For the first term, Lemma~\ref{lem:first_term_bound} gives
\begin{align*}
\E_{\sgd} \left[
\mathcal T_{1,n}\mathbf 1_{\Omega_{n+1,R}}
\right] &\le \inf_{r\in\mathbb{N}}\left\{
\left(\prod_{k=0}^{n}(1-\gamma_k\kappa_r)\right) \|\Delta_0\|_{L_2(\mu_R)}
+ \mathcal{R}(\Delta_0,r)
\right\}  \eqsp.
\end{align*}
For the second term, using Lemma~\ref{lem:second_term_bound}, together with the Cauchy--Schwarz inequality and the fact that $\gamma_k = C_\gamma/(k+1)$, we obtain
\begin{align*}
\E_{\sgd} \left[ \mathcal T_{2,n} \un_{\Omega_{n+1,R}} \right]
&\le C \mu_R(\mathcal{X}_R) B_R^2 \|\Delta_0\|_{L_2(\mu_R)}
\Bigg[
\frac{d^{3/2} B_R}{\sqrt{m}}
\E_{\sgd} \left[
\left(\sum_{k=0}^{n} \gamma_k \eta_k \right)
\un_{\Omega_{n+1,R}}
\right] \\
&\qquad\qquad
+ d \sqrt{\frac{d\log(m)+\log(1/\delta)}{m}}
\sum_{k=0}^{n}\gamma_k
\Bigg] \\
&\le C \mu_R(\mathcal{X}_R) B_R^2 \|\Delta_0\|_{L_2(\mu_R)}
\Bigg[
\frac{d^{3/2} B_R}{\sqrt{m}}
C_\gamma \log (\rme(N+1)) \sum_{k=0}^{n} \gamma_k
\E_{\sgd} \left[\|\Delta_k\|_{L_2(\mu_R)}+\tau_c\right] \\
&\qquad\qquad
+ d \sqrt{\frac{d\log(m)+\log(1/\delta)}{m}}
C_\gamma \log (N+1)
\Bigg] \eqsp .
\end{align*}
Indeed, by the definition of $\eta_k$ and using the fact that, for all
$r\le k-1\le n$, one has
$\Omega_{n+1,R}\subset \Omega_{r+1,R}\subset \Omega_{r,R}$, we have
\begin{align*} 
\E_{\sgd}\left[ \left(\sum_{k=0}^{n}\gamma_k\eta_k\right) \mathbf 1_{\Omega_{n+1,R}} \right]
&\le \sum_{k=0}^{n}\gamma_k \sum_{r=0}^{k-1}\gamma_r \E_{\sgd}\left[ c(t^{(r)})\|\loss_r\| \mathbf 1_{\Omega_{n+1,R}} \right] \\
&\le \sum_{k=0}^{n}\gamma_k \sum_{r=0}^{k-1}\gamma_r \E_{\sgd}\left[ c(t^{(r)})\|\loss_r\| \mathbf 1_{\Omega_{r+1,R}} \right] \\
&\le \sqrt{\mu_R(\mathcal X_R)} \sum_{k=0}^{n}\gamma_k \sum_{r=0}^{k-1}\gamma_r ( \E_{\sgd} \left[\|\Delta_{r}\|_{L_2(\mu_R)} \un_{\Omega_{r,R}} \right] +\tau_c) \eqsp.
\end{align*}
Under the induction hypothesis $\E_{\sgd}\!\left[ \|\Delta_r\|_{L_2(\mu_R)}\mathbf 1_{\Omega_{r,R}}
\right] \le \widetilde S_N$,  therefore
\begin{align*}
\E_{\sgd}\!\left[
\left(\sum_{k=0}^{n}\gamma_k\eta_k\right)
\mathbf 1_{\Omega_{n+1,R}}
\right]
&\le
\sqrt{\mu_R(\mathcal X_R)}
(\widetilde S_N+\tau_c)
\left(\sum_{k=0}^{n}\gamma_k\right)^2 \\
&\le
\sqrt{\mu_R(\mathcal X_R)}
C_\gamma^2
\log^2(\rme(N+1))
(\widetilde S_N+\tau_c) \eqsp .
\end{align*}
By the width condition $m \ge c \mu_R^3(\mathcal{X}_R) B_R^6 d^3 C_\gamma^4 \log^4 (N+1)$ for a sufficiently large constant $c>0$, the first contribution can be absorbed into the induction bound. Therefore,
\begin{align*}
\E_{\sgd} \left[ \mathcal T_{2,n} \un_{\Omega_{n+1,R}} \right] \le C \mu_R(\mathcal{X}_R) B_R^2 d
C_\gamma \log (N+1) \sqrt{\frac{d\log(m)+\log(1/\delta)}{m}} \left(\|\Delta_0\|_{L_2(\mu_R)}+\tau_c\right) \eqsp.
\end{align*}
For the third term, Lemma~\ref{lem:controle_third_term} gives
\begin{align*}
\E_{\sgd} \left[ \mathcal{T}_{3,n} \un_{\Omega_{n+1,R}} \right]
&\le C\left( 1+\sqrt{\frac d m} +\sqrt{\frac{\log(1/\delta)}{m}} \right)^2 \sqrt{c_\infty \mu_R(\mathcal{X}_R)} B_R^2 C_{\gamma} e^{\Lambda_R C_{\gamma}} \left(1+\frac{1}{1-2\Lambda_R C_\gamma}\right)^{1/2} \sigma_0  \eqsp .
\end{align*}
For the fourth term, using the induction bound, for all $0\le k\le n$, we have
\[
\begin{aligned}
\E_{\sgd} [\eta_{k+1} \un_{\Omega_{n+1,R}}]
&\leq
\sqrt{\mu_R(\mathcal X_R)} C_\gamma \log(\rme(N+1)) (\widetilde S_N+\tau_c) \eqsp.
\end{aligned}
\]
Using Lemma \ref{lem:controle_epsilon_term}, we obtain
\begin{align*}
&\E_{\sgd}\left[\mathcal T_{4,n}\un_{\Omega_{n+1,R}}\right] \\
&\quad \le
C dB_R^2\sqrt{c_\infty\mu_R(\mathcal X_R)}
\left(
\frac{B_R\sqrt{d\mu_R(\mathcal X_R)}}{\sqrt m}
C_\gamma \log(\rme(N+1))(\widetilde S_N+\tau_c)
+
\sqrt{\frac{d+\log(1/\delta)}{m}}
\right)^{1/2}
\sum_{k=0}^{n}\gamma_k\sigma_k \eqsp.
\end{align*}
Moreover,
\[
\sum_{k=0}^{n}\gamma_k\sigma_k
\le
C_\gamma e^{\Lambda_R C_\gamma}\sigma_0
\frac{(N+1)^{\Lambda_R C_\gamma}}{\Lambda_R C_\gamma}
=
\frac{e^{\Lambda_R C_\gamma}}{\Lambda_R}
\sigma_0
(N+1)^{\Lambda_R C_\gamma} \eqsp.
\]
Using the definition of $\Lambda_R$, we have
\[
\frac{dB_R^2\sqrt{c_\infty\mu_R(\mathcal X_R)}}{\Lambda_R}
=
\frac{d}{d+\|A\|^2/m}
\le 1 \eqsp.
\]
Therefore,
\begin{align*}
&\E_{\sgd}\left[\mathcal T_{4,n}\un_{\Omega_{n+1,R}}\right] \\
&\quad \le
C e^{\Lambda_R C_\gamma}
\sigma_0
(N+1)^{\Lambda_R C_\gamma}
\left(
\frac{B_R\sqrt{d\mu_R(\mathcal X_R)}}{\sqrt m}
C_\gamma \log(\rme(N+1)) (\widetilde S_N+\tau_c)
+
\sqrt{\frac{d+\log(1/\delta)}{m}}
\right)^{1/2} \eqsp.
\end{align*}
Using $\sqrt{x+y}\le \sqrt{x}+\sqrt{y}$, we get
\[
\E_{\sgd}\left[\mathcal T_{4,n}\un_{\Omega_{n+1,R}}\right]
\le I_1+I_2 \eqsp,
\]
where
\begin{align*}
I_1 &= C e^{\Lambda_R C_\gamma}
\sigma_0
(N+1)^{\Lambda_R C_\gamma}
\left(
\frac{B_R\sqrt{d\mu_R(\mathcal X_R)}}{\sqrt m}
C_\gamma \log(\rme(N+1)) (\widetilde S_N+\tau_c)
\right)^{1/2} \eqsp, \\
I_2 &= C e^{\Lambda_R C_\gamma}
\sigma_0
(N+1)^{\Lambda_R C_\gamma}
\left(\frac{d+\log(1/\delta)}{m}\right)^{1/4} \eqsp.
\end{align*}
By definition of $E_{3,N}$, and by increasing the universal constant if necessary,
\[
E_{3,N}
\ge
c
\sqrt{c_\infty\mu_R(\mathcal X_R)}
B_R^2 C_\gamma
e^{\Lambda_R C_\gamma}
\sigma_0 \eqsp,
\]
for some universal constant $c>0$. 
Moreover, since
\[
D_0
=
\|\Delta_0\|_{L_2(\mu_R)}+\tau_c
\le \sqrt{2}\,\sigma_0 \eqsp,
\]
we have
\[
\widetilde S_N+\tau_c
\le
C\,\Theta_{R,N}\sigma_0
\eqsp,
\]
where
\[
\Theta_{R,N}
\eqdef
1
+\mu_R(\mathcal X_R)B_R^2 d C_\gamma\log(N+1)
+\sqrt{c_\infty\mu_R(\mathcal X_R)}
B_R^2C_\gamma\rme^{\Lambda_R C_\gamma}
\left(
1+\frac{1}{1-2\Lambda_R C_\gamma}
\right)^{1/2} \eqsp.
\]
Hence,
\[
I_1^2
\le
C\rme^{2\Lambda_RC_\gamma}
\sigma_0^3
(N+1)^{2\Lambda_RC_\gamma}
\frac{B_R\sqrt{d\mu_R(\mathcal X_R)}}{\sqrt m}
C_\gamma\log(\rme(N+1))
\Theta_{R,N} \eqsp.
\]
Therefore, under the width condition
\[
m
\ge
c
\frac{
(N+1)^{4\Lambda_RC_\gamma}
d\log^2(\rme(N+1))
\Theta_{R,N}^2
}{
c_\infty^2\mu_R(\mathcal X_R)
B_R^6C_\gamma^2
} \eqsp,
\]
where $c>0$ may depend on $\sigma_0$, we obtain
\[
I_1\le \frac{E_{3,N}}{2}\eqsp.
\]
% Moreover, the imposed width condition
% \[
% m
% \ge
% c
% \frac{
% e^{2\Lambda_R C_\gamma}
% \sigma_0^2
% (N+1)^{4\Lambda_R C_\gamma}
% d\log^2(\rme(N+1))
% }{
% c_\infty B_R^2
% }
% \]
% implies that
% \[
% C e^{2\Lambda_R C_\gamma}
% (N+1)^{2\Lambda_R C_\gamma}
% \frac{B_R\sqrt{d\mu_R(\mathcal X_R)}}{\sqrt m}
% C_\gamma \log(\rme(N+1))
% \le
% c
% \sqrt{c_\infty\mu_R(\mathcal X_R)}
% B_R^2 C_\gamma
% e^{\Lambda_R C_\gamma}
% \sigma_0
% \le E_{3,N} \eqsp.
% \]
% Therefore,
% \begin{align*}
% I_1^2
% &=
% C e^{2\Lambda_R C_\gamma}
% (N+1)^{2\Lambda_R C_\gamma}
% \frac{B_R\sqrt{d\mu_R(\mathcal X_R)}}{\sqrt m}
% C_\gamma \log(\rme(N+1))
% \widetilde S_N \le E_{3,N}^2/4 \eqsp.
% \end{align*}
% Consequently,
% \[
% I_1\le \widetilde E_{3,N}/2 \eqsp.
% \]
Similarly, other width condition gives
\[
m
\ge
c
\frac{
(N+1)^{4\Lambda_R C_\gamma}
(d+\log(1/\delta))
}{
c_\infty^2\mu_R(\mathcal X_R)^2B_R^8C_\gamma^4
} \eqsp.
\]
Equivalently,
\[
I_2 = C e^{\Lambda_R C_\gamma}
\sigma_0
(N+1)^{\Lambda_R C_\gamma}
\left(\frac{d+\log(1/\delta)}{m}\right)^{1/4}
\le
c \sqrt{c_\infty\mu_R(\mathcal X_R)}
B_R^2 C_\gamma
e^{\Lambda_R C_\gamma}
\sigma_0
\le \frac{E_{3,N}}{2} \eqsp.
\]
Combining the bounds on $I_1$ and $I_2$, we obtain
\[
\E_{\sgd}\left[\mathcal T_{4,n}\un_{\Omega_{n+1,R}}\right]
\le E_{3,N} \eqsp.
\]
This completes the fourth term.
Thus, under the width condition
\begin{multline}
\label{eq:condition_width}
m \ge c \max\Bigg\{
\mu_R^3(\mathcal{X}_R) B_R^6 d^3 C_\gamma^4 \log^4 (N+1),\;
\frac{(N+1)^{4\Lambda_R C_\gamma}(d+\log(1/\delta))} {c_\infty^2\mu_R(\mathcal X_R)^2B_R^8C_\gamma^4} , \\
\frac{
(N+1)^{4\Lambda_RC_\gamma}
d\log^2(\rme(N+1))
\Theta_{R,N}^2
}{
c_\infty^2\mu_R(\mathcal X_R)
B_R^6C_\gamma^2
}
\Bigg\} \eqsp,
\end{multline}
combining the previous bounds yields
\begin{align*}
\E_{\rm sgd}\left[
\|\Delta_{n+1}\|_{L_2(\mu_R)}
\mathbf 1_{\Omega_{n+1,R}}
\right]
&\le
E_{1,n+1}+E_{2,N}+2E_{3,N}\\
&\le
2\left(E_{1,n+1}+E_{2,N}+E_{3,N}\right)\\
&=S_{n+1} \eqsp,
\end{align*}
where the second inequality follows from the nonnegativity of
$E_{1,n+1}$ and $E_{2,N}$. This closes the induction.
%the induction bound holds at iteration $n+1$.
For readability, this condition can be summarized by introducing 
\begin{align} \label{eq:def_C_R} 
\mathcal C_R \eqdef \max\left\{ \mu_R^3(\mathcal X_R)B_R^6C_\gamma^4, \frac{1}{c_\infty^2\mu_R^2(\mathcal X_R)B_R^8C_\gamma^4}, \frac{\Theta_{R,N}^2}{c^2_\infty\mu_R(\mathcal X_R)B_R^6C_\gamma^2} \right\} \eqsp. 
\end{align} Indeed, a sufficient condition for \eqref{eq:condition_width} is 
\begin{align} \label{eq:condition_width_sim} 
m \gtrsim \mathcal C_R (N+1)^{4\Lambda_R C_\gamma} d^3 \bigl(d+\log(1/\delta)\bigr) \log^4(\rme(N+1)) =: \mathfrak m_R \eqsp. 
\end{align}
This concludes the proof of Theorem~\ref{th:rate_two_layer_SGD}.

\section{Additional Derivations for the NTK Analysis}
\label{app:ntk_derivations}

Recall that,
$$
\loss_k
\eqdef
\net \left(\ora \X^{(k)};\param^{(k)}\right)-X_0^{(k)}
=
\net \left(\ora \X^{(k)};\param^{(k)}\right)
-
\left(f^\star \left(\ora \X^{(k)}\right)+\xi_k\right) \eqsp.
$$

\subsection{Details on the Neural Tangent Kernel}
\label{app:ntk_kernel_details}

For every $\x\in\R^{d+1}$, we define the empirical matrix-valued NTK, 
$
K_k:\R^{d+1}\times\R^{d+1}\to\R^{d\times d}
$
by
\begin{align}
K_k(\x,\x')
&\eqdef
\sum_{i=1}^m
\nabla_{\param_i}\net(\x;\param[k])\,
\nabla_{\param_i}\net(\x';\param[k])^\top
\notag\\
&=
\frac{\x^\top\x'}{m}
\sum_{i=1}^m
\mathbf 1_{\{(\param[k]_i)^\top \x \ge 0\}}
\mathbf 1_{\{(\param[k]_i)^\top \x' \ge 0\}}
\,A_{\cdot i}A_{\cdot i}^\top
\notag\\
&=
\frac{\x^\top\x'}{m}\,
A\,D_{\x,\param[k]}D_{\x',\param[k]}A^\top .
\label{eq:empirical_matrix_ntk}
\end{align}

\begin{remark}
The NTK matrix appears in the first order linearization in the parameter space given by 
% \begin{align*}
%     \net \left(\x; \param[k+1] \right)
%     &= \net \left(\x; \param[k] + \left( \param[k+1] - \param[k] \right)\right) \\
%     &\approx \net\left(\x; \param[k] \right)
%     + \sum_{i=1}^m \nabla_{\param_i} \net(\x, \param[k])
%     \left( \param[k+1]_i - \param[k]_i \right) \\
%     & = \net \left(\x; \param[k] \right)
%     + \frac{1}{\sqrt{m}} \sum_{i=1}^m  A_{\cdot,i} \mathbf{1}_{ \{ (\param_i^{(k)})\top \x \geq 0 \}} \x^\top 
%     \left( \param[k+1]_i - \param[k]_i \right) \\
%     & = \net \left(\x; \param[k] \right)
%     -\gamma_k\,c(t^{(k)}) \frac{1}{m} \sum_{i=1}^m  A_{\cdot,i} \mathbf{1}_{ \{ \param_i^{(k)\top} \x \geq 0 \}} \x^\top 
%     \left( \mathbf 1_{\{ \param_i^{(k)\top} \ora \X^{(k)}\ge 0\}}
%  A_{\cdot,i}^\top \loss_k 
% \ora \X^{(k)}  \right) \\
% &= \net \left(\x; \param[k] \right)
%     -\gamma_k c(t^{(k)}) K_k(\x,\ora \X^{(k)}) \loss_k  \\
% \end{align*}
\begin{align*}
\net \left(\x; \param[k+1] \right)
&= \net \left(\x; \param[k] + \left( \param[k+1] - \param[k] \right)\right) \\
&\approx \net\left(\x; \param[k] \right)
+ \sum_{i=1}^m \nabla_{\param_i} \net(\x, \param[k])
\left( \param[k+1]_i - \param[k]_i \right) \\
&= \net \left(\x; \param[k] \right)
+ \frac{1}{\sqrt{m}} \sum_{i=1}^m  A_{\cdot,i}
\mathbf{1}_{ \{ \param_i^{(k)\top} \x \geq 0 \}}
\x^\top 
\left( \param[k+1]_i - \param[k]_i \right) \\
&= \net \left(\x; \param[k] \right)
-\gamma_k\,c(t^{(k)}) \frac{1}{m} \sum_{i=1}^m  A_{\cdot,i}
\mathbf{1}_{ \{ \param_i^{(k)\top} \x \geq 0 \}} \\
&\qquad \times \x^\top 
\left(
\mathbf{1}_{\{ \param_i^{(k)\top} \ora \X^{(k)}\ge 0\}}
A_{\cdot,i}^\top \loss_k 
\ora \X^{(k)}
\right) \\
&= \net \left(\x; \param[k] \right)
-\gamma_k c(t^{(k)}) K_k(\x,\ora \X^{(k)}) \loss_k \eqsp.
\end{align*}
where we used \eqref{eq:SGD_update} in the fourth line and \eqref{eq:empirical_matrix_ntk} in the fifth line. 
\end{remark}
It follows that,
\begin{equation*}
\net(\x;\param[k+1])
=
\net(\x;\param[k])
-
\gamma_k c(t^{(k)}) K_k(\x,\ora \X^{(k)}) \loss_k
-
\varepsilon_k(u, \ora \X^{(k)}) \eqsp,
\label{eq:one_step_linearization}
\end{equation*}
where
\begin{equation*} 
\varepsilon_k(u, \ora \X^{(k)}) \eqdef \net(\x;\param[k]) -\net(\x;\param[k+1]) - \gamma_k c(t^{(k)}) K_k(\x,\ora \X^{(k)}) \loss_k \label{eq:def_linearization_error} 
\end{equation*}
is a nonlinear remainder. It measures the error made when replacing the exact one-step update of the network output by its first-order NTK approximation.

At initialization, the empirical matrix-valued NTK admits a deterministic infinite-width limit. More precisely, assume that
$$
\{\param_i^{(0)}\}_{i=1}^m \overset{\mathrm{i.i.d.}}{\sim} \mathcal N(0,\Id_{d+1}),
$$
independently of the output matrix $A$, whose entries are i.i.d. Rademacher. Then, for every fixed $\x,\x'\in\R^{d+1}$,
$$
K_0(\x,\x')
=
\frac{\x^\top\x'}{m}
\sum_{i=1}^m
\mathbf 1_{\{(\param_i^{(0)})^\top \x \ge 0\}}
\mathbf 1_{\{(\param_i^{(0)})^\top \x' \ge 0\}}
\,A_{\cdot i}A_{\cdot i}^\top
\xrightarrow[m\to\infty]{a.s.}
K_\infty(\x,\x') \eqsp.
$$
with 
$$
K_\infty(\x,\x')
=
(\x^\top \x')\,
\E_{w \sim \mathcal{N}(0,\Id_{d+1})}
\Big[
\mathbf{1}\{w^\top \x \ge 0\}\,
\mathbf{1}\{w^\top \x' \ge 0\}
\Big]\,
\Id_d
$$
since $\E[A_{\cdot i}A_{\cdot i}^\top]=\Id_d$.

\subsection{SGD Update Derivation} 
\label{proof:sgd_udpates}

% For convenience, we first record the row-wise form of
% \eqref{sgd_updates}: for each $i\in\{1,\dots,m\}$, the $i$-th row of the matrix weight update (viewed as a column vector in $\R^{d+1}$) satisfies
% $$
% \param[k+1]_i = \param[k]_i
% -\gamma_k\,c(t^{(k)})\,
% \frac{1}{\sqrt m}\,
% \mathbf 1_{\{ \param_i^{(k)\top}  \ora \X^{(k)}\ge 0\}}
%  A_{\cdot,i}^\top \left(\net \left(\ora \X^{(k)};\param^{(k)}\right)-X_0^{(k)}\right)
% \ora \X^{(k)} \eqsp.
% $$

\begin{lemma}\label{lem:sgd_updates}
The SGD update \eqref{eq:SGD_update} writes as
\begin{equation} \label{sgd_updates_bis}
\param[k+1] = \param[k] - \gamma_k\,c(t^{(k)})\,
\frac{1}{\sqrt m}\,
D_k A^\top \left(\net \left(\ora \X^{(k)};\param^{(k)}\right)-X_0^{(k)}\right) \,(\ora \X^{(k)})^\top \eqsp,
\end{equation}
where $D_k \eqdef D_{\ora \X^{(k)},\param[k]}
=
\operatorname{diag} \left(
\mathbf 1_{\{ \param_1^{(k)\top}  \ora \X^{(k)}\ge 0\}},
\dots,
\mathbf 1_{\{ \param_m^{(k)\top}  \ora \X^{(k)}\ge 0\}}
\right)$. 
Similarly, for each $i\in\{1,\dots,m\}$, the $i$-th row (viewed as a column vector in $\R^{d+1}$) satisfies
\[
\param[k+1]_i = \param[k]_i
-\gamma_k\,c(t^{(k)})\,
\frac{1}{\sqrt m}\,
\mathbf 1_{\{ \param_i^{(k)\top}  \ora \X^{(k)}\ge 0\}}
 A_{\cdot,i}^\top \left(\net \left(\ora \X^{(k)};\param^{(k)}\right)-X_0^{(k)}\right)
\ora \X^{(k)} \eqsp.
\]
\end{lemma}

For simplicity, we introduce $\loss_k \in \R^d$ defined by
\begin{align*}
\loss_k \eqdef \net \left(\ora \X^{(k)};\param^{(k)}\right)-X_0^{(k)}
= \net \left(\ora \X^{(k)};\param^{(k)}\right) - \left(f^\star \left(\ora \X^{(k)}\right)+\xi_k\right) \eqsp.
\end{align*}
where
\begin{align*}
\ora \X^{(k)} \eqdef \bigl(\ora X_{t^{(k)}}^{(k)},t^{(k)}\bigr),
\qquad \xi_k
\eqdef X_0^{(k)}-f^\star\!\bigl(\ora \X^{(k)}\bigr) \eqsp.
\end{align*}

\begin{proof}
For every $x\in\R^{d+1}$ and every $i,j\in\{1,\dots,m\}$ such that $\param_i^\top x\neq 0$, we have
$$
\nabla_{\param_j}\sigma(\param_i^\top x)
=
\mathbf 1_{\{i=j\}}\,
\mathbf 1_{\{\param_j^\top x\ge 0\}} x \eqsp.
$$
Hence
$$
\nabla_{\param_i}\net(x;\param)
=
\frac{1}{\sqrt m}
A_{\cdot,i}
\mathbf 1_{\{\param_i^\top x\ge 0\}}
x^\top
\in \R^{d\times(d+1)} \eqsp.
$$
Therefore,
\begin{align*}
\nabla_{\param_i}
\left(
\frac12\|\net(\ora \X^{(k)};\param[k])-X_0^{(k)}\|^2
\right)
&=
\left(\nabla_{\param_i}\net(\ora \X^{(k)};\param[k])\right)^\top
\bigl(\net(\ora \X^{(k)};\param[k])-X_0^{(k)}\bigr) \\
&=
\frac{1}{\sqrt m} 
\mathbf 1_{\{ \param_i^{(k)\top} \ora \X^{(k)}\ge 0\}} 
A_{\cdot,i}^\top \loss_k 
\ora \X^{(k)} \eqsp.
\end{align*}
Hence,
$$
\nabla_{\param}
\left(
\frac12\|\net(\ora \X^{(k)};\param[k])-X_0^{(k)}\|^2
\right)
=
\frac{1}{\sqrt m}\,
D_k A^\top \loss_k\,(\ora \X^{(k)})^\top \eqsp.
$$
\end{proof}

\subsection{Technical Lemmas}

\begin{lemma}[Gaussian tail bound]
\label{lem:tail_bound_X}
Let
$$
\mu_Z \eqdef \E\|Z\|
=
\sqrt{2}\,\frac{\Gamma((d+1)/2)}{\Gamma(d/2)},
\qquad Z\sim\mathcal N(0,\Id_d) \eqsp,
$$
% \textbf{If \mathbb $X_0$ is subgaussian is bounded almost surely}, $\mathbb P(\|X_0\|\le D)=1$, then, for every $t\in(0,T]$ and every $R>m_tD+\sigma_t\mu_Z$,
% \begin{equation} \label{eq:bound_tail_t}
% \mathbb P(\|\ora X_t\|>R)
% \le
% \exp \left(
% -\frac{(R-m_t D-\sigma_t\mu_Z)^2}{2\sigma_t^2}
% \right) \eqsp,
% \end{equation}
and suppose that Assumption \ref{ass:subgaussian_data} holds. Then  for any $R>\sigma_t\mu_Z$,
$$
\mathbb{P} \left( \|\ora X_t\|>R \right) \le
(1+C_0)
\exp \left(
-\frac{(R-\sigma_t\mu_Z)^2}
{8\max(\sigma_t^2,m_t^2\nu_0^2)}
\right) \eqsp.
$$
% and also
% \begin{equation} \label{eq:bound_tail_tau}
% \mathbb P(\|\ora X_\tau\|>R) \leq \int_0^T q(t) \exp \left(
% -\frac{(R-m_t D-\sigma_t\mu_Z)^2}{2\sigma_t^2}
% \right) \rmd t
% \end{equation}
\end{lemma}

\begin{proof}
Conditionally on $X_0=x$, 
$$
\ora X_t = m_t x + \sigma_t Z,
\qquad Z\sim\mathcal N(0,\Id_d) \eqsp.
$$
Define the map
$$
F_x(z)\eqdef \|m_t x+\sigma_t z\| 
$$
and note that, by the reverse triangle inequality, $F_x$ is $\sigma_t$-Lipschitz. By the Gaussian concentration inequality for Lipschitz functions
(see Theorem B.7, Appendix B.2.2 in \cite{giraud2021highdim}),
for every $r>0$,
$$
\mathbb P \left(F_x(Z)\ge \E[F_x(Z)] + r\right)
\le
\exp \left(-\frac{r^2}{2\sigma_t^2}\right),
$$
Moreover,
$$
\E[F_x(Z)]
\le
\|m_t x\|+\sigma_t \E\|Z\|
\le
m_t \| x \|+\sigma_t\mu_Z \eqsp.
$$
Therefore, for every $R>0$,
$$
\mathbb P(\|\ora X_t\|>R\mid X_0=x)
\le
\exp \left(
-\frac{(R-m_t \|x\| -\sigma_t\mu_Z)_+^2}{2\sigma_t^2}
\right) \eqsp,
$$
Integrating with respect to the distribution of $X_0$ yields 
$$
\mathbb P(\|\ora X_t\|>R)
\le
\E \left[
\exp \left(
-\frac{(R-m_t\|X_0\|-\sigma_t\mu_Z)_+^2}{2\sigma_t^2}
\right)
\right] \eqsp.
$$
Hence, on the event $\{ \| X_0 \| \leq D  \}$, one has for every $t\in(0,T]$ and every $R>m_tD+\sigma_t\mu_Z$,
\begin{equation} \label{eq:bound_tail_t}
\mathbb P(\|\ora X_t\|>R)
\le
\exp \left(
-\frac{(R-m_t D-\sigma_t\mu_Z)^2}{2\sigma_t^2}
\right) \eqsp.
\end{equation}
Then, note that for any $D > 0$
\begin{align*}
\mathbb P(\|\ora X_t\|>R)
&\le
\mathbb P(\|\ora X_t\|>R, \|X_0\|\le D)
+
\mathbb P(\|X_0\|>D) \\
&\le
\exp \left(
-\frac{(R-m_t D-\sigma_t\mu_Z)_+^2}{2\sigma_t^2}
\right)
+
C_0\exp \left(-\frac{D^2}{2\nu_0^2}\right) \eqsp.
\end{align*}
In particular, choosing
$$
D=\frac{R- \sigma_t \mu_Z}{2 m_t},
$$
which is positive when $R>\sigma_t\mu_Z$, yields
\begin{align*}
\mathbb P(\|\ora X_t\|>R)
& \le
\exp\!\left(
-\frac{(R-\sigma_t\mu_Z)^2}{8\sigma_t^2}
\right)
+
C_0\exp \left(
-\frac{ (R-\sigma_t\mu_Z)^2}{ 8 m_t^2\nu_0^2}
\right) \\
&
\le
(1+C_0)
\exp\!\left(
-\frac{(R-\sigma_t\mu_Z)^2}
{8\max(\sigma_t^2,m_t^2\nu_0^2)}
\right) \eqsp.
\end{align*}

% For the second claim note that
% \begin{align*}
% \mathbb P(\|\ora X_\tau\|>R)
% =
% \mathbb E \left[\mathbf 1_{\{\|\ora X_\tau\|>R\}}\right] 
% =
% \mathbb E \left[ \mathbb E \left[\mathbf 1_{\{\|\ora X_\tau\|>R\}} \middle| \tau \right]
% \right] \eqsp.
% \end{align*}
% Since, conditionally on $\{\tau=t\}$, the random variable $\ora X_\tau$ has the
% same law as $\ora X_t$, we have
% $$
% \mathbb E\!\left[\mathbf 1_{\{\|\ora X_\tau\|>R\}} \middle| \tau=t\right]
% =
% \mathbb P(\|\ora X_t\|>R).
% $$
% Therefore,
% $$
% \mathbb P(\|\ora X_\tau\|>R)
% =
% \int_0^T q(t) \mathbb P(\|\ora X_t\|>R) \rmd t \eqsp,
% $$
% so that \eqref{eq:bound_tail_tau} follows immediately from the previous pointwise bound.
\end{proof}

\begin{lemma}[VC dimension bound]
\label{lem:VC_dim_bound}
There exists $C>0$, such that, with probability at least $1-\delta$,
$$
\sup_{u\in\mathbb S^d, t \geq 0}
\left|
\frac1m\sum_{i=1}^m \mathbf 1_{\{|\param_i^{(0)\top}u|\le t\}}
-
\mathbb P\bigl(|\param_1^{(0)\top}u|\le t \bigr)
\right|
\le
C \sqrt{\frac{d+\log(1/\delta)}{m}}.
$$
\end{lemma}

\begin{proof}
Define,
$$
\mathcal F
:=
\left\{
w\mapsto \mathbf 1_{\{|w^\top u|\le t\}}:\ u \in \mathbb{S}^{d}, t \ge 0
\right\} \eqsp,
$$
where $\mathbb{S}^{d}$ denotes the unit sphere of $\R^{d+1}$. Following Example 8.3.12 from \citet{vershynin2026hdp}
$$
\mathrm{VC}(\mathcal{F}) \leq 40(d+1) \eqsp.
$$
Using the VC law of large numbers, \citep[Theorem~8.3.15]{vershynin2026hdp}, there exists a universal constant $C> 0$, such that
$$
\E \left[ \sup_{u\in\mathbb S^d, t \geq 0}
\left|
\frac1m\sum_{i=1}^m \mathbf 1_{\{|\param_i^{(0)\top}u|\le t \}}
-
\mathbb P\bigl(|\param_1^{(0)\top}u|\le t \bigr) \right| \right]\leq C \sqrt{ \frac{\mathrm{VC}(\mathcal{F})}{m}} \eqsp.
$$
Then using the bounded difference inequality \citep[Theorem~5.7.1]{vershynin2026hdp}, there exists $C>0$ such that with probability at least $1-\delta$,
$$
\sup_{u\in\mathbb S^d, t \geq 0}
\left|
\frac1m\sum_{i=1}^m \mathbf 1_{\{|\param_i^{(0)\top}u|\le t \}}
-
\mathbb P\bigl(|\param_1^{(0)\top}u|\le t \bigr) \right| \leq C \sqrt{\frac{d+\log(1/\delta)}{m}} \eqsp.
$$
\end{proof}

Let
\begin{align*}
I_i^0(\x)=\mathbf{1}\{\param_i^{(0)\top}\x\ge0\},
\qquad
p(\x,\x')=
\E_{\param\sim\mathcal N(0,\Id_{d+1})}
\left[
\mathbf{1}\{\param^\top\x\ge0\}
\mathbf{1}\{\param^\top\x'\ge0\}
\right] \eqsp.
\end{align*}

\begin{lemma} \label{lem:cov_proba_bound}
There exists $C>0$ such that, with probability at least $1-\delta$,
\begin{align*}
\sup_{\x,\x'\in\mathcal{X}_R}
\left\|
\frac1m\sum_{i=1}^m
I_i^0(\x) I_i^0(\x') A_{\cdot i} A_{\cdot i}^\top
- p(\x,\x') \Id_d
\right\|
&\le
C\left[
\sqrt{\frac{d\log(\rme m)+\log(1/\delta)}{m}}
+
\frac{d\log(\rme m)+\log(1/\delta)}{m}
\right] \eqsp .
\end{align*}
Moreover, if $m\ge d\log(\rme m)+\log(1/\delta)$, then there exists $C'>0$ such that
\[
\sup_{\x,\x'\in\mathcal{X}_R}
\left\|
\frac1m\sum_{i=1}^m
I_i^0(\x) I_i^0(\x') A_{\cdot i} A_{\cdot i}^\top
- p(\x,\x') \Id_d
\right\|
\le
C'\sqrt{\frac{d\log(m)+\log(1/\delta)}{m}} \eqsp.
\]
\end{lemma}

% \begin{lemma} \label{lem:cov_proba_bound}
% There exists $C>0$ such that with probability $1-\delta$
% {\small
% \begin{align*}
% & \sup_{\x,\x' \in \mathcal{X}_R} \Biggl\|
% \frac{1}{m}\sum_{i=1}^m
% \mathbf{1}\{\param_i^{(0)\top} \x \ge 0\}
% \mathbf{1}\{\param_i^{(0)\top} \x' \ge 0\}
% A_{\cdot i}A_{\cdot i}^\top
% - \mathbb{E}_{\param \sim\mathcal N(0,\Id_{d+1})}
% \Big[ \mathbf{1}\{\param^{\top} \x \ge 0\}
% \mathbf{1}\{\param^{\top} \x' \ge 0\} \Big] \Id_d \Biggr\| \\
% & \qquad \qquad
% \leq C\left[
% \sqrt{\frac{d\log(\rme m)+\log(1/\delta)}{m}}
% +
% \frac{d\log(\rme m)+\log(1/\delta)}{m}
% \right] \eqsp.
% \end{align*}
% }
% If $m\ge d\log(em)+\log(1/\delta)$, then there exists $C'>0$ such that 
% {\small
% \begin{align*}
% & \sup_{\x,\x' \in \mathcal{X}_R} \Biggl\|
% \frac{1}{m}\sum_{i=1}^m
% \mathbf{1}\{\param_i^{(0)\top} \x \ge 0\}
% \mathbf{1}\{\param_i^{(0)\top} \x' \ge 0\}
% A_{\cdot i}A_{\cdot i}^\top
% - \mathbb{E}_{\param \sim\mathcal N(0,\Id_{d+1})}
% \Big[ \mathbf{1}\{\param^{\top} \x \ge 0\}
% \mathbf{1}\{\param^{\top} \x' \ge 0\} \Big] \Id_d \Biggr\| \\
% & \le
% C'\sqrt{\frac{d\log(m)+\log(1/\delta)}{m}} \eqsp.
% \end{align*}
% }
% \end{lemma}

\begin{proof}
For $\x,\x'\in\mathcal{X}_R$, define 
$$
h_{\x,\x'}(w)
\eqdef
\mathbf 1_{\{w^\top\x\ge0\}}
\mathbf 1_{\{w^\top\x'\ge0\}} \quad \text{ and } \qquad
\mathcal H
\eqdef
\{h_{\x,\x'}:\x,\x'\in\mathcal{X}_R\} \eqsp.
$$
The class $\mathcal H$ consists of intersections of halfspaces in $\R^{d+1}$, hence its VC dimension satisfies \citep[Example~8.3.5]{vershynin2026hdp},
$$
\mathrm{VC}(\mathcal H)\le C_{\mathcal{H}} (d+1) 
$$
for a universal constant $C>0$. For $h\in\mathcal H$, set  $w\sim\mathcal N(0,\Id_{d+1})$ so that
\begin{align*}
\frac1m\sum_{i=1}^m h(\param_i^{(0)})A_{\cdot i}A_{\cdot i}^\top
- \E[h(w)] \Id_d
&= \frac1m\sum_{i=1}^m h(\param_i^{(0)})
\left(A_{\cdot i}A_{\cdot i}^\top-\Id_d\right) + \left(
\frac1m\sum_{i=1}^m h(\param_i^{(0)})- \E[h(w)]
\right)\Id_d \eqsp .
\end{align*}
\paragraph{Step 1: control of the first term.} Note that for each $h\in\mathcal H$, conditional on $\param^{(0)}$,
$$
\mathbf b(h)
\eqdef \left( h(\param_1^{(0)} ), \dots, h(\param_m^{(0)}) \right)
\in \{0,1\}^m \eqsp.
$$
Although $\mathcal H$ is infinite, the number of distinct binary vectors that can be realized on the fixed sample $\{\param_i^{(0)}\}_{i=1}^m$ is finite. Since
$\mathrm{VC}(\mathcal H) \le C_{\mathcal{H}} (d+1)$, the Sauer--Shelah lemma \citep[Lemma~8.3.9]{vershynin2026hdp} gives
$$
N_m = \left|
\left\{
\big(h(\param_1^{(0)}),\dots,h(\param_m^{(0)})\big):
h\in\mathcal H
\right\}
\right|
\le
\left(\frac{\rme m}{ C_{\mathcal{H}}(d+1)} \right)^{ C_{\mathcal{H}}(d+1)} \eqsp.
$$
Let the distinct patterns be indexed by
$$
\mathbf b^{(1)},\dots,\mathbf b^{(N_m)}\in\{0,1\}^m.
$$
For each $j\in\{1,\dots,N_m\}$, define the corresponding active set by
$$
S_j
\eqdef
\{i\in[m]: b_i^{(j)}=1\}.
$$
Then, conditionally on $\param^{(0)}$,
$$
\sup_{h\in\mathcal H}
\left\|
\frac1m\sum_{i=1}^m h(\param_i^{(0)})
\left(A_{\cdot i}A_{\cdot i}^\top-\Id_d\right)
\right\|
=
\max_{1\le j\le N_m}
\left\|
\frac1m\sum_{i\in S_j}
\left(A_{\cdot i}A_{\cdot i}^\top-\Id_d\right)
\right\| \eqsp.
$$
Fix one pattern $j$ and note that
$$
\frac1m\sum_{i\in S_j}
\left(A_{\cdot i}A_{\cdot i}^\top-\Id_d\right)
=
\frac{|S_j|}{m}
\left[
\frac1{|S_j|}\sum_{i\in S_j}A_{\cdot i}A_{\cdot i}^\top-\Id_d
\right] \eqsp.
$$
Following \citep[Example 3.4.3]{vershynin2026hdp}, $A_{\cdot i}$ is sub-Gaussian and in particular, there exists $K_A>0$, such that
$$
\|\langle A_{\cdot i},z\rangle\|_{\psi_2}\le K_A \|z\|_2 \eqsp, \qquad z \in \R^d \eqsp,
$$
where $\| \cdot \|_{\psi_2}$ is the subgaussian norm defined in \citep[Example 2.6.4]{vershynin2026hdp}. Therefore, $A_{\cdot i}$ satisfies condition (4.29) of Theorem 4.7.1. from \citet{vershynin2026hdp} with (using their notation). It follows from Remark 4.7.3 \citep{vershynin2026hdp} that, for all $s \ge 0$ and fixed pattern $j$,
$$
\mathbb P \left( \left\| \frac1{|S_j|}\sum_{i\in S_j} A_{\cdot i} A_{\cdot i}^\top - \Id_d \right\|
\geq
C\left[  \sqrt{\frac{d+s}{|S_j|}}+ \frac{d+s}{|S_j|} \right] \middle| \param[0] \right) \le 2 \rme^{-s} \eqsp.
$$
Multiplying by $|S_j|/m$ and using $|S|\le m$
$$
\mathbb P \left( \left\|
\frac1m\sum_{i\in S_j}
\left( A_{\cdot i} A_{\cdot i}^\top - \Id_d \right)
\right\| \ge C \left[ \sqrt{\frac{d+s}{m}} + \frac{d+s}{m} \right] \middle| \param[0] \right) \le 2 \rme^{-s} \eqsp
$$
Therefore, 
\begin{align*}
& \mathbb P \left( \max_{ 1 \le j \le N_m} \left\|
\frac1m\sum_{i\in S_j}
\left( A_{\cdot i} A_{\cdot i}^\top - \Id_d \right)
\right\| \ge C \left[ \sqrt{\frac{d+s}{m}} + \frac{d+s}{m} \right] \middle| \param[0] \right) \\
& = \mathbb P \left( \bigcup_{j=1}^{N_m} \left\{ \left\|
\frac1m\sum_{i\in S_j}
\left( A_{\cdot i} A_{\cdot i}^\top - \Id_d \right)
\right\| \ge C \left[ \sqrt{\frac{d+s}{m}} + \frac{d+s}{m} \right] \right\} \middle| \param[0] \right) \\
& \le 2 N_m  \rme^{-s} \eqsp
\end{align*}
Setting
$ s = \log \left( 4N_m/\delta \right) $,
makes
$ 2N_m \rme^{-u}\le \delta/2  $.
Since $N_m\le (\rme m)^{C_{\mathcal H}(d+1)}$, we have
$$
u
\le
C_{\mathcal H}(d+1)\log(\rme m)+\log(4/\delta) \eqsp.
$$
Therefore, conditionally on $\param[0]$, there exists $C>0$ such that with probability at least $1-\delta/2$,
$$
\sup_{h\in\mathcal H}
\left\|
\frac1m\sum_{i=1}^m h(\param_i^{(0)})
\left(A_{\cdot i}A_{\cdot i}^\top-\Id_d\right)
\right\|
\le
C\left[
\sqrt{\frac{d\log(\rme m)+\log(1/\delta)}{m}}
+
\frac{d\log(\rme m)+\log(1/\delta)}{m}
\right] \eqsp.
$$
Since the above bound holds conditionally on $\param[0]$, the unconditional claim follows by the tower property.

\paragraph{Step 2: control of the second term.}

Since $\mathrm{VC}(\mathcal H)\le C_{\mathcal H}(d+1)$, the VC uniform law
of large numbers \citep[Theorem~8.3.15]{vershynin2026hdp}, together with
McDiarmid's inequality \citep[Theorem~5.7.1]{vershynin2026hdp}, gives, with
probability at least $1-\delta/2$ over $\param^{(0)}$,
$$
\sup_{h\in\mathcal H}
\left|
\frac1m\sum_{i=1}^m h(\param_i^{(0)})- \E[h(w)]
\right|
\le
C\sqrt{\frac{d+\log(1/\delta)}{m}} \eqsp.
$$

\paragraph{Step 3: combining both upper bounds.}
By Step 1, with probability at least $1-\delta/2$,
$$
\sup_{h\in\mathcal H}
\left\|
\frac1m\sum_{i=1}^m h(\param_i^{(0)})
\left( A_{\cdot i} A_{\cdot i}^\top- \Id_d \right)
\right\|
\le
C
\left[
\sqrt{ \frac{d\log(em)+ \log(1/\delta)}{m} }
+
\frac{ d \log(em) + \log(1/ \delta)}{m}
\right] \eqsp.
$$
By Step 2, with probability at least $1-\delta/2$,
$$
\sup_{h\in\mathcal H}
\left|
\frac1m\sum_{i=1}^m h(\param_i^{(0)})-\E[h(w)]
\right|
\le
C\sqrt{\frac{d+\log(1/\delta)}{m}} \eqsp.
$$
Therefore, both estimates hold simultaneously with
probability at least $1-\delta$. On this event, for every $h\in\mathcal H$,
\begin{align*}
&\left\|
\frac{1}{m} \sum_{i=1}^m h(\param_i^{(0)}) A_{\cdot i} A_{\cdot i}^\top
-
\E[h(w)] \Id_d
\right\| \\
&\le
\left\|
\frac1m\sum_{i=1}^m h(\param_i^{(0)})
\left(A_{\cdot i}A_{\cdot i}^\top-\Id_d\right)
\right\|
+
\left|
\frac{1}{m} \sum_{i=1}^m h( \param_i^{(0)}) - \E[h(w)]
\right| \eqsp.
\end{align*}
Taking the supremum over $h\in\mathcal H$ and absorbing constants gives
\begin{align*}
\sup_{h\in\mathcal H}
\left\|
\frac{1}{m} \sum_{i=1}^m h(\param_i^{(0)}) A_{\cdot i} A_{\cdot i}^\top
-
\E[h(w)] \Id_d
\right\|
&\le
C
\left[
\sqrt{ \frac{ d \log(em)+ \log(1/\delta)}{m}}
+
\frac{ d \log(em) + \log(1/\delta) }{m}
\right] \eqsp .
\end{align*}
\end{proof}

\begin{lemma} \label{lem:bound:S_k}
There exists a universal constant $C>0$ such that, for every $k\ge0$, with probability at least $1-\delta$, the following holds on $\Omega_{n,R}$:
\begin{align*}
\frac{\|S_k\|_{\infty,R}}{m}
\le
\frac{2B_R\sqrt{d}}{\sqrt{2\pi m}} \eta_k + C \sqrt{\frac{d+\log(1/\delta)}{m}} \eqsp,
\end{align*}
where
\begin{align*}
\eta_k \eqdef 
\sum_{r=0}^{k-1} \gamma_r c(t^{(r)}) \|\loss_r\| \eqsp.
\end{align*}
\end{lemma}

\begin{proof}
For all $k \ge 0$ and $\x \in \mathbb{R}^d \times [0,T]$, let $O_k(\x)$ denote the set of neurons whose activation sign at input $\x$ has changed between initialization and iteration $k$, and let $S_k(\x)$ be its cardinality:
\begin{equation} \label{eq:sign_flip}
O_k(\x) := \left\{\, i \in [m] : \operatorname{sgn}\big(\param_i^{(k)\top} \x\big)\neq
\operatorname{sgn}\big(\param_i^{(0)\top} \x\big)\right\},
\qquad
S_k(\x) := \bigl|O_k(\x)\bigr| \eqsp.
\end{equation}
% For any $k \in \mathbb{N}$ and any $\x \in \mathbb{R}^{d+1}$, we have
% \begin{align*}
% S_k(\x)=|O_k(\x)|=\sum_{i=1}^m \mathbf 1_{\{i\in O_k(\x)\}},
% \end{align*}
If $i\in O_k(\x)$, then the sign of $\param_i^\top \x$ changes between iterations $0$ and $k$. Hence, by Lemma~\ref{lem:sgd_updates},
\begin{align*}
| \param_i^{(0)\top}  \x | \leq | \left( \param[k]_i - \param[0]_i \right)^\top \x | & \leq \| \param[k]_i - \param[0]_i  \| \| \x\| \\
& \leq \left( \frac{B_R \|A_{\cdot,i}\|}{\sqrt m}
\sum_{r=0}^{k-1}
\gamma_r c(t^{(r)}) \|\loss_r\| \right) \| \x\| \\
& \leq \left( \frac{B_R \sqrt{d}}{\sqrt m}
\sum_{r=0}^{k-1}
\gamma_r c(t^{(r)}) \|\loss_r\| \right) \| \x\| \\
& \leq \rho_{k} \| \x\| \eqsp,
\end{align*}
where
$$
\rho_k \eqdef \frac{B_R \sqrt{d}}{\sqrt m}
\sum_{r=0}^{k-1} \gamma_r c(t^{(r)}) \|\loss_r\| \eqsp.
$$
Therefore,
$$
S_k(\x)\le \sum_{i=1}^m \mathbf 1_{\{|\param_i^{(0)\top}\x|\le \rho_k\|\x\|\}}.
$$
Now, for every $\x\neq 0$, letting $u=\x/\|\x\|\in \mathbb S^{d}$, we have
$$
|\param_i^{(0)\top}\x|\le \rho_k\|\x\|
\quad\Longleftrightarrow\quad
|\param_i^{(0)\top}u|\le \rho_k.
$$
Taking the supremum yields
\begin{equation}
\label{eq:Sk_sphere_reduction}
\|S_k\|_{\infty,R}
\le
\sup_{u\in\mathbb S^d}
\sum_{i=1}^m
\mathbf 1_{\{|\param_i^{(0)\top}u|\le \rho_k\}}.
\end{equation}

Using Lemma \ref{lem:VC_dim_bound}, there exists a universal constant $C>0$ such that with probability at least $1-\delta$,
$$
\sup_{u\in\mathbb S^d}
\left|
\frac1m\sum_{i=1}^m \mathbf 1_{\{|\param_i^{(0)\top}u|\le \rho_k\}}
-
\mathbb P\bigl(|\param_1^{(0)\top}u|\le \rho_k\bigr)
\right|
\le
C \sqrt{\frac{d+\log(1/\delta)}{m}} \eqsp.
$$

In particular,
$$
\sup_{u\in\mathbb S^d}
\frac1m\sum_{i=1}^m \mathbf 1_{\{|\param_i^{(0)\top}u|\le \rho_k\}}
\le
\sup_{u\in\mathbb S^d}\mathbb P\bigl(|\param_1^{(0)\top}u|\le \rho_k\bigr)
+
C \sqrt{\frac{d+\log(1/\delta)}{m}}.
$$
Since $\param_1^{(0)\top}u\sim \mathcal N(0,1)$ for every $u\in\mathbb S^d$, we have
$$
\sup_{u\in\mathbb S^d}\mathbb P\bigl(|\param_1^{(0)\top}u|\le \rho_k\bigr)
=
\mathbb P(|Z|\le \rho_k)
\le \frac{2\rho_k}{\sqrt{2\pi}} \eqsp,
$$
where $Z \sim \mathcal{N}(0,1)$. Combining this with \eqref{eq:Sk_sphere_reduction} yields
$$
\frac{\|S_k\|_{\infty,R} }{m}
\le
\frac{2\rho_k}{\sqrt{2\pi}}
+
C\sqrt{\frac{d+\log(1/\delta)}{m}} \eqsp.
$$
Substituting $\rho_k=\frac{B_R\sqrt d}{\sqrt m}\eta_k$ concludes the proof.
\end{proof}

\newpage

\section{Generation error}
\label{sec:sampling_error}

\subsection{Sampling-error Decomposition}

Recall that the exact reverse-time process is given by
\begin{equation}
\label{eq:exact_reverse_sampling}
\rmd \ola X_t = \Big( \alpha \beta_{T-t}\ola X_t + 2 \beta_{T-t} \nabla\log p_{T-t}(\ola X_t) \Big)\rmd t + \sqrt{2\beta_{T-t}}\,\rmd B_t \eqsp, \qquad \ola X_0 \sim p_T \eqsp.
\end{equation}
For the $n$-th SGD iterate $\param[n]$, define
\begin{equation} \label{eq:score_param_sampling}
\snet^{(n)}(x,t) \eqdef \frac{1}{\sigma_t^2} \left( m_t \fnet(x,t;\param[n])-x \right) \eqsp,
\end{equation}
and introduce the corresponding learned reverse process
\begin{equation}
\label{eq:learned_reverse_sampling}
\rmd \ola X_t^{(n)} = \Big( \alpha\beta_{T-t}\ola X_t^{(n)}
+ 2\beta_{T-t}\snet^{(n)}(\ola X_t^{(n)},T-t)
\Big)\rmd t + \sqrt{2\beta_{T-t}}\,\rmd B_t \eqsp,
\qquad \ola X_0^{(n)} \sim \pi_\infty \eqsp.
\end{equation}
Fix a deterministic grid $0=t_0<\cdots<t_K=T-\epsilon$, with $\epsilon \in (0,T)$ and set $h\eqdef\max_{0\le k<K}(t_{k+1}-t_k)$. We consider the continuous Euler--Maruyama interpolation $\bar X^{(n,h)}$, defined, for $t\in[t_k,t_{k+1})$ and
$k=0,\ldots,K-1$, by
\begin{equation}
\label{eq:euler_interpolation_learned}
\rmd \bar X_t^{(n,h)} = \Big( \alpha\beta_{T-t_k} \bar X_{t_k}^{(n,h)} + 2\beta_{T-t_k} \snet^{(n)}(\bar X_{t_k}^{(n,h)},T-t_k) \Big)\rmd t + \sqrt{2\beta_{T-t_k}}\,\rmd B_t \eqsp, \qquad \bar X_0^{(n,h)} = \ola X_0^{(n)} \sim\pi_\infty \eqsp.
\end{equation}
Denote the corresponding drifts by
$$
b_t(x) \eqdef \alpha \beta_{T-t} x + 2 \beta_{T-t}\nabla\log p_{T-t}(x) \eqsp,
\qquad b_t^{(n)}(x) \eqdef \alpha \beta_{T-t} x + 2 \beta_{T-t} \snet^{(n)}(x,T-t) \eqsp.
$$

For the analysis below, we construct the three processes on a common probability space. The initial pair $(\ola X_0,\ola X_0^{(n)})$ is chosen according to an optimal $\mathcal W_2$ coupling of $p_T$ and $\pi_\infty$, we set $\bar X_0^{(n,h)}=\ola X_0^{(n)}$, and the exact, learned, and Euler--Maruyama processes are driven by the same Brownian motion.

Unless $\E_{\sgd}$ is displayed, all expectations and laws involving the learned process, or the above synchronous coupling, are understood conditionally on $\mathcal G_n$, i.e., for a fixed realization of the
trained network. In particular, $\mathcal L(\ola X_t^{(n)})$ denotes the conditional law of the generated sample given $\mathcal G_n$.

To control the propagation of the generation error, we only require a one-sided Lipschitz condition in expectation along the synchronous coupling introduced above.

\begin{assumption}[Coupling-averaged one-sided Lipschitz condition]
\label{ass:mean_stability_sampling}
For almost every realization of the trained parameters, there exists $\ell^{(n,h)}\in L^1([0,T-\epsilon])$ such that, for almost every $t\in[0,T-\epsilon]$,
\begin{equation} \label{eq:mean_stability_sampling}
\E\left[ \left\langle \ola X_t-\bar X_t^{(n,h)}, b_t^{(n)}(\ola X_t) - b_t^{(n)}(\bar X_t^{(n,h)}) \right\rangle \right] \le \ell_t^{(n,h)} \E\left[ \|\ola X_t-\bar X_t^{(n,h)}\|^2 \right] \eqsp. 
\end{equation}
\end{assumption}

% \begin{remark}
% Such quantity can be estimated via a Monte Carlo estimate.
% \end{remark}

% \begin{assumption}[Coupling-averaged one-sided Lipschitz condition]
% \label{ass:mean_stability_sampling}
% For almost every realization of the trained parameters, there exists
% $\ell^{(n)}\in L^1([0,T-\epsilon])$ such that, for almost every
% $t\in[0,T-\epsilon]$,
% \begin{equation}
% \label{eq:mean_stability_sampling}
% \E\left[ \left\langle \ola X_t-\ola X_t^{(n)}, b_t^{(n)}(\ola X_t)-b_t^{(n)}(\ola X_t^{(n)}) \right\rangle \right] \le \ell_t^{(n)} \E\left[ \|\ola X_t-\ola X_t^{(n)}\|^2 \right] \eqsp.
% \end{equation}
% \end{assumption}
%
Assumption~\ref{ass:mean_stability_sampling} only controls the learned drift contribution in expectation along the synchronous coupling. In particular, it does not require a pointwise one-sided Lipschitz inequality uniformly
over all pairs of spatial points. Lemma~\ref{lem:network_osl} below shows that the global one-sided Lipschitz regularity of the trained two-layer ReLU network provides an explicit sufficient condition for Assumption~\ref{ass:mean_stability_sampling}.

\begin{lemma}[Generation error]
\label{lem:sampling_error_osl}
Let $\epsilon\in(0,T)$ and suppose that Assumptions~\ref{ass:subgaussian_data} and~\ref{ass:mean_stability_sampling} hold. Assume moreover that $\beta$ is Lipschitz on $[0,T]$ and that
$$
\int_0^{T-\epsilon} \beta_{T-s} \E \left[ \left\| \nabla\log p_{T-s}(\ola X_s) - \snet^{(n)}(\ola X_s,T-s) \right\|^2 \right]^{1/2} \rmd s <\infty \eqsp.
$$
Then,
\begin{equation}
\label{eq:full_sampling_error}
\mathcal{W}_2\!\left( \mathcal L(\bar X^{(n,h)}_{T-\epsilon}), \pi_{\rm data} \right)
\le \mathcal E_{\rm early} + \mathcal E_{\rm init}^{(n,h)} + \mathcal E_{\rm learn}^{(n,h)} + \mathcal E_{\rm discr}^{(n,h)} \eqsp,
\end{equation}
where
\begin{align}
& \mathcal E_{\rm early} \eqdef \left( (1-m_\epsilon)^2\E\|X_0\|^2 + d\sigma_\epsilon^2 \right)^{1/2} \eqsp, \label{eq:error_early} \\
& \mathcal E_{\rm init}^{(n,h)} \eqdef \exp\!\left( \int_0^{T-\epsilon}\ell_r^{(n,h)}\,\rmd r \right) \mathcal{W}_2(p_T,\pi_\infty) \eqsp, \label{eq:error_initialization} \\
& \mathcal E_{\rm learn}^{(n,h)} \eqdef 2\int_0^{T-\epsilon} \exp\!\left( \int_s^{T-\epsilon}\ell_r^{(n,h)}\,\rmd r \right) \beta_{T-s} \E\left[ \left\| \nabla\log p_{T-s}(\ola X_s) -
\snet^{(n)}(\ola X_s,T-s) \right\|^2 \right]^{1/2} \rmd s \label{eq:error_learning} \\ 
&\mathcal E_{\rm discr}^{(n,h)} \le C_{n,\epsilon}\sqrt{h} \eqsp. \label{eq:error_discretization}
\end{align}
Here $C_{n,\epsilon}<\infty$ depends on the trained parameters $\param[n]$, $\alpha$, $\beta$, $T$, $d$, $\epsilon$, and the second moment of $\pi_\infty$, but is independent of the discretization grid apart from its maximal step size $h$.
\end{lemma}

\begin{proof}
By the triangle inequality and the time-reversal identity
$$
\mathcal{W}_2\!\left( \mathcal L(\bar X^{(n,h)}_{T-\epsilon}),\pi_{\rm data} \right) \le \mathcal{W}_2\!\left( \mathcal L(\bar X^{(n,h)}_{T-\epsilon}), \mathcal L(\ola X_{T-\epsilon}) \right) + \mathcal{W}_2(p_\epsilon,p_0) \eqsp.
$$
Using the forward representation $\ora X_\epsilon=m_\epsilon X_0+\sigma_\epsilon Z$,
with $Z\sim\mathcal N(0,\Id_d)$ independent of $X_0$, yields
\begin{equation} \label{eq:early_stop_estimate}
\mathcal{W}_2(p_\epsilon,p_0) \le \left( (1-m_\epsilon)^2\E\|X_0\|^2 +d\sigma_\epsilon^2 \right)^{1/2} = \mathcal E_{\rm early} \eqsp.
\end{equation}
We now use the synchronous coupling introduced above. Since $\bar X_0^{(n,h)}=\ola X_0^{(n)}$ and $(\ola X_0,\ola X_0^{(n)})$ is an optimal coupling of $p_T$ and $\pi_\infty$,
\begin{equation} \label{eq:init_coupling}
\E\left[ \|\ola X_0-\bar X_0^{(n,h)}\|^2 \right]^{1/2} = \mathcal W_2(p_T,\pi_\infty) \eqsp.
\end{equation}
Fix $t\in[t_k,t_{k+1})$. By It\^o's formula,
\begin{align*}
\frac12\frac{\rmd}{\rmd t} \E\|\ola X_t-\bar X_t^{(n,h)}\|^2
&= \E\left[ \left\langle \ola X_t-\bar X_t^{(n,h)}, b_t(\ola X_t) - b_{t_k}^{(n)}(\bar X_{t_k}^{(n,h)}) \right\rangle \right] +\frac d2 \left( \sqrt{2\beta_{T-t}} - \sqrt{2\beta_{T-t_k}} \right)^2 \eqsp.
\end{align*}
Adding and subtracting $b_t^{(n)}(\ola X_t)$ and $b_t^{(n)}(\bar X_t^{(n,h)})$, we get
\begin{multline*}
b_t(\ola X_t) - b_{t_k}^{(n)}(\bar X_{t_k}^{(n,h)}) = \\
b_t^{(n)}(\ola X_t) - b_t^{(n)}(\bar X_t^{(n,h)})
+ 2\beta_{T-t} \left( \nabla\log p_{T-t}(\ola X_t) -\snet^{(n)}(\ola X_t,T-t)\right) + b_t^{(n)}(\bar X_t^{(n,h)}) - b_{t_k}^{(n)}(\bar X_{t_k}^{(n,h)}) \eqsp.
\end{multline*}
Therefore, using Assumption~\ref{ass:mean_stability_sampling} and Cauchy--Schwarz,
\begin{align} \label{eq:differential_gen_bound}
&\frac{\rmd}{\rmd t} \E \|\ola X_t-\bar X_t^{(n,h)}\|^2 \nonumber \\
&\qquad \le 2\ell_t^{(n,h)} \E\|\ola X_t-\bar X_t^{(n,h)}\|^2 + 4 \beta_{T-t} (\E\|\ola X_t-\bar X_t^{(n,h)}\|^2)^{1/2} \times \E\left[ \left\| \nabla\log p_{T-t}(\ola X_t) - \snet^{(n)}(\ola X_t,T-t) \right\|^2 \right]^{1/2} \nonumber \\
&\qquad+ 2 (\E\|\ola X_t-\bar X_t^{(n,h)}\|^2)^{1/2} \times \E\left[ \left\| b_t^{(n)}(\bar X_t^{(n,h)}) - b_{t_k}^{(n)}(\bar X_{t_k}^{(n,h)}) \right\|^2 \right]^{1/2} + d \left( \sqrt{2\beta_{T-t}} - \sqrt{2\beta_{T-t_k}} \right)^2 \eqsp.
\end{align}
Set
$$
L_f^{(n)} \eqdef \frac{\|A\|_{\rm op}\|\param[n]\|_{\rm op}}{\sqrt m} < \infty \eqsp.
$$
Since the ReLU activation is $1$-Lipschitz, for all
$x,y\in\R^d$ and $s,t\in[\epsilon,T]$,
\begin{equation}
\label{eq:network_augmented_lipschitz}
\|\fnet(x,s;\param[n])-\fnet(y,t;\param[n])\| \le L_f^{(n)} \left( \|x-y\|+|s-t| \right) \eqsp.
\end{equation}
In particular, since the network has no bias,
\begin{equation}
\label{eq:network_linear_growth}
\|\fnet(x,u;\param[n])\| \le L_f^{(n)}(\|x\|+T) \eqsp.
\end{equation}
Then using \eqref{eq:score_param_sampling} and \eqref{eq:learned_reverse_sampling} yields
\begin{equation}
\label{eq:learned_drift_ac}
b_t^{(n)}(x)
= \beta_{T-t} \left( \alpha-\frac{2}{\sigma_{T-t}^2} \right)x
+
\frac{2\beta_{T-t} m_{T-t}}{\sigma_{T-t}^2}\fnet(x,T-t;\param[n])
\eqsp.
\end{equation}
For $u\in[\epsilon,T]$, set
$$
a(u) \eqdef \beta_u\left(\alpha-\frac{2}{\sigma_u^2}\right) \eqsp,
\qquad
c(u) \eqdef \frac{2\beta_um_u}{\sigma_u^2} \eqsp.
$$
Since $\beta$ is Lipschitz, $m$ and $\sigma^2$ are Lipschitz on
$[0,T]$. Moreover, $\sigma_u^2 \ge \sigma_\epsilon^2>0$ for $u \in [\epsilon,T]$. Hence $a$ and $c$ are bounded and Lipschitz on $[\epsilon,T]$. Consequently, for $s,t\in[0,T-\epsilon]$, there exists $L_{n,\epsilon} < \infty$ such that,
\begin{align} \label{eq:drift_bound}
&\|b_t^{(n)}(x)-b_s^{(n)}(y)\| \nonumber \\
&\le \|b_t^{(n)}(x)-b_t^{(n)}(y)\| + \|b_t^{(n)}(y)-b_s^{(n)}(y)\| \nonumber \\
&\le \left( \|a\|_\infty + \|c\|_\infty L_f^{(n)} \right)\|x-y\| +
\left[ \operatorname{Lip}(a) \|y\| + \operatorname{Lip}(c) L_f^{(n)}(\|y\|+T) + \|c\|_\infty L_f^{(n)} \right]|t-s| \nonumber \\
&\le L_{n,\epsilon} \left( \|x-y\| + |t-s|(1+\|y\|) \right) \eqsp.
\end{align}

% In particular, setting $s=t$ in the preceding estimate shows that
% $b_t^{(n)}$ is uniformly Lipschitz in space on $[0,T-\epsilon]$: there
% exists $L_{b,n,\epsilon}<\infty$ such that
% \[
% \|b_t^{(n)}(x)-b_t^{(n)}(y)\|
% \le
% L_{b,n,\epsilon}\|x-y\|
% \eqsp.
% \]
% Hence
% \[
% \left\langle
% x-y,b_t^{(n)}(x)-b_t^{(n)}(y)
% \right\rangle
% \le
% L_{b,n,\epsilon}\|x-y\|^2,
% \]
% so that $L_{b,n,\epsilon}$ is also an admissible coefficient in
% Assumption~\ref{ass:mean_stability_sampling}. Replacing
% $\ell_t^{(n,h)}$ by
% $\min\{\ell_t^{(n,h)},L_{b,n,\epsilon}\}$ if necessary, we may therefore
% assume without loss of generality that
% \[
% \ell_t^{(n,h)}\le L_{b,n,\epsilon},
% \qquad
% t\in[0,T-\epsilon].
% \]

In particular, there exists $L_{n,\epsilon}<\infty$, independent of the discretization grid, such that
$$
\|b_t^{(n)}(x)-b_t^{(n)}(y)\| \le L_{n,\epsilon}\|x-y\| \eqsp, \qquad t\in[0,T-\epsilon] \eqsp.
$$
Hence
$$
\left\langle x-y,b_t^{(n)}(x)-b_t^{(n)}(y) \right\rangle \le L_{n,\epsilon}\|x-y\|^2 \eqsp,
$$
so $L_{n,\epsilon}$ is itself an admissible coefficient in
Assumption~\ref{ass:mean_stability_sampling}. Thus, replacing
$\ell_t^{(n,h)}$ by
$\min\{\ell_t^{(n,h)},L_{n,\epsilon}\}$ if necessary, we may assume without loss of generality that
\begin{equation}
\label{eq:ell_uniform_upper_bound}
\ell_t^{(n,h)} \le L_{n,\epsilon} \eqsp, \qquad t\in[0,T-\epsilon] \eqsp.
\end{equation}

Similarly, for a possibly different finite constant $L_{n,\epsilon}$
\begin{equation} \label{eq:growth_drift}
    \|b_t^{(n)}(x)\| \le L_{n,\epsilon}(1+\|x\|) \eqsp.
\end{equation}
Using the Euler recursion, the linear-growth bound on $b^{(n)}$, and the fact that
$\pi_\infty$ is Gaussian, a discrete Gr\"onwall argument gives
$$
\sup_{0\le t\le T-\epsilon} \E\|\bar X_t^{(n,h)}\|^2 \le C_{n,\epsilon} \eqsp.
$$
Moreover, for $t\in[t_k,t_{k+1})$,
$$
\E\left[ \|\bar X_t^{(n,h)}-\bar X_{t_k}^{(n,h)}\|^2 \right] \le C_{n,\epsilon}h \eqsp.
$$
Combining these estimates with the regularity of $b^{(n)}$
yields
$$
\E\left[ \left\| b_t^{(n)}(\bar X_t^{(n,h)}) - b_{t_k}^{(n)}(\bar X_{t_k}^{(n,h)}) \right\|^2 \right]^{1/2} \le C_{n,\epsilon}\sqrt h \eqsp.
$$
Finally, since $\beta$ is Lipschitz and $(\sqrt a-\sqrt b)^2\le |a-b|$ for $a,b\ge0$,
$$
d\left( \sqrt{2\beta_{T-t}} - \sqrt{2\beta_{T-t_k}} \right)^2 \le 2d \operatorname{Lip}(\beta) h \eqsp.
$$
Combining the preceding estimates, for $t\in[t_k,t_{k+1})$, with \eqref{eq:differential_gen_bound}
\begin{multline*}
\frac{\rmd}{\rmd t} \E\|\ola X_t-\bar X_t^{(n,h)}\|^2 \le 2 \ell_t^{(n,h)} \E\|\ola X_t-\bar X_t^{(n,h)}\|^2 \\
+ 2\left\{ 2\beta_{T-t} \E\left[ \left\| \nabla\log p_{T-t}(\ola X_t) - \snet^{(n)}(\ola X_t,T-t) \right\|^2 \right]^{1/2} + C_{n,\epsilon}\sqrt h \right\} \left(\E\|\ola X_t-\bar X_t^{(n,h)}\|^2 \right)^{1/2} + C_{n,\epsilon}h \eqsp.
\end{multline*}
Applying Gr\"onwall inequality together with \eqref{eq:init_coupling} yields
\begin{align*}
\E\left[ \|\ola X_t-\bar X_t^{(n,h)}\|^2 \right]^{1/2}
&\le \exp\!\left( \int_0^t\ell_r^{(n,h)}\,\rmd r \right) \mathcal W_2(p_T,\pi_\infty) \\
&\quad+ 2\int_0^t \exp\!\left( \int_s^t\ell_r^{(n,h)}\,\rmd r \right) \beta_{T-s} \E \left[ \left\| \nabla\log p_{T-s}(\ola X_s) - \snet^{(n)}(\ola X_s,T-s) \right\|^2 \right]^{1/2} \rmd s
\\
&\quad+ C_{n,\epsilon}\sqrt h \int_0^t \exp\!\left( \int_s^t\ell_r^{(n,h)}\,\rmd r \right) \rmd s + \left[ C_{n,\epsilon}h \int_0^t \exp\!\left( 2\int_s^t\ell_r^{(n,h)}\,\rmd r
\right) \rmd s \right]^{1/2} \eqsp.
\end{align*}
By~\eqref{eq:ell_uniform_upper_bound}, for $0\le s\le t\le T-\epsilon$,
$$
\exp\!\left\{ \int_s^t\ell_r^{(n,h)}\,\rmd r \right\} \le \exp\!\left\{L_{n,\epsilon}(t-s)\right\} \le \exp\!\left\{ L_{n,\epsilon}T\right\} \eqsp,
$$
uniformly over the discretization grid. Hence the last two terms above are bounded by $C_{n,\epsilon}\sqrt h$, after possibly changing the value of the generic constant.

Since the Wasserstein distance is bounded above by the root mean-square distance of this coupling, taking $t=T-\epsilon$ gives
$$
\mathcal W_2\!\left( \mathcal L(\bar X_{T-\epsilon}^{(n,h)}), \mathcal L(\ola X_{T-\epsilon}) \right) \le \mathcal E_{\rm init}^{(n,h)} + \mathcal E_{\rm learn}^{(n,h)} + \mathcal E_{\rm discr}^{(n,h)} \eqsp.
$$
Combining this inequality with the early-stopping estimate \eqref{eq:early_stop_estimate} concludes the proof.

\end{proof}

\subsection{Sufficient condition for Assumption~\ref{ass:mean_stability_sampling}}

We now give an explicit sufficient condition for Assumption~\ref{ass:mean_stability_sampling} in terms of the spatial regularity of the trained denoiser. This condition is global in the state space and is therefore expected to be conservative: the resulting bound
controls the worst-case one-sided Lipschitz behavior of the learned drift, whereas Assumption~\ref{ass:mean_stability_sampling} only requires stability on average along the synchronous coupling actually explored by the reverse dynamics.

\begin{lemma}[Global one-sided Lipschitz bound] \label{lem:network_osl}
For $t\in(0,T]$, define
$$
L_n(t) \eqdef \sup_{x\neq y} \frac{ \left\langle x-y, \fnet(x,t;\param[n])-\fnet(y,t;\param[n]) \right\rangle
}{ \|x-y\|^2 } \eqsp.
$$
Let $W_x^{(n)}\in\R^{m\times d}$ denote the first $d$ columns of $\param[n]$. Then
\begin{equation} \label{eq:network_osl_bound}
L_n(t) = \operatorname*{ess\,sup}_{x\in\R^d} \lambda_{\max} \left( \operatorname{Sym} \left[ \frac{1}{\sqrt m} A D_{(x,t),\param[n]}W_x^{(n)} \right] \right) \le \frac{\|A\|_{\rm op}\|W_x^{(n)}\|_{\rm op}}{\sqrt m} < \infty \eqsp.
\end{equation}
Consequently, for every $t\in[0,T-\epsilon]$ and every $x,y\in\R^d$,
\begin{equation} \label{eq:reverse_drift_osl_network}
\left\langle x-y, b_t^{(n)}(x)-b_t^{(n)}(y) \right\rangle \le \ell_{t,\mathrm{glob}}^{(n)} \|x-y\|^2 \eqsp,
\end{equation}
where
\begin{equation} \label{eq:reverse_drift_osl_coefficient}
\ell_{t,\mathrm{glob}}^{(n)} \eqdef \beta_{T-t} \left[ \alpha + \frac{2\bigl(m_{T-t}L_n(T-t)-1\bigr)}{\sigma_{T-t}^2} \right] \eqsp.
\end{equation}
In particular, Assumption~\ref{ass:mean_stability_sampling} holds with
the admissible choice
$$
\ell_t^{(n)} = \ell_{t,\mathrm{glob}}^{(n)} \eqsp.
$$
\end{lemma}

\begin{proof}
For almost every $x\in\R^d$,
$$
\nabla_x\fnet(x,t;\param[n]) = \frac{1}{\sqrt m} A D_{(x,t),\param[n]}W_x^{(n)} \eqsp.
$$
For any $x,y\in\R^d$,
$$
\fnet(x,t;\param[n])-\fnet(y,t;\param[n]) = \int_0^1 \nabla_x\fnet(y+r(x-y),t;\param[n])(x-y)\,\rmd r.
$$
Since $z^\top Mz = z^\top\operatorname{Sym}(M)z \le \lambda_{\max}(\operatorname{Sym}(M))\|z\|^2$, for every matrix $M$, we obtain
\begin{align*}
&\left\langle x-y, \fnet(x,t;\param[n])-\fnet(y,t;\param[n]) \right\rangle \\
&\quad= \int_0^1 \left\langle x-y, \nabla_x\fnet(y+r(x-y),t;\param[n])(x-y) \right\rangle \,\rmd r \\
&\quad\le \operatorname*{ess\,sup}_{z\in\R^d} \lambda_{\max} \left( \operatorname{Sym} \left[ \frac{1}{\sqrt m} A D_{(z,t),\param[n]}W_x^{(n)} \right] \right) \|x-y\|^2 \eqsp.
\end{align*}
This gives the characterization of $L_n(t)$ in \eqref{eq:network_osl_bound}. Moreover, since $D_{(x,t),\param[n]}$ is diagonal with entries in $\{0,1\}$,
$$
\|D_{(x,t),\param[n]}\|_{\rm op}\le 1,
$$
and therefore
$$
L_n(t) \le \frac{1}{\sqrt m} \|A\|_{\rm op} \|W_x^{(n)}\|_{\rm op} < \infty \eqsp.
$$
Finally, by~\eqref{eq:score_param_sampling},
\begin{align*}
&\left\langle x-y, b_t^{(n)}(x)-b_t^{(n)}(y) \right\rangle \\
&\quad\le \beta_{T-t} \left[ \alpha+ \frac{2\bigl(m_{T-t}L_n(T-t)-1\bigr)}{\sigma_{T-t}^2} \right] \|x-y\|^2 \\
&\quad= \ell_{t,\mathrm{glob}}^{(n)} \|x-y\|^2 \eqsp,
\end{align*}
which proves~\eqref{eq:reverse_drift_osl_network}. Applying this pointwise inequality to
$x=\ola X_t$ and $y=\ola X_t^{(n)}$ and taking expectations yields Assumption~\ref{ass:mean_stability_sampling} with
$\ell_t^{(n)}=\ell_{t,\mathrm{glob}}^{(n)}$.
\end{proof}

\subsection{Final Generation Error}
\label{sec:proof_generation_main}

\textit{Proof of Theorem \ref{th:generation_error_main}.} By Lemma~\ref{lem:sampling_error_osl} and the definition of $\Gamma_{n,h}(\epsilon)$ in \eqref{eq:gamma_def},
$$
\mathcal E_{\rm init}^{(n,h)} \le \Gamma_{n,h}(\epsilon)\, \mathcal W_2(p_T,\pi_\infty) \eqsp,
$$
and
$$
\mathcal E_{\rm learn}^{(n,h)} \le 2 \Gamma_{n,h}(\epsilon) \int_0^\tau \beta_{T-s} \E\!\left[ \left\| \score_{T-s}(\ola X_s) -\snet^{(n)}(\ola X_s,T-s) \right\|^2 \right]^{1/2} \rmd s \eqsp.
$$
Using the time-reversal property and Cauchy--Schwarz yield
\begin{align*}
\mathcal E_{\rm learn}^{(n,h)}
&\le 2\Gamma_{n,h}(\epsilon) \left( \int_\epsilon^T\frac{\beta_t^2}{\lambda(t)}\,\rmd t \right)^{1/2} \left( \int_\epsilon^T \lambda(t) \E\!\left[ \|\score_t(\ora X_t)-\snet^{(n)}(\ora X_t,t)\|^2 \right]\rmd t \right)^{1/2} \\
&\le 2\sqrt2\,\Gamma_{n,h}(\epsilon) \left( \int_\epsilon^T\frac{\beta_t^2}{\lambda(t)}\,\rmd t \right)^{1/2} \sqrt{\mathcal L(\param[n])} \eqsp.
\end{align*}
Then, the result follows from~\eqref{eq:full_sampling_error}. \qed

\begin{lemma}[Contractive VP]
\label{lem:contractive_vp_discretization}
Let $\epsilon\in(0,T)$, and suppose that $\alpha=1$ and $\beta\equiv1$. For a fixed trained network, let $L_{n,\epsilon}$ be such that \eqref{eq:drift_bound} and \eqref{eq:growth_drift} hold. Assume moreover that, for some deterministic $\kappa>0$,
\begin{equation} \label{eq:vp_integrated_contraction_app}
\int_s^{T- \epsilon} \ell_r^{(n,h)}\,\rmd r \le -2\kappa(T - \epsilon -s) \eqsp, \qquad 0\le s\le\tau  \eqsp.
\end{equation}
Let $ h\le 1\wedge\kappa/5$ be deterministic and choose,
conditionally on the trained network, a grid whose mesh size
$\bar h=\max_k(t_{k+1}-t_k)$ satisfies
\begin{equation}
\label{eq:vp_mesh_condition}
\bar h \le \frac{ h}{1+L_{n,\epsilon}^2} \eqsp.
\end{equation}
Then
\begin{align}
\mathcal W_2\!\left( \mathcal L(\bar X_{T-\epsilon}^{(n,h)}),\pi_{\rm data}\right) \le & \eqsp
\mathcal E_{\rm early}
+ \rme^{-\kappa(T-\epsilon)}\mathcal W_2(p_T,\pi_\infty) + \bar C \sqrt h \nonumber\\
&+ 2\sqrt2 \left( \int_\epsilon^T \frac{\rme^{-2\kappa(t-\epsilon)}}{\lambda(t)} \,\rmd t \right)^{1/2} \sqrt{\mathcal L(\param[n])}  \eqsp,
\label{eq:vp_generation_conditional}
\end{align}
where $\bar C<\infty$ is deterministic and,
in particular, independent of the trained network and of the grid.
\end{lemma}

\begin{proof}
Let
$$
r_t \eqdef \E \left[ \|\ola X_t-\bar X_t^{(n,h)}\|^2 \right]^{1/2} \eqsp, \qquad
M \eqdef \sup_{0\le t\le T - \epsilon}(\E\|\ola X_t\|^2)^{1/2}<\infty  \eqsp.
$$
Since $\mathcal L(\ola X_t)=p_{T-t}$, $M$ only depends on the second
moment of $\pi_{\rm data}$ and on $d,T$ not on the grid.
Note that from the proof of Lemma \ref{lem:sampling_error_osl} \eqref{eq:drift_bound} and \eqref{eq:growth_drift}, there exists a finite constant $L_{n,\epsilon}$, independent of the discretization grid, such that, for every
$s,t\in[0,T-\epsilon]$ and $x,y\in\R^d$,
\begin{equation}
\label{eq:learned_drift_regularity}
\begin{aligned}
\|b_t^{(n)}(x)-b_s^{(n)}(y)\| &\le L_{n,\epsilon} \bigl(\|x-y\|+|t-s|(1+\|y\|)\bigr) \eqsp,\\
\|b_t^{(n)}(x)\| &\le L_{n,\epsilon}(1+\|x\|) \eqsp.
\end{aligned}
\end{equation}
Fix $t\in[t_k,t_{k+1})$ and set $\delta=t-t_k$. By the Euler interpolation \eqref{eq:euler_interpolation_learned} representation and the above,
\begin{align*}
\E\left[ \|\bar X_t^{(n,h)}-\bar X_{t_k}^{(n,h)}\|^2 \right]^{1/2}
&\le \delta L_{n,\epsilon} \left( 1+(\E\|\bar X_{t_k}^{(n,h)}\|^2)^{1/2} \right) + \sqrt{2d\delta} \eqsp.
\end{align*}
Moreover,
$$
(\E\|\bar X_{t_k}^{(n,h)}\|^2)^{1/2} \le M + r_t + \E\left[ \|\bar X_t^{(n,h)}-\bar X_{t_k}^{(n,h)}\|^2 \right]^{1/2} \eqsp.
$$
By~\eqref{eq:vp_mesh_condition}, $\delta L_{n,\epsilon}\le1/2$, and therefore
\begin{equation}
\label{eq:vp_euler_increment}
\E\left[ \|\bar X_t^{(n,h)}-\bar X_{t_k}^{(n,h)}\|^2 \right]^{1/2} \le 2 \delta L_{n,\epsilon}(1+M+r_t) + 2\sqrt{2d\delta} \eqsp.
\end{equation}
Using~\eqref{eq:learned_drift_regularity}, \eqref{eq:vp_euler_increment}, and again \eqref{eq:vp_mesh_condition}, we obtain
\begin{align*}
\E\left[ \left\| b_t^{(n)}(\bar X_t^{(n,h)}) - b_{t_k}^{(n)}(\bar X_{t_k}^{(n,h)}) \right\|^2 \right]^{1/2}
& \le \delta\bigl(4L_{n,\epsilon}^2+L_{n,\epsilon}\bigr) (1+M+r_t) + 4L_{n,\epsilon}\sqrt{2d\delta} \\
&\le 5h\,r_t+\bar C\sqrt h \eqsp,
\end{align*}
Since $\beta\equiv1$, the diffusion coefficient is constant and the
diffusion discretization term vanishes. Thus,
$$
r_t' \le \bigl(\ell_t^{(n,h)}+5h\bigr)r_t + 2\E\left[ \left\| \score_{T-t}(\ola X_t) - \snet^{(n)}(\ola X_t,T-t) \right\|^2 \right]^{1/2} + \bar C \sqrt h \eqsp.
$$
Since $h\le\kappa/5$, \eqref{eq:vp_integrated_contraction_app} implies
$$
\int_s^\tau \bigl(\ell_r^{(n,h)}+5h\bigr)\,\rmd r \le -\kappa(\tau-s) \eqsp.
$$
Therefore,
\begin{multline*}
r_{T-\epsilon}
\le{} \rme^{-\kappa (T-\epsilon)} \mathcal W_2(p_T,\pi_\infty) \\
+ 2\int_0^{T-\epsilon} \rme^{-\kappa(T-\epsilon-s)} \E\left[ \left\| \score_{T-s}(\ola X_s) - \snet^{(n)}(\ola X_s,T-s) \right\|^2 \right]^{1/2} \rmd s + \frac{ \bar C}{\kappa}\sqrt h \eqsp.
\end{multline*}
Changing variables $t=T-s$, using time reversal and Cauchy--Schwarz,
$$ 
\mathcal E_{\rm learn}^{(n,h)} \le 2\sqrt2 \left( \int_\epsilon^T \frac{\rme^{-2\kappa(t-\epsilon)}}{\lambda(t)} \,\rmd t \right)^{1/2} \sqrt{\mathcal L(\param[n])} \eqsp.
$$
Finally,
$$
\mathcal W_2\!\left( \mathcal L(\bar X_{T-\epsilon}^{(n,h)}),\pi_{\rm data}
\right) \le r_{T-\epsilon}+\mathcal W_2(p_\epsilon,p_0) \eqsp,
$$
and \eqref{eq:early_stop_estimate} yields the result, after absorbing
the factor $1/\kappa$ into $\bar C$
\end{proof}

\textit{Proof of Corollary~\ref{cor:generation_error_vp}.}
Apply Lemma~\ref{lem:contractive_vp_discretization}. For the
Ornstein--Uhlenbeck process,
$m_\epsilon=\rme^{-\epsilon}$ and
$\sigma_\epsilon^2=1-\rme^{-2\epsilon}$, hence, for $\epsilon\le1$,
$$
\mathcal E_{\rm early}
\lesssim
\sqrt{\epsilon(\E\|X_0\|^2+d)}.
$$
Moreover,
$$
\mathcal W_2(p_T,\pi_\infty)
\le
\rme^{-T}\mathcal W_2(\pi_{\rm data},\pi_\infty).
$$
Taking $\E_{\sgd}$ in
\eqref{eq:vp_generation_conditional} concludes the proof.
\qed

\end{document}